\documentclass[journal]{IEEEtran}
\usepackage[utf8]{inputenc}

\usepackage{cite}
\usepackage{upgreek,bm}
\usepackage{amsmath, amssymb, amsfonts}
\usepackage[dvipsnames]{xcolor}
\usepackage{amsthm}
\newtheorem{proposition}{Proposition}

\usepackage{amsthm}
\usepackage{graphicx}
\usepackage{todonotes}
\usepackage[colorlinks=true]{hyperref}
\usepackage{float} 
\usepackage{xcolor} 
\usepackage{forest} 
\usepackage{listings} 
\usepackage{pdflscape}
\usepackage{adjustbox}
\usepackage{framed} 
\colorlet{shadecolor}{green!5} 
\usepackage{paralist, tabularx}

\usepackage{multirow}
\usepackage{booktabs} 
\usepackage{colortbl} 
\usepackage{algorithm}
\usepackage{algpseudocode}

\newcommand\scalemath[2]{\scalebox{#1}{\mbox{\ensuremath{\displaystyle #2}}}}

\usepackage{array}
\newcommand{\PreserveBackslash}[1]{\let\temp=\\#1\let\\=\temp}
\newcolumntype{C}[1]{>{\PreserveBackslash\centering}m{#1}}
\newcolumntype{R}[1]{>{\PreserveBackslash\raggedleft}p{#1}}
\newcolumntype{L}[1]{>{\PreserveBackslash\raggedright}p{#1}}

\usepackage{tikz}
\usetikzlibrary{calc}
\usepackage{xparse}
\DeclareDocumentCommand{\hcancel}{mO{0pt}O{0pt}O{0pt}O{0pt}}{%
\tikz[baseline=(tocancel.base)]{
\node[inner sep=0pt,outer sep=0pt] (tocancel) {$#1$};
\draw[red] ($(tocancel.south west)+(#2,#3)$) -- ($(tocancel.north east)+(#4,#5)$);
}%
}%

\usepackage{mathtools}
\DeclarePairedDelimiter\abs{\lvert}{\rvert}%
\DeclarePairedDelimiterX{\norm}[1]{\lVert}{\rVert}{#1}

\makeatletter
\let\oldabs\abs
\def\abs{\@ifstar{\oldabs}{\oldabs*}}
\let\oldnorm\norm
\def\norm{\@ifstar{\oldnorm}{\oldnorm*}}

\usepackage{xcolor}

\allowdisplaybreaks

\makeatletter
\def\ps@IEEEtranplain{%
  \def\@oddfoot{\normalfont\hfil\thepage\hfil}%
  \def\@evenfoot{\normalfont\hfil\thepage\hfil}%
  \def\@oddhead{\hbox{}\hfil\normalfont\footnotesize 
    IEEE TRANSACTIONS ON ROBOTICS (T-RO) - PREPRINT VERSION. ACCEPTED SEPTEMBER 2025\hfil\hbox{}}%
  \def\@evenhead{\hbox{}\hfil\normalfont\footnotesize 
    IEEE TRANSACTIONS ON ROBOTICS (T-RO) - PREPRINT VERSION. ACCEPTED SEPTEMBER 2025\hfil\hbox{}}%
}
\def\ps@IEEEtitlepagestyle{\ps@IEEEtranplain}
\makeatother

\makeatletter
\def\ps@IEEEtranplain{%
  \def\@oddfoot{\normalfont\hfil\thepage\hfil}
  \def\@evenfoot{\normalfont\hfil\thepage\hfil}%
  \def\@oddhead{\hbox{}\hfil\normalfont\footnotesize 
    IEEE TRANSACTIONS ON ROBOTICS (T-RO) - PREPRINT VERSION. ACCEPTED SEPTEMBER 2025\hfil\hbox{}}%
  \def\@evenhead{\hbox{}\hfil\normalfont\footnotesize 
    IEEE TRANSACTIONS ON ROBOTICS (T-RO) - PREPRINT VERSION. ACCEPTED SEPTEMBER 2025\hfil\hbox{}}%
}
\makeatother

\newcommand{\customcopyrightpage}{
    \begin{titlepage}
        \pagestyle{IEEEtranplain} 
        \parbox{1.0\textwidth}{ 
  This paper has been accepted for publication in IEEE Transactions on Robotics (T-RO). The final version will be available via IEEE Xplore. \\
  \\
            \textcopyright~2025 IEEE.
            Personal use of this material is permitted. Permission from IEEE must be obtained for all other uses, in any
            current or future media, including reprinting/republishing this material for advertising or promotional purposes, creating new
            collective works, for resale or redistribution to servers or lists, or reuse of any copyrighted component of this work in other
            works.
        }\\[1cm]
    \end{titlepage}
}

\begin{document}

\customcopyrightpage


\pagestyle{IEEEtranplain}
\def\ps@IEEEtitlepagestyle{\ps@IEEEtranplain}

\markboth{IEEE TRANSACTIONS ON ROBOTICS,~Vol.~x, No.~x, ~2019}%
{Peng \MakeLowercase{\textit{et al.}}: IEEE TRANSACTIONS ON ROBOTICS}

\title{
$\sqrt{\mathbf{VINS}}$: Robust and Ultrafast Square-Root Filter-based 3D Motion Tracking
}
\author{
Yuxiang Peng,~\IEEEmembership{Student Member,~IEEE,} 
Chuchu Chen,~\IEEEmembership{Member,~IEEE,}
Kejian Wu,~\IEEEmembership{Member,~IEEE,} and Guoquan Huang,~\IEEEmembership{Senior Member,~IEEE}
\thanks{This work was partially supported by 
the University of Delaware (UD) College of Engineering, 
Delaware NASA/EPSCoR Seed Grant, 
NSF (SCH-2014264, MRI-2018905, IIS-2410019),  
Meta Reality Labs, and Google.
Chen was also  supported by the UD Doctoral Fellowship.
}
\thanks{Peng and Huang are with the Department of Mechanical Engineering, University of Delaware, Newark, DE, USA. {\tt\small \{yxpeng,ghuang\}@udel.edu}.}
\thanks{Chen is with the Department of Department of Mechanical and Aerospace Engineering, George Washington University, Washington, DC, USA. {\tt\small chuchu.chen@gwu.edu}.} 
\thanks{Wu is with XREAL Inc, Beijing, China. {\tt\small kejian@xreal.com}.}
}

\maketitle

\IEEEpubid{\begin{minipage}{\columnwidth}\footnotesize
This paper has been accepted for publication in the Special Issue on Visual SLAM of \textbf{IEEE Transactions on Robotics (T-RO)}.\\
The final version will be available via IEEE Xplore. DOI: \textbf{10.1109/TRO.2025.3619051}.\\
\textcopyright~2025 IEEE. Personal use of this material is permitted. Permission from IEEE must be obtained for all other uses, in any\\
current or future media, including reprinting/republishing this material for advertising or promotional purposes, creating new\\
collective works, for resale or redistribution to servers or lists, or reuse of any copyrighted component of this work in other\\
works.
\end{minipage}}

\newtheorem{lemma}{Lemma}


\begin{abstract}

In this paper, we develop and open-source, for the first time, 
a robust and efficient square-root filter (SRF)-based visual-inertial navigation system (VINS), termed $\sqrt{\mathrm{VINS}}$, 
which is ultra-fast, numerically stable, and capable of dynamic initialization even under extreme conditions 
(i.e., extremely small time window).   
Despite recent advancements in VINS, resource constraints and numerical instability on embedded (robotic) systems with limited precision remain critical challenges. 
A square-root covariance-based filter offers a promising solution by providing numerical stability, efficient memory usage, and guaranteed positive semi-definiteness. 
However, canonical SRFs suffer from inefficiencies caused by disruptions in the triangular structure of the covariance matrix during updates.
The proposed method significantly improves VINS efficiency with a novel Cholesky decomposition (LLT)-based SRF update, by fully exploiting the system structure and the SRF to preserve the upper triangular structure of square-root covariance.
Moreover, we design a fast,
robust, dynamic initialization method,
which first quickly recovers the minimal states without triangulating 3D features and then efficiently performs  {\em iterative} SRF update to refine the full states, enabling seamless VINS operation even in challenging scenarios.
The proposed LLT-based SRF is extensively verified through numerical studies, demonstrating superior numerical stability under challenging conditions and achieving robust efficient performance on 32-bit single-precision floats, operating at {\em twice} the speed of state-of-the-art (SOTA) methods.
Our initialization method, tested on both mobile workstations and Jetson Nano computers
achieving a high success rate of initialization even within a $100ms$ window under minimal conditions.
%
\textcolor{black}{
Finally, the proposed $\sqrt{\mathrm{VINS}}$ is extensively validated across diverse scenarios, demonstrating strong efficiency, robustness, and reliability.
The full open-source implementation is released to support future research and applications.
}

%

\begin{center} 
$\sqrt{\mathrm{VINS}}$: \url{https://github.com/rpng/sqrtVINS}
\end{center}

\end{abstract}

\IEEEpeerreviewmaketitle


\section{Introduction}

Visual-inertial navigation systems (VINS), which utilize a (single) camera and an inertial measurement unit (IMU) for 3D motion tracking, hold significant potential and are widely used in autonomous robots and mobile devices (e.g., see~\cite{Wu2017ICRA,Chen2022IROS,Chen2023ICRA,Peng2024ICRAb}).
%
%
Numerous VINS algorithms have been developed in recent years~\cite{Huang2019ICRA} and can be broadly categorized into covariance-based and information-based methods. 
The former, such as the multi-state constraint Kalman filter (MSCKF)~\cite{Mourikis2007ICRA}, update estimates using a dense covariance matrix but can lose positive definiteness, causing instability and divergence.
The latter, such as the sliding-window optimization~\cite{Leutenegger2015IJRR},  exploit the sparse structure of the information matrix for efficiency, which, however, may become ill-conditioned. While double-precision arithmetic could mitigate the numerical error, this problem is especially pronounced on embedded platforms, where single-precision floating-point arithmetic is only available on the computation unit or is required to speed up and achieve real-time performance.

The square-root filter (SRF), which tracks the square-root covariance matrix, offers promising benefits such as numerical stability, guaranteed positive-definiteness, and reduced memory requirements~\cite{Bierman2006Book}. However, applying SRF to VINS is challenging due to inefficiencies in maintaining the triangular structure of the covariance matrix during updates, particularly with large measurement sizes. This limitation has prevented SRF from being widely used in VINS.
In the recent work~\cite{Peng2024ICRA},
a novel permuted-QR (P-QR) decomposition of SRF was proposed to efficiently utilize the upper-triangular structure during matrix factorization, 
which has enabled seamless integration of SRF into VINS.
%
In this work, building upon~\cite{Peng2024ICRA}, we develop an enhanced Cholesky decomposition (LLT)-based SRF update to further improve computational efficiency while maintaining numerical stability.

Another critical module for VINS is dynamic initialization, which plays a vital role in recovering the initial states \emph{on the fly} to enable seamless and continuous operation.
Existing methods typically construct a linear system using image features and inertial measurements to obtain a closed-form solution, followed by nonlinear optimization to refine the estimated states.
Although the minimal case of dynamic initialization requires 3 frames and 2 features, corresponding to only $100ms$ with a 20Hz camera~\cite{Martinelli2014IJCV}, existing methods often require over 1 second to achieve robust initialization~\cite{Kneip2011IROS, Mur2017RAL, Qin2017IROS, Qin2018TRO, Von2018ICRA, Campos2020ICRA, Zuniga2021RAL, Wei2022TIM, He2023CVPR, He2024arXiv, Wang2024arXiv, Martinelli2011thesis, Martinelli2014IJCV, Dong2012IROS, Li2014RSS, Kaiser2016RAL, Dominguez2018ISMAR, Campos2019ICRA, Evangelidis2021RAL, Zhou2022ECCV, Merrill2023RSS}.
In this work, we develop a fast initialization method that, for the first time, demonstrates the ability to robustly initialize the system, even in the minimal case.

In particular, the main contributions of this paper are summarized as follows:
\begin{itemize}
    \item 
    We propose square-root VINS ($\sqrt{\mathrm{VINS}}$), achieving exceptional efficiency and remarkable numerical stability on 32-bit computing platforms, with performance more than twice as fast as state-of-the-art (SOTA) algorithms for 3D motion tracking.    
    %
    \item Our $\sqrt{\mathrm{VINS}}$ introduces a novel Cholesky decomposition (LLT)-based SRF update method, which significantly outperforms existing approaches by fully leveraging and preserving the upper-triangular structure of the SRF.
    \item Our $\sqrt{\mathrm{VINS}}$ introduces a novel dynamic initialization module that, for the first time, robustly achieves initialization under the theoretical minimal case ($100ms$ with just 3 keyframes).
    \item We perform extensive numerical studies to highlight potential numerical challenges in VINS and underscore the advantages of the proposed $\sqrt{\mathrm{VINS}}$. 
    Comprehensive real-world experiments validate the notable efficiency boost of the proposed method while maintaining high accuracy.
\end{itemize}

The rest of the paper is organized as follows:
After reviewing the literature focusing on VINS estimation and dynamic initialization in Section~\ref{sec:related}, 
we describe the proposed improvement on the SRF efficiency in Section~\ref{sec:srf}.
Section~\ref{sec:vins} presents the proposed $\sqrt{\mathrm{VINS}}$ while 
Section~\ref{sec:dy_init} details the proposed ultrafast dynamic initialization.
The comprehensive evaluation of the proposed methods is conducted in Sections~\ref{sec:exp_srf}, \ref{sec:exp_init}, and \ref{sec:exp_vins}.
Finally, the paper concludes in Section~\ref{sec:conclu}.



\section{Related Work}
\label{sec:related}

\subsection{Visual-Inertial State Estimation}
An accurate, efficient and reliable state estimation algorithm is the foundation of successful VINS.
From the perspective of iterative update and relinearization, VINS can be categorized into optimization-based method~\cite{Qin2018TRO,Leutenegger2015IJRR,Campos2021TRO,Usenko2019RAL,Chen2023IROS}, and filter-based method~\cite{Bloesch2017IJRR,Huai2019IJRR,Geneva2019ICRA,Huai2021ICRA,Van2021ICRA,Van2023TRO,Chen2024WAFR}.
The former performs iterative update and relinearization to solve the non-linear optimization problems in VINS and can achieve potential accuracy gain at the cost of more required computation, while the latter only performs one-time linearization and achieves superior efficiency, especially desirable for low-end platforms, where computation power is in severe constraint.

From the other perspective, VINS state estimation algorithms can also be categorized into covariance and information forms.
In the former such as the extended Kalman filter (EKF) and its variants, the estimator keeps tracking the dense covariance matrix to update the estimate~\cite{Mourikis2007ICRA, Li2013IJRR, Geneva2019CVPR,Van2021ICRA,Yang2022RAL,Chen2022ICRA,Van2023TRO}.
In contrast, the information estimators such as extended Information filters (EIF)~\cite{Huang2011IROS} or optimization-based methods~\cite{Leutenegger2015IJRR, Qin2018TRO, Von2018ICRA, Campos2021TRO,Chen2024ICRA}, 
maintain the information (Hessian) matrix and exploit its sparse structure in solving for estimates.
However, both covariance and information filters face challenging numerical issues, in particular on resource-constrained edge platforms \cite{Kottas2015ICRA,Wu2015RSS}, when limited word length (32-bit float, instead of 64-bit double) is available or it is required to achieve real-time performance.
In the covariance form, the covariance matrix tends to lose its positive definiteness and cause the estimator to diverge. 
In the information form, as the information matrix can easily become ill-conditioned (e.g., condition number larger than $10^{9}$~\cite{Kottas2015ICRA}), naively inverting it during optimization would lead to large numerical errors (see Chapter 3.5.1 in \cite{Golub2013Book}).

To address this numerical instability, there exist methods that use the square root of the information matrix instead of its full matrix and were shown to be effective to some extent in VINS \cite{Wu2015RSS, Wu2016TR, Caruso2017ICIPIN, Ke2019IROS, Huai2021ICRA, Huai2022RAL, Dellaert2006IJRR, Kaess2007ICRA, Kaess2012IJRR, Demmel2021ICCV, Givens2023JGCD, Wu2024Thesis}.
For example, the method in \cite{Wu2015RSS} maintains an upper triangular square root of prior information and uses QR-decomposition to incorporate new measurements into the prior, then invert it to solve for the state update.
While this estimator achieves the same accuracy with the half of the word length, 
it still has the concerning numerical issue with a relatively high condition number ($10^{5}$) over time, 
especially when paired with a high-precision IMU~\cite{Chauchat2020TCST, Chauchat2022CDC}, 
lead to substantial numerical inaccuracies for long-term operations. 
This numerical issue was addressed recently in~\cite{Ke2024aXiv} by preconditioning with a square root information filter.  

On the contrary, VINS estimators in the covariance form tend to offer better numerical stability.
For instance, in the EKF-based VINS, the only matrix that typically requires inversion, the innovation covariance,
usually possesses a good condition number~\cite{Kottas2015ICRA}. 
This makes the use of the square-root covariance matrix highly appealing for VINS, as it combines the advantages of the covariance form with the benefits of square-root properties.
Surprisingly, this idea of SRF remains largely unexplored in VINS.
Looking into history, the SRF has undergone significant improvements over the decades. 
Back in the 1960s, the initial SRF formulation was proposed by Potter and played a significant role in the Apollo project's success \cite{Potter1963GCC, Battin1964Book}, which has been extended to account for propagation (process) noise~\cite{Dyer1969JOTA}.
While several update methods have been developed to enhance SRF efficiency~\cite{Bellantoni1967AIAA, Andrews1968AIAA, Kaminski1971TAC}, these methods are generally less efficient than conventional KFs due to the disruption of the square-root covariance’s triangular structure during updates. 
Agee~\cite{Agee1972WSMRT} and Carlson~\cite{Carlson1973AIAA} proposed methods that preserve the triangular structure and offer similar efficiency to KF, but these are limited to sequential updates.
Modern systems, however, prefer batch updates with vector operations that enable more efficient level-3 BLAS operations~\cite{Dongarra1990TOMS}. 
These limitations of the SRF make it unsuitable for VINS applications that demand quick, real-time estimates from large-size measurements.


To address the aforementioned issues and fully harness the benefits of the square-root covariance, the recent work~\cite{Peng2024ICRA} introduced a novel (P-)QR-based SRF update and applied it to VINS, gaining computational efficiency and numerical stability.
In this work, we significantly extend~\cite{Peng2024ICRA} by designing a novel LLT-based update method for SRF to achieve even greater efficiency and an ultrafast dynamic initialization module, as well as performing substantially more comprehensive experimental evaluations.

\subsection{Visual-Inertial System Initialization}
In addition to the state estimation algorithm, VINS relies on accurate initial conditions (e.g., velocity and gravity) for a successful operation. 
Although initial conditions can be recovered by assuming static motion, this often involves risky assumptions in practical systems. 
Dynamic initialization, on the other hand, enables VINS to recover key parameters while in motion, eliminating the need for the platform to remain stationary at startup. 
\color{black}
This is also crucial for practical applications, as it enables the system to re-initialize after failures and significantly enhances robustness
\color{black}

Dynamic initialization methods are broadly categorized into tightly coupled and loosely coupled approaches. 
Tightly coupled methods (closed-form solutions) recover initial states directly by integrating both visual and inertial measurements~\cite{Martinelli2011TRO,Martinelli2011thesis,Martinelli2014IJCV,Dong2012IROS,Li2014RSS,Kaiser2016RAL,Dominguez2018ISMAR,Campos2019ICRA,Evangelidis2021RAL,Zhou2022ECCV,Merrill2023RSS}. 
The first closed-form solution for gravity direction, velocity, scale, and accelerometer bias was introduced in~\cite{Martinelli2011TRO} and has been extended to include unknown camera-IMU calibration~\cite{Dong2012IROS} with the minimal case and degenerate motion analysis~\cite{Martinelli2014IJCV}.  
In contrast to formulating a linear system, \cite{Li2014RSS} formulates a MAP (Maximum A Posteriori) problem that incorporates sensor noise to further improve accuracy.
Later on, \cite{Kaiser2016RAL} demonstrated the impact of gyro bias on initialization accuracy and proposed its estimation via nonlinear optimization. 
This was later improved by using a rotation-only constraint \cite{Kneip2013ICCV} and adding line constraints \cite{Wang2024arXiv, He2024arXiv}.

On the other hand, loosely coupled methods first perform visual structure-from-motion (SfM), followed by visual-inertial alignment to match IMU measurements and recover the initial states~\cite{Kneip2011IROS,Mur2017RAL,Qin2017IROS,Qin2018TRO,Von2018ICRA,Campos2020ICRA,Zuniga2021RAL,Wei2022TIM,He2023CVPR,He2024arXiv,Wang2024arXiv}.
%
For example, Qin et al.~\cite{Qin2017IROS,Qin2018TRO} leverage a simplified SfM pipeline to obtain the up-to-scale trajectory, and then formulate a linear system that recovers scale, gravity, and velocity.
The IMU uncertainty is further being considered in visual-inertial alignment to improve performance~\cite{Campos2021TRO}.
A more recent method showed that up-to-scale SfM can be formulated in a maximum-likelihood framework and constrain gravity magnitude, improving accuracy and avoiding issues with iterative solutions~\cite{Zuniga2021RAL}.
%

A key challenge for the aforementioned initialization methods is their requirement to solve for 3D feature positions during the initialization process. Given their cubic complexity with respect to the number of features, this can significantly impact efficiency.
Various approaches have been proposed to address this issue.
For example, \cite{Evangelidis2021RAL} enhances efficiency in tightly coupled initialization by marginalizing (projecting) the depth of each feature bearing and the redundant 3DoF features in a reference frame. 
However, this still requires triangulating the 3D positions of features, which can make the marginalization numerically unstable under low parallax conditions, where some of the features are close to rank deficiency.
Recently, \cite{Merrill2023RSS, Merrill2024IJRR} incorporated learned image depth and leverage affine-invariant single-image depth to reduce feature parameters but relies on a depth network.
The most related work is by He et al.~\cite{He2023CVPR}, which employs the LiGT constraint~\cite{Cai2021TPAMI}, a linear constraint that uses known rotation and camera measurements to formulate the linear system without the need to estimate 3D features.
However, their tightly coupled approach incorporates all visual measurements, resulting in a large measurement size. Their loosely coupled method solves the LiGT constraints for positions up to scale before solving for inertial states with additional velocity parameters, increasing the state size.

In contrast, we introduce a novel linear system formulation using inertial and bearing measurements to solve for initial velocity and gravity. 
We neither recover 3D feature positions nor retain unnecessary unknown scales, making our solution both highly efficient and robust.
To ensure successful VINS operation after initialization, we also introduce an SRF-based refinement that improves state accuracy and efficiently computes the initial covariance. Unlike the \color{black}Visual-inertial Bundle Adjustment (VI-BA)\color{black} approach, which relies on potentially unstable matrix inversion, our method directly computes the covariance, ensuring robustness for VINS.
Given these benefits, our initialization method is the first among all existing approaches to achieving the theoretical minimum condition, requiring only 3 sequential frames in just $100~ms$ initialization window.

\begin{figure*}
\centering
\includegraphics[width=0.9\linewidth]{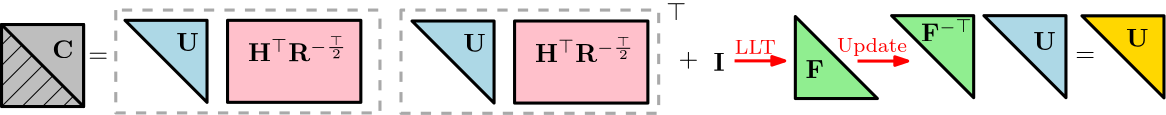}
\caption{Visualization of the matrix computation and its structure evolution during the proposed LLT-based SRF update.
}
\label{fig:srf_update}
\end{figure*}

\section{Improving SRF Computational Efficiency}
\label{sec:srf}

%

The SRF \color{black}has existed \color{black} for decades and potentially has better numerical stability and efficiency~\cite{Maybeck1982Book}. Specifically, unlike a standard EKF (and its variants) that estimates the (dense) covariance matrix $\mathbf{P}$, 
the SRF propagates and updates the corresponding {upper-triangular} square-root matrix $\mathbf{U}$ which is given by~\cite{Maybeck1982Book}:
\begin{align} \label{eq:UtU}
    \mathbf{U}^{\top}\mathbf{U} = \mathbf{P}
\end{align}
%

Because of this special structure, the SRF can represent a much broader dynamic range of numbers and reduce the condition numbers of the involved matrices (especially, the condition number of $\mathbf{U}$ as compared to that of  $\mathbf{P}$).
It also implicitly guarantees the symmetry and positive semi-definiteness of the covariance matrix.
These numerical advantages of the SRF ensure better numerical stability (over the EKF).
Furthermore, the SRF can potentially be more efficient in computation and memory usage, as it can operate with shorter word lengths (e.g., 32-bit floating-point instead of 64-bit) without compromising accuracy.

However, we have not seen its wide adoption in VINS, because it is challenging to fully capitalize on these benefits in \color{black}practice\color{black}.
%
In particular, the canonical update form in the SRF is computationally expensive. 
For example, the Potter SRF~\cite{Potter1963GCC}, the earliest SRF method, destroys the upper triangular structure of $\mathbf{U}$ during the update, necessitating an additional costly matrix triangulation step to restore the structure. 
The Carlson SRF~\cite{Carlson1973AIAA} improved upon the Potter form by preserving the upper triangular structure during updates, resulting in significant speedups.
However, both the Potter and Carlson forms are derived for scalar measurements.
While they can be extended to vector measurements through sequential scalar updates~\cite{Andrews1968AIAA,Kaminski1971TAC}, they fail to fully leverage Single Instruction/Multiple Data (SIMD) in modern computers to accelerate matrix operations.
There are SRF update forms that support direct updates for vector measurements~\cite{Andrews1968AIAA,Kaminski1971TAC,Wu2024Thesis}.
However, they either fail to maintain upper-triangular $\mathbf{U}$ or require a theoretically higher number of floating-point operations during updates, which diminish the benefit of vectorization.
To address these limitations in existing SRFs, we propose a novel Cholesky decomposition (LLT)-based SRF update method. 
This approach overcomes computational bottlenecks by doing vector updates and properly maintaining upper-triangular $\mathbf{U}$ during the update, and also has a smaller number of operations when the measurement size is large.
Moreover, we also fully exploit the upper-triangular structure of the SRF and further optimize its computation in both propagation and update. 
\color{black}

%


\textcolor{black}{In the following section, we first briefly introduce the standard SRF propagation step for completeness. We then present our novel update formulation, which is the core contribution of this work.}

\subsection{QR-based SRF Propagation}
\label{sec:sr-prop}

As the SRF propagates and updates the state estimate (mean) in the same way as the EKF~\cite{Maybeck1982Book},
we here will focus on deriving EKF-equivalent propagation and update of square-root covariance matrix $\mathbf U$.
In particular, given $\mathbf P_{k-1|k-1}=\mathbf{U}_{k-1|k-1}^\top\mathbf{U}_{k-1|k-1}$
leveraging QR decomposition, we can derive the covariance  and its square root propagation as follows:

\begin{align}
\scalemath{.93}{\mathbf{P}_{k|k-1}} &=  \scalemath{.93}{\bm{\Phi}_{k-1}\mathbf{P}_{k-1|k-1}  \bm{\Phi}_{k-1}^{\top} +   \mathbf{W}_{k-1}} \label{eq:cov111}\\
    &= \scalemath{.93}{\bm{\Phi}_{k-1}\mathbf{U}_{k-1|k-1}^\top\mathbf{U}_{k-1|k-1}\bm{\Phi}_{k-1}^{\top} + \mathbf{W}_{k-1}^{\frac{\top} 2}\ \mathbf{W}_{k-1}^{\frac{1}{2}}}\\
    &= \scalemath{.93}{\begin{bmatrix}
          \label{eq:srf_prop_cov_pf1}
        \mathbf{W}_{k-1}^{\frac{\top}{2}}
        &
        \bm{\Phi}_{k-1}
        \mathbf{U}_{k-1|k-1}^\top
    \end{bmatrix}
    \begin{bmatrix}
        \mathbf{W}_{k-1}^{\frac{1}{2}}
        \\
        \mathbf{U}_{k-1|k-1}
        \bm{\Phi}_{k-1}^{\top}         
    \end{bmatrix}}\\
    & \scalemath{.93}{\overset{\mathrm{QR}}{=:   }
    \begin{bmatrix}
        \mathbf{U}_{k|k-1}^\top
        &
        \mathbf{0}
    \end{bmatrix} \mathbf{Q}_{k-1}^\top 
     \mathbf{Q}_{k-1} 
    \begin{bmatrix}
        \mathbf{U}_{k|k-1}
        \\
        \mathbf{0}
    \end{bmatrix}}\\
    &= \mathbf{U}_{k|k-1}^\top \mathbf{U}_{k|k-1}
\end{align}
where $\bm\Phi_{k-1}$ is the state transition or system Jacobian matrix,
and we have employed the Cholesky decomposition (LLT) on the system noise covariance $\mathbf{W}_{k-1}$, i.e., 
$\mathbf{W}_{k-1} \overset{\mathrm{LLT}}{=} \mathbf{W}_{k-1}^{\frac{\top}{2}}\mathbf{W}_{k-1}^{\frac{1}{2}}$,
as well as the fact that $\mathbf Q_{k-1}$ is orthonormal.
It is clear from the above derivations that the square root covariance $\mathbf U_{k|k-1}$ propagates from time $t_{k-1}$ to $t_{k}$  via efficient QR decomposition:
\begin{align}
    \begin{bmatrix}
          \label{eq:srf_prop_cov}
        \mathbf{W}_{k-1}^{\frac{1}{2}}
        \\
        \mathbf{U}_{k-1|k-1}
        \bm{\Phi}_{k-1}^{\top}
    \end{bmatrix}
    \stackrel{\text{QR}}{=} 
    \mathbf{Q}_{k-1} 
    \begin{bmatrix}
        \mathbf{U}_{k|k-1}
        \\
        \mathbf{0}
    \end{bmatrix}
\end{align}

\subsection{LLT-based SRF Update}
As discussed earlier, the update is a computational bottleneck if a canonical SRF form is used.
Although a recent work~\cite{Peng2024ICRA} introduced the permuted-QR (P-QR)-based square-root update to mitigate this issue, there is still room to further improve efficiency with Cholesky decomposition,
which leads to our novel LLT-based efficient SRF update algorithm:
\begin{lemma}
\color{black}{Let the measurement model be given by} $\mathbf z_k = \mathbf h(\mathbf x_k) + \mathbf n_k$, with measurement residual as:\footnote{
Throughout the paper $\hat{\mathbf{x}}$ is used to denote the  estimate of a random variable~$\mathbf{x}$, while $\tilde{\mathbf{x}}$
is the corresponding error state.
The subscript $a|b$ denotes the estimate at time $t_a$ by fusing all the measurements up to time $t_b$.
}
\begin{align}\label{eq:general-meas}
{\mathbf{r}}_{k} := \mathbf z_k -\hat{\mathbf z}_{k|k-1} \simeq \mathbf{H}_{k} \tilde{\mathbf{x}}_{k|k-1}+\mathbf{n}_{k}
\end{align}
where the measurement noise assumes $\mathbf n_k \sim \mathcal N(\mathbf 0, \mathbf R_k)$
and $\mathbf{H}_{k}$ is the measurement Jacobian, the state and square-root covariance update of the SRF can be written as:
\begin{align}
      \label{eq:srf_up_cov}
       \mathbf{U}_{k|k} & = \mathbf{F}_{k}^{-\top}  \mathbf{U}_{k|k-1}       \\
        \hat{\mathbf{x}}_{k|k}    &=        \hat{\mathbf{x}}_{k|k-1}        +
       \mathbf{U}_{k|k}^{\top} \mathbf{U}_{k|k} \mathbf{H}_{k}^{\top}\mathbf{R}_{k}^{-1} {\mathbf{r}}_{k}        \label{eq:srf_up_mean}
\end{align}
where $\mathbf{F}_k$ is computed via Cholesky decomposition of:
\begin{align}
    \label{eq:update}
    \mathbf{C}_k
    & :=
    \mathbf{I} +  \mathbf{U}_{k|k-1}  \mathbf{H}_{k}^{\top} \mathbf{R}_{k}^{-1}\mathbf{H}_{k}\mathbf{U}_{k|k-1}^{\top}
    \stackrel{\mathrm{LLT}}{=}
    \mathbf{F}_{k}^{\top}\mathbf{F}_{k}   
\end{align}
Note that a similar permutation is applied prior to the Cholesky decomposition, following the approach in~\cite{Peng2024ICRA}, to ensure that the resulting matrix $\mathbf{F}_k$ is lower triangular. Consequently, $\mathbf{F}_k^{-\top}$ is upper triangular, as required in the update.
The SRF update proceeds by first computing $\mathbf{U}_{k|k}$ in~\eqref{eq:srf_up_cov}, followed by the state update in~\eqref{eq:srf_up_mean}.
\color{black}
\end{lemma}

\begin{proof}
We now show the proposed LLT-based SRF update is equivalent to the EKF:
\begin{align}
    & 
    \scalemath{.9}{  \mathbf{P}_{k|k} = 
    \mathbf{P}_{k|k-1} - \mathbf{P}_{k|k-1}\mathbf{H}_{k}^{\top}\left(\mathbf H_k\mathbf P_{k|k-1}\mathbf H_k^\top +\mathbf R_k \right)^{-1}\mathbf{H}_{k}\mathbf{P}_{k|k-1} } \notag
    \\
    &=
    \scalemath{.725}{
    \mathbf{U}_{k|k-1}^{\top}\left(
    \mathbf{I}- \mathbf{U}_{k|k-1} \mathbf{H}_{k}^{\top}
    \left(
    \mathbf{H}_{k} \mathbf{U}_{k|k-1}^{\top}\mathbf{U}_{k|k-1} \mathbf{H}_{k}^{\top}+\mathbf{R}_{k} 
    \right)^{-1}
     \mathbf{H}_{k} \mathbf{U}_{k|k-1}^{\top}
    \right)\mathbf{U}_{k|k-1}
    } \notag
    \\
    &=
    \mathbf{U}_{k|k-1}^{\top}
    \left(
    \underbrace{
    \mathbf{I} +  \mathbf{U}_{k|k-1}  \mathbf{H}_{k}^{\top} \mathbf{R}_{k}^{-1}\mathbf{H}_{k}\mathbf{U}_{k|k-1}^{\top}
    }_{\mathbf{C}_k = \mathbf{F}_{k}^{\top}\mathbf{F}_{k}}
    \right)^{-1}
    \mathbf{U}_{k|k-1} \notag
    \\
    \notag
    &=
    \mathbf{U}_{k|k-1}^{\top}\mathbf{F}_{k}^{-1}
    \mathbf{F}_{k}^{-\top}\mathbf{U}_{k|k-1}
    =:
    \mathbf{U}_{k|k}^{\top} \mathbf{U}_{k|k} \notag
\\
\notag
\\
 & \mathbf{P}_{k|k-1}\mathbf{H}_{k}^{\top}\left(\mathbf H_k\mathbf P_{k|k-1}\mathbf H_k^\top +\mathbf R_k \right)^{-1} {\mathbf{r}}_{k}
\notag
\\
\notag
&= 
\mathbf{P}_{k|k-1}\mathbf{H}_{k}^{\top}\left(\mathbf H_k\mathbf P_{k|k-1}\mathbf H_k^\top +\mathbf R_k \right)^{-1}\mathbf R_k \mathbf{R}_{k}^{-1} {\mathbf{r}}_{k}
\\
\notag
&= 
\scalemath{0.8}{(\mathbf{P}_{k|k-1}\mathbf{H}_{k}^{\top} - \mathbf{P}_{k|k-1}\mathbf{H}_{k}^{\top}\left(\mathbf H_k\mathbf P_{k|k-1}\mathbf H_k^\top +\mathbf R_k \right)^{-1}\mathbf{H}_{k}\mathbf{P}_{k|k-1}\mathbf{H}_{k}^{\top}) \mathbf{R}_{k}^{-1} {\mathbf{r}}_{k}}
\\
\notag
&= 
\scalemath{0.8}{(\mathbf{P}_{k|k-1} - \mathbf{P}_{k|k-1}\mathbf{H}_{k}^{\top}\left(\mathbf H_k\mathbf P_{k|k-1}\mathbf H_k^\top +\mathbf R_k \right)^{-1}\mathbf{H}_{k}\mathbf{P}_{k|k-1})\mathbf{H}_{k}^{\top} \mathbf{R}_{k}^{-1} {\mathbf{r}}_{k}}
\\
&= \mathbf{P}_{k|k} \mathbf{H}_{k}^{\top}\mathbf{R}_{k}^{-1} {\mathbf{r}}_{k}
\notag
\\
&= \mathbf{U}_{k|k}^\top \mathbf{U}_{k|k} \mathbf{H}_{k}^{\top}\mathbf{R}_{k}^{-1} {\mathbf{r}}_{k}
\notag
\end{align}
\color{blue}\end{proof}
\color{black}
\textcolor{black}{
In the above derivations, we fully exploit the \textit{upper triangular} structure of $\mathbf{F}_k^{-\top}$. Figure~\ref{fig:srf_update} illustrates the matrix operations and the structural evolution throughout this process. While algebraically equivalent to the EKF, the proposed LLT-based SRF update is significantly more efficient than other square-root update methods. It avoids redundant factorizations and explicitly leverages the triangular structure for faster and more scalable computation. A detailed discussion of these efficiency gains is provided below.
}

\begin{table} 
\centering
\caption{
Flops of the different SRF measurement update forms assuming uncorrelated measurements and considering only the dominant computation terms
($m$ measurements, $n$ states).
}
\resizebox{0.99\linewidth}{!}{%
\begin{tabular}{@{}ccccccc@{}}
\toprule
\textbf{Methods} & \textbf{Potter}\cite{Potter1963GCC} & \textbf{Carlson}\cite{Carlson1973AIAA}   & 
\textbf{Kaminski}\cite{Kaminski1971TAC}\cite{Wu2024Thesis} & 
\textbf{P-QR}\cite{Peng2024ICRA}  
& \textbf{Proposed}      \\ \midrule
\textbf{Flops}   & $6mn^2$ & $\frac{7}{2}mn^2$ & $2m^2n+5mn^2+\frac{4}{3}n^3$ & $3mn^2+\frac{1}{3}n^3$ & $2mn^2+\frac{2}{3}n^3$ \\ \bottomrule  \\ 
\end{tabular}\label{tab:srf-flops}
}
\end{table}



We stress that efficient computation of the lower-triangular matrix $\mathbf{F}_k$ is critical to enable efficient update of the square-root covariance~\eqref{eq:srf_up_cov}, because $\mathbf F_k^{-\top}$ and $\mathbf U_{k|k-1}$  are both upper triangular and their product~\eqref{eq:srf_up_cov} would be trivial and become upper-triangular by structure.
To this end,  when computing $\mathbf{C}_k$~\eqref{eq:update}, we exploit the upper-triangular structure of $\mathbf{U}_{k|k}$ and the symmetry of $\mathbf{C}_k$.
When computing $\mathbf{U}_{k|k}$~\eqref{eq:srf_up_cov}, 
we utilize the upper triangular structure of $\mathbf{F}_k^{\top}$ and  $\mathbf{U}_{k|k-1}$ 
and efficiently solve $\mathbf{F}_k^{\top} \mathbf{U}_{k|k} = \mathbf{U}_{k|k-1}$ via back substitution.
%
To quantify the computational cost,
we calculate the number of arithmetic operations (i.e., floating point operations or FLOPs) required in the update under the common assumptions:
(i) measurements are uncorrelated, and (ii) only the dominant (highest) order of computation is considered.
%
Table~\ref{tab:srf-flops} summarizes the computational complexity of various SRF update methods for $m$ measurements and $n$
states. 
As evident, the proposed LLT-based update method requires fewer operations when $m > \frac{4}{9}n$, as compared to the Carlson update~\cite{Carlson1973AIAA}.
In contrast to the P-QR-based method~\cite{Peng2024ICRA}, this new update has a smaller coefficient for the term involving $mn^2$ while a slightly larger coefficient for the term involving $n^3$. 
As a result, the proposed method is more efficient when $m > \frac{1}{3}n$,
which typically holds true for VINS, where the measurement numbers are significantly larger than the state dimensions.
Theoretically, when decomposing a matrix to obtain its square root, the Cholesky-based method requires fewer FLOPs than the QR-based method, albeit with slightly reduced numerical stability~\cite{Golub2013Book}.
However, due to the relatively small condition number of $\mathbf{C}$  in practical VINS, Cholesky-based decomposition can be employed here without encountering numerical issues.

\section{Square-Root VINS}
\label{sec:vins}

We now apply the proposed SRF presented in the preceding section to the VINS problem and design an ultrafast and numerically stable visual-inertial state estimator, termed {\em Square-Root VINS (sqrt-VINS)} or $\sqrt{\mathrm{VINS}}$. 
The proposed $\sqrt{\mathrm{VINS}}$ is developed within an efficient sliding-window filtering framework by fully exploiting the specific block triangular structure of the system to achieve {\em faster-than-ever} performance.

First of all, we propose a special ordering of state vector to avoid unnecessary computations in $\sqrt{\mathrm{VINS}}$.
Specifically, at time $t_{k}$, the state vector $\mathbf{x}_k$ consists of the current navigation states $\mathbf{x}_{I_k}$, \textcolor{black}{the calibration state $\mathbf{x}_{cb}$}, historical IMU pose clones $\mathbf{x}_{C}$, and a set of 3D environmental (SLAM) features $\mathbf{x}_f$, in the following order:
\begin{align}
    \mathbf{x}_k
    &= 
    \label{eq:state}
    \begin{bmatrix}
    \mathbf{x}_{I_k}^{\top}
    ~
    \mathbf{x}_{cb}^{\top}
    ~|~
    \mathbf{x}_{C}^{\top}
    ~
    \mathbf{x}_{f}^{\top}
    \end{bmatrix}^{\top}           
    =: \begin{bmatrix} \mathbf{x}_{x_k}^{\top} & | & \mathbf{x}_{C}^{\top} &   \mathbf{x}_{f}^{\top}\end{bmatrix}^\top
    \\
      \mathbf{x}_{I_k}
    &=
    \begin{bmatrix}
    \label{eq:state_imu}
    {}_{G}^{I_k}\Bar{q} ^{\top} &
    {}^{G}\mathbf{p}_{I_k} ^{\top} &
    {}^{G}\mathbf{v}_{I_k} ^{\top} &
    \mathbf{b}_{g} ^{\top}&
    \mathbf{b}_{a} ^{\top} 
    \end{bmatrix}^{\top}
    \\
    \mathbf{x}_{cb} &=
    \begin{bmatrix}
        t_d & {}_{C}^{I}\bar{q}^{\top} & {}^{C}\mathbf{p}_I^{\top} & \boldsymbol{\zeta}^{\top}
    \end{bmatrix}
    \\
    \mathbf{x}_{C}
    &=
    \begin{bmatrix}
    \mathbf{x}_{T_{k}}^{\top} 
    \dots
    \mathbf{x}_{T_{k-c}}^{\top}
    \end{bmatrix}^{\top} 
    \\
    \mathbf{x}_f
    \label{eq:state_feat}
    & =
    \begin{bmatrix}
    {}^{G}\mathbf{f}_{1}^{\top}
    \dots
   {}^{G}\mathbf{f}_{L}^{\top}
    \end{bmatrix}^{\top}
\end{align}
where $ {}_{G}^{I}\Bar{q}$ is the JPL unit quaternion corresponding to rotation matrix ${}_{G}^{I}\mathbf{R}$ that represents the rotation from the global $\{G\}$ to the IMU frame $\{I\}$; 
\textcolor{black}{We use the JPL convention to maintain consistency with the majority of prior work in filter-based VIO. For a detailed comparison between the JPL and Hamiltonian conventions, we refer readers to \cite{sola2017quaternion}.
}
${}^{G}\mathbf{p}_{I}$, ${}^{G}\mathbf{v}_{I}$, and ${}^{G}\mathbf{f}_i$ are the IMU position, velocity, and $i$-th feature position in $\{G\}$;
$\mathbf{b}_{g}$ and $\mathbf{b}_a$ are the gyroscope and accelerometer biases; 
$\mathbf{x}_{T_i}=[{}_{G}^{I_{i}}\Bar{q}^{\top}~ {}^{G}\mathbf{p}_{I_{i}} ^{\top} ]^\top$ is the $i$-th cloned pose;
$t_d$  and $\{{}_{C}^{I}\bar{q}^{\top} , {}^{C}\mathbf{p}_I^{\top}\}$ are respectively the time offset and extrinsic calibration between the camera and IMU, while $\boldsymbol{\zeta}$ is the camera intrinsic parameters.

Inspired by~\cite{Wu2024Thesis}, the special ordering of the state variables~\eqref{eq:state} is discovered based on our computational complexity analysis and follows the principle: 
The states to be frequently marginalized are placed at the end (i.e., IMU clones and feature positions),
while the variables that will be kept in the state vector (such as the current navigation states and calibration parameters) are stored in the front.
By doing so, we can leverage the triangular structure of the SRF to efficiently insert and delete (marginalize) variables from the state vector without expensive re-ordering, which will become clear in the ensuing $\sqrt{\mathrm{VINS}}$ operations (see Sections~\ref{sec:clone} and \ref{sec:cam_meas}).


\subsection{IMU Propagation and Integration}
\label{sec:vins_prop}

A standard 6-axis IMU provides local linear acceleration and angular velocity measurements
$\mathbf a_m$ and $\bm\omega_m$ at time $t_k$:
\begin{align}
    \label{eq:imu_mea_a}
     \mathbf{a}_m(t_k) & = \mathbf{a}(t_k) - {}_{G}^{I}\mathbf{R}(t_k){}^{G}\mathbf{g} + \mathbf{b}_a(t_k) + \mathbf{n}_a(t_k)
    \\
      \label{eq:imu_mea_b}
    \boldsymbol{\omega}_m(t_k) & = \boldsymbol{\omega}(t_k) + \mathbf{b}_g(t_k) + \mathbf{n}_g(t_k)
\end{align}
where $\boldsymbol{\omega}(t_k)$
and $\mathbf{a}(t_k)$ are the true angular velocity and linear acceleration in the IMU local frame $\{I\}$, 
${}^G\mathbf{g} \simeq [0,0,-9.8]^{\top}$ is the known global gravitational acceleration,
and  $\mathbf{n}_g$ and $\mathbf{n}_a$ are the zero-mean white Gaussian noise. 
These IMU measurements are used to drive the inertial navigation system (INS) 
whose continuous-time kinematics is given by~\cite{Chatfield1997}:
\begin{align}
    {}_{G}^{I}\dot{\bar{q}}(t) 
    & =\frac{1}{2}\boldsymbol{\Omega}(\boldsymbol{\omega}(t)) {}_{G}^{I}\bar{q}(t)
    \label{eq:imumodel_q}
    \\
    {}^G\dot{\mathbf{p}}_{I}(t) & = {}^G\mathbf{v}_{I}(t)
     \label{eq:imumodel_p}
    \\
    {}^G\dot{\mathbf{v}}_{I}(t) & = {}^{I(t)}_G\mathbf{R}^{\top} \mathbf{a}(t)
     \label{eq:imumodel_v}
    \\
    \dot{\mathbf{b}}_g(t) & = \mathbf{n}_{wg} (t)
     \label{eq:imumodel_bg}
    \\
    \dot{\mathbf{b}}_a(t) & = \mathbf{n}_{wa}(t)
    \label{eq:imumodel_ba}
\end{align}
where
$\boldsymbol{\Omega(\boldsymbol \omega}) = 
\begin{bmatrix}
    -\lfloor\boldsymbol{\omega}\rfloor & \boldsymbol{\omega} \\
    -\boldsymbol{\omega}^{\top} & 0
\end{bmatrix}$ and here $\lfloor \cdot \rfloor$ is the skew-symmetric matrix.
We have modeled the gyroscope and accelerometer biases as random walk, with their time derivatives represented as white Gaussian processes, denoted by $\mathbf{n}_{wg}$ and $\mathbf{n}_{wa}$, respectively.
Integration of the above continuous-time inertial kinematics from $t_k$ to $t_{k+1}$ yields the following discrete-time motion model:
%
%
\begin{align}
{}^{I_{k+1}}_G \mathbf{R} &= {}^{I_{k+1}}_{I_k} \Delta\mathbf{R}~{}^{I_k}_G \mathbf{R} \label{eq:RRR}\\
{}^G \mathbf{p}_{I_{k+1}}
&= \scalemath{0.9}{
{}^G \mathbf{p}_{I_k}+{}^G \mathbf{v}_{I_k}\Delta T +\frac{1}{2}{}^G \mathbf{g} \Delta T^2 + {}^{I_k}_{G} \mathbf{R}^\top {}^{I_k}\boldsymbol{\alpha}_{I_{k+1}}} \label{eq:dp}
\\
{}^G \mathbf{v}_{I_{k+1}} &= {}^G \mathbf{v}_{I_k}+{}^G \mathbf{g} \Delta T + {}^{I_k}_{G} \mathbf{R}^\top  {}^{I_k} \bm \beta_{I_{k+1}} 
\label{eq:dv}
\\
\mathbf{b}_{g_{k+1}} &= \mathbf{b}_{{g_k}} + \mathbf n_{g_k}\\
\mathbf{b}_{{a_{k+1}}} &= \mathbf{b}_{a_k} + \mathbf n_{a_k} \label{eq:bak1}
\end{align}
where 
$\Delta T = t_{k+1}- t_k$, 
$\mathbf n_{g_k}=\int_{t_k}^{t_{k+1}} \mathbf{n}_{w g}(u) ~du$, 
and $\mathbf n_{a_k} = \int_{t_k}^{t_{k+1}} \mathbf{n}_{w a}(u) ~du$.
Importantly, ${}^{I_{k+1}}_{I_k} \Delta\mathbf{R}$ is the gyro preintegration in $[t_k,t_{k+1}]$,
while ${}^{I_k} \bm \alpha_{k+1}$ and ${}^{I_k} \bm \beta_{k+1} $ are the preintegration of acceleration measurements.
\textcolor{black}{Due to space limitations, we do not provide full derivations of the preintegration terms. However, readers can refer to related works on IMU preintegration theory (see \cite{Lupton2012TRO, Forster2015RSS, Eckenhoff2019IJRR}). In particular, \cite{Eckenhoff2019IJRR} and its associated technical report offer detailed derivations of these terms}:
\begin{align}
\nonumber
{}^{I_{k+1}}_{I_k} \Delta\mathbf{R} 
&
= 
-\int_{t_k}^{t_{k+1}} {}^{I}\boldsymbol{\omega} (t_\tau) d\tau
\\
{}^{I_k} \bm \alpha_{I_{k+1}} &=
\int_{t_k}^{t_{k+1}} \int_{t_k}^{s} {}^k_{u}\Delta\mathbf{R}\left(\mathbf{a}_m(u)- \mathbf{b}_a(u)-\mathbf{n}_a(u)\right) du ds \notag \\
{}^{I_k} \bm \beta_{I_{k+1}} &=
\int_{t_k}^{t_{k+1}} {}^k_{u}\Delta\mathbf{R}\left(\mathbf{a}_m(u)- \mathbf{b}_a(u) -\mathbf{n}_a(u)\right) du \notag
\end{align}
With the above model model~\eqref{eq:RRR}-\eqref{eq:bak1},
we can perform the SRF propagation as described in Section~\ref{sec:sr-prop}.


\begin{figure*}
\centering
\includegraphics[width=0.8\linewidth]{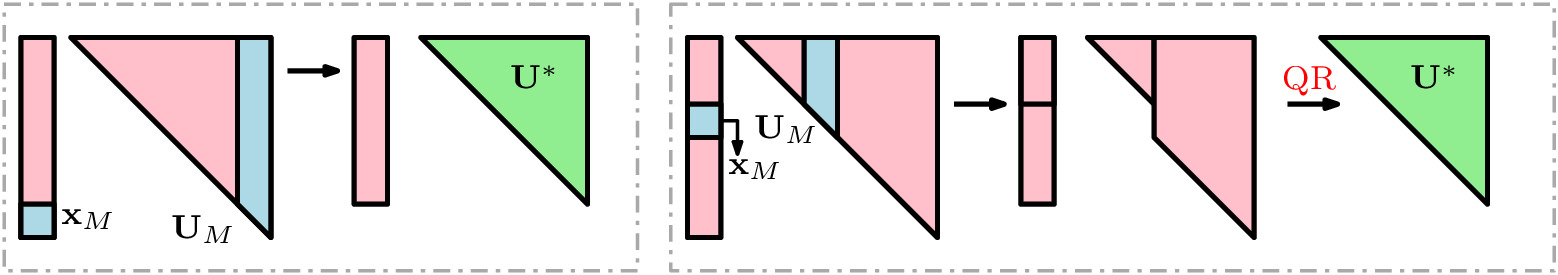}
\caption{
An illustration of the marginalization process of 
$
\sqrt{\mathrm{VINS}}
$ with different state ordering.
The blue blocks represent the marginalized state $\mathbf{x}_M$ and its covariance block $\mathbf{U}_M$, while the pink blocks represent the other states. 
The left side shows an extreme case where the marginalized state is at the bottom, allowing direct extraction of the upper-triangular square-root covariance $\mathbf{U}^*$. The right side shows the marginalized state in the middle, requiring an additional QR operation. Placing marginalized states further toward the bottom improves structure preservation and reduces the QR operation cost.
}
\label{fig:marg}
\end{figure*}

\subsection{Stochastic Cloning and Marginalization}
\label{sec:clone}

The proposed $\sqrt{\mathrm{VINS}}$ adopts the sliding-window filtering methodology to control the ever-growing computational cost over time while (sub-)optimally fusing all the measurements available in the current window.
This is achieved by performing efficient stochastic cloning~\cite{Roumeliotis2002ICRAa} and state marginalization,
which is akin to that in the multi-state constraint Kalman filter (MSCKF)~\cite{Mourikis2007ICRA,Geneva2020ICRA} {\em but} operating on the square-root covariance matrix $\mathbf U$ [see \eqref{eq:UtU}].

To that end, we first partition the state vector~\eqref{eq:state} into 
the most recent navigation state $\mathbf{x}_I$, the state to be marginalized ${\mathbf{x}}_M$, and the remaining state ${\mathbf{x}}_R$,
and their corresponding columns of the square-root covariance matrix are $\mathbf{U}_I$, $\mathbf{U}_M$ and $\mathbf{U}_R$, respectively.
%
As the system propagates forward, we perform stochastic cloning~\cite{Roumeliotis2002ICRAa} to augment the state and square-root covariance  with those  of the most recent or cloned camera/IMU pose in order to better process its corresponding measurements at a later time.
%
%
As the square-root covariance for each state corresponds to the respective column blocks in $\mathbf{U}$~\cite{Peng2023TR,Wu2024Thesis}, we specifically copy the cloned pose $\mathbf{x}_T$ and its corresponding column $\mathbf{U}_T$, to get the augmented  state and square-root covariance as follows:
\begin{align}
    \begin{bmatrix}
    \mathbf{x}_I^\top & \mathbf{x}_M^\top & \mathbf{x}_R^\top
\end{bmatrix}
  ~~ 
   & \xrightarrow{\text{cloning}}
    ~~
\begin{bmatrix}
    \mathbf{x}_I^\top & \mathbf{x}_T^\top & \mathbf{x}_M^\top & \mathbf{x}_R^\top 
\end{bmatrix} \label{eq:clone-x}
\\
    \begin{bmatrix}
         \mathbf{U}_I &  \mathbf{U}_M &  \mathbf{U}_R
    \end{bmatrix}
        ~~ 
   & \xrightarrow{\text{cloning}}
    ~~
      \begin{bmatrix}
         \mathbf{U}_I &  \mathbf{U}_T & \mathbf{U}_M &  \mathbf{U}_R
    \end{bmatrix} \label{eq:clone-U}
\end{align}
%
This can be proven by demonstrating that $\mathbf{U}^\top\mathbf{U} = \mathbf{P}$ through the aforementioned process for $\mathbf{U}$ and the covariance augmentation in the conventional MSCKF. 
The detailed proof can be found in Section 2.3 of \cite{Peng2023TR}.

%
To bound the size of the state so as to control computational cost, 
we marginalize and remove the unwanted state $\mathbf{x}_M$ from \eqref{eq:clone-x} and its corresponding block $\mathbf{U}_M$ from~\eqref{eq:clone-U}.
At the same time, because this marginalization destroys the upper triangular structure of the square-root covariance,
we perform QR decomposition to ensure the resulting upper triangular $\mathbf U^*$:
\begin{align}
\nonumber
    &\begin{bmatrix}
         \mathbf{U}_I & \mathbf{U}_T &  \mathbf{U}_M &  \mathbf{U}_R
    \end{bmatrix}
       \xrightarrow{\text{marginalization}} \\
    &
    \label{eq:marg}
    \begin{bmatrix}
         \mathbf{U}_I &  \mathbf{U}_T &   \mathbf{U}_R
    \end{bmatrix}
     \stackrel{\text{QR}}{=} 
    \mathbf{Q}^*
    \begin{bmatrix}
        \mathbf{U}^*
        \\
        \mathbf{0}
    \end{bmatrix}
\end{align}



At this point, we stress the significance of the proposed special ordering of the state [see \eqref{eq:state}].
During marginalization, by ordering the state with non-marginalized variables (e.g., $\mathbf{x}_I$ and $\mathbf{x}_{cb}$) at the top and marginalized variables (e.g., features) at the bottom, the resulting square-root covariance remains close to upper-triangular, leading to significant computational savings. 
Figure~\ref{fig:marg} illustrates this process: The left side shows an extreme and most ideal case where the marginalized state $\mathbf{x}_M$ is at the end of the state vector. After removing its corresponding column $\mathbf{U}_M$ from the square-root covariance, we directly obtain the upper-triangular square root covariance $\mathbf{U}^*$ without further operations. 
By contrast, the right side of Figure~\ref{fig:marg} shows a general case where the marginalized state in the middle of the state vector, requiring an additional QR operation to restore the upper-triangular structure.
Note that, although we would like to place the marginalized state at the very bottom so as to completely avoid QR operations, in practice, this is hard to guarantee -- for example, when dealing with lost track features, the marginalized state might not always be at the very end -- and QR is still necessary. 
However, we only need to perform QR on a sub-block of the square root covariance matrix. In comparison to the square-root information filter~\cite{Wu2015RSS} that requires processing the entire matrix, our approach remains significantly more efficient thanks to the nature of SRF marginalization.

\subsection{Structure-Aware Efficient Measurement Update}
\label{sec:cam_meas}

Assume that at time $t_k$  we have established a visual feature track within the current sliding window
in which we assume $M$ features in total being detected and tracked in the time window of $[t_{k-c},\cdots, t_k]$.
A feature $\mathbf f_i$ ($i=1,\cdots, M$) observed at time $t_\kappa$ has the following nonlinear bearing measurement model 
by projecting its 3D position onto the 2D image plane, which is further linearized for SRF update [see~\eqref{eq:state}]:
\begin{align}    \label{eq:cam_linear_full}
\mathbf{z}_{i,\kappa} &=  h_d(h_p ({}^{C_\kappa}\mathbf f_i),\boldsymbol{\zeta}) + \mathbf{n}_{i,\kappa} ~~\Rightarrow \\
\mathbf {r}_{i,\kappa}      
    &\simeq
    \mathbf{H}_{I,\kappa} \tilde{\mathbf{x}}_{I,\kappa|\kappa-1}
    +
    \mathbf{H}_{cb,\kappa}\Tilde{\mathbf{x}}_{cb,\kappa|\kappa-1} + 
    \mathbf{H}_{f_i,\kappa}\Tilde{\mathbf{f}}_{i,\kappa|\kappa-1} + \mathbf n_{i,\kappa}
    \notag \\
    & =
     \label{eq:cam_linear}
    \mathbf{H}_{x,\kappa}\Tilde{\mathbf{x}}_{x,\kappa|\kappa-1} +  \mathbf{H}_{f,\kappa}\Tilde{\mathbf{x}}_{f,\kappa|\kappa-1} + \mathbf n_{i,\kappa}
\end{align}
where $h_d$ and $h_p$ are the intrinsic distortion and projection functions, respectively. 
Note that $h_d$ is general and can support any camera model (e.g., radial-tangential and equidistant~\cite{Hartley2003book}).
In a typical visual tracking scenario, $M$ can be large (in the order of thousands) and easily result in prohibitive computational cost if processing them without discrimination.
As such, we categorize them into the SLAM features that are included in the state and the MSCKF features that will not be kept,
and accordingly, partition the measurement set into different subsets depending on the observed features' categories similar as~\cite{Li2013IJRR,Wu2024Thesis}:
\begin{align} \label{eq:meas-subsets}
\{\mathbf z_{i,\kappa}\}_{i=1,\cdots,M}^{\kappa=k-c,\cdots, k}    
= \underbrace{\mathcal Z_N \cup \mathcal Z_S}_{\mathcal Z_{SLAM}} \cup
\underbrace{\mathcal Z_M \cup \mathcal Z_U}_{\mathcal Z_{VIO}}
\end{align}
where $\mathcal Z_{SLAM}$ includes all the measurements available in the window related to the SLAM features that either are already in the state vector (with respect to $\mathcal Z_{S}$) or are to be initialized into the state (using $\mathcal Z_{N}$),
and $\mathcal Z_{VIO}$ contains the remaining measurements related to the MSCKF features, some of which are to be processed at the current time ($\mathcal Z_{M}$) while the others are delayed for future update ($\mathcal Z_{U}$).
Except $\mathcal Z_{U}$, 
we perform different update strategies for these measurement subsets to ensure efficiency and accuracy.

Without loss of generality (w.l.o.g), we first consider using the measurement set $\mathcal Z_N$ [see \eqref{eq:meas-subsets}] to initialize a new SLAM feature $\mathbf{f}_{N}$ into the state vector~\eqref{eq:state}.
Although we can use the delayed initialization to efficiently estimate the feature position $\hat{\mathbf f}_{N,\kappa|\kappa}$, its square-root covariance would be more challenging.
To this end, we first stack all the linearized measurement residuals of $\mathcal Z_N$ [see \eqref{eq:cam_linear}] and have:\footnote{We here have dropped off the time index for brevity and employed the subscript symbol ``$N$'' to refer to the association with $\mathcal Z_N$.
Note that similar notations are used for the other measurement sets.}
\begin{align} \label{eq:rN}
    \mathbf{r}_{N} = 
    \mathbf{H}_{N,x}\tilde{\mathbf x}_x
    +    \mathbf{H}_{N,f}\tilde{\mathbf{f}}_N
    +
    \mathbf{n}_{N}
\end{align}
Note that $\mathcal Z_N$ typically has more than the minimum required measurements for feature initialization in order to obtain better accuracy; that is, the initialization is over constrained and $\mathbf{H}_{N,f}$ is tall.
This implies that over-constrained $\mathcal Z_N$ also encompasses information that constrains the existing state $\mathbf x_x$ (not just the new feature $\mathbf f_N$), which we seek to optimally utilize too.
As evident from \eqref{eq:rN} that the range space of $\mathbf{H}_{N,f}$ essentially constrains the new feature while its nullspace constrains the existing state, we perform QR decomposition on $\mathbf{H}_{N,f}$ in order to initialize the square-root covariance
as well as extract constraints on $\mathbf x_x$ from $\mathcal Z_N$:
   \begin{align}
   \label{eq:slam_init_pqr}
        \mathbf{H}_{N,f}
       &  \stackrel{\text{QR}}{=} 
       \begin{bmatrix}
           \mathbf{Q}_1 & \mathbf{Q}_2
       \end{bmatrix}
        \begin{bmatrix}
           \mathbf{0} 
            \\
         \mathbf J_{N,f} 
        \end{bmatrix}
   \end{align}
where $\mathbf{Q}_1$ spans the nullspace of $\mathbf{H}_{N,f}$ and $\mathbf{Q}_2$ spans the range space. 
It is important to note that during the above QR decomposition, we use the P-QR decomposition introduced in~\cite{Peng2024ICRA} to permute $\mathbf J_{N,f} $ into a lower triangular matrix, instead of upper triangular.
This is different from the traditional SLAM initialization~\cite{Li2014} and can better preserve the structure of the SRF to improve efficiency. 
Left multiplication~\eqref{eq:cam_linear} by $[\mathbf{Q}_1 ~ \mathbf{Q}_2]^{\top}$ yields:
   \begin{align}
           \begin{bmatrix}
           \mathbf{Q}_1^{\top} \\ \mathbf{Q}_2^{\top}
       \end{bmatrix}
       \mathbf{r}_N
       &=
       \begin{bmatrix}
           \mathbf{Q}_1^{\top} \\ \mathbf{Q}_2^{\top}
       \end{bmatrix} 
       \begin{bmatrix}
           \mathbf H_{N,f} & \mathbf H_{N,x}
       \end{bmatrix}
       \begin{bmatrix}
       \tilde{\mathbf{f}}_{N}
        \\
       \tilde{\mathbf{x}}_x
       \end{bmatrix}
       +  \begin{bmatrix}
           \mathbf{Q}_1^{\top} \\ \mathbf{Q}_2^{\top}
       \end{bmatrix}
       \mathbf{n}_N 
       \notag
       \\ 
       \Rightarrow
      \begin{bmatrix}
           \bm\gamma_N \\ \bm\varsigma_N
       \end{bmatrix} 
       &:=
       \begin{bmatrix}
           \mathbf{0} & \bm\Gamma_{N,x}
            \\
           \mathbf J_{N,f} & \mathbf J_{N,x}
        \end{bmatrix}
        \begin{bmatrix}
       \tilde{\mathbf{f}}_{N}
        \\
       \tilde{\mathbf{x}}_x
       \end{bmatrix}
       +  \begin{bmatrix}
           \bm\eta_N \\ \bm\xi_N
       \end{bmatrix}
 \label{eq:slam_init}
   \end{align}
where $\begin{bmatrix} \bm\Gamma_{N,x} \\ \mathbf J_{N,x}\end{bmatrix} := \begin{bmatrix} \mathbf Q_1^\top \\ \mathbf Q_2^\top \end{bmatrix} \mathbf H_{N,x}$.
Clearly, as the top linear system: $\bm\gamma_N=\bm\Gamma_{N,x} \Tilde{\mathbf{x}}_x + \bm\eta_N$, 
depends only on the existing state, it should be utilized as normal measurements like $\mathcal Z_S$ to update the state.
On the other hand, the bottom linear system: $\bm\varsigma_N=\mathbf J_{N,f} \tilde{\mathbf{f}}_{N} + \mathbf J_{N,x} \Tilde{\mathbf{x}}_x + \bm\xi_N$, $\bm\xi_N \sim \mathcal N(\mathbf 0, \bm\Omega)$, is used to efficiently initialize the square-root covariance with the new feature being included in the state as follows:
   \begin{align}
   \label{eq:slam_init_cov}
       \mathbf{U}'
       & =
       \begin{bmatrix}
           \mathbf{U} & -\mathbf{U}\mathbf J_{N,x}^{-\top}\mathbf{J}_{N,f}^{-\top}
           \\
           \mathbf{0} & {\bm\Omega}^\frac{1}{2} \mathbf{J}_{N,f}^{-\top}
       \end{bmatrix}
   \end{align}
   This can easily be proved by performing $\mathbf{U}'^\top\mathbf{U}'=\mathbf{P}'$ to show its equivalence to the SLAM feature initialization in EKF~\cite{Li2014thesis}.
   
   


For $\mathcal Z_S$, as all the measurements relate to SLAM features~$\mathbf{x}_f$ that are already in the state vector, we simply stack all the measurement residuals~\eqref{eq:cam_linear} for batch update:
\begin{align}
\label{eq:slam_up}
\mathbf{r}_{S} &= \mathbf{H}_{S,x} \tilde{\mathbf{x}}_x +  \mathbf{H}_{S,f} \tilde{\mathbf x}_{f} + \mathbf{n}_{S}
\end{align} 

In contrast,  $\mathcal Z_M$ correspond to the MSCKF features $\mathbf{f}_M$ which are not in the state but we still want to utilize their constraints on the state.
To this end, we project the linearized measurement residual onto the left nullspace $\mathbf{N}$ of the feature Jacobian $\mathbf{H}_{M,f}$, effectively eliminating the feature dependency and improving efficiency by avoiding the need to keep these features in the state vector~\cite{Mourikis2007ICRA}.
\begin{align}
\label{eq:nullspace1}
    \mathbf{r}_{M} & = 
        \mathbf{H}_{M,x} \tilde{\mathbf{x}}_x 
        +
         \mathbf{H}_{M,f}  \tilde{\mathbf{f}}_M
        +
        \mathbf{n}_{M}
\\
\label{eq:nullspace2}
\mathbf{N}^{\top} \mathbf{r}_{M} &=  \mathbf{N}^{\top}\mathbf{H}_{M,x} \tilde{\mathbf{x}}_x +
\mathbf{N}^{\top}\mathbf{n}_{M}
\\
\Rightarrow ~
\label{eq:msckf_up}
\bm\gamma_{M} &:= \bm\Gamma_{M,x}\tilde{\mathbf{x}}_x  + \bm\eta_{M}
\end{align} 


We have thus far built the linearized measurement residuals [see \eqref{eq:slam_init}, \eqref{eq:slam_up} and \eqref{eq:msckf_up}] for different measurement sets~\eqref{eq:meas-subsets} available in the current sliding window, which are stacked in a compact form for the SRF batch update:
%
\begin{align}
    \begin{bmatrix}
        \bm\gamma_{M}
        \\
        \bm\gamma_N
        \\
        \mathbf{r}_{S}
    \end{bmatrix}
    &= 
    \begin{bmatrix}
        \bm\Gamma_{M,x} & \mathbf{0} &  \mathbf{0} \\
        \bm\Gamma_{N,x} &  \mathbf{0} &  \mathbf{0}  \\
        \mathbf{H}_{S,x} & \mathbf{H}_{S,f}  &  \mathbf{0}\\
    \end{bmatrix}
    \begin{bmatrix}
        \Tilde{\mathbf{x}}_{x} \\
        \tilde{\mathbf{x}}_{f}\\
        \tilde{\mathbf{f}}_{N}
    \end{bmatrix}
    + \begin{bmatrix}
        \bm\eta_M
        \\
        \bm\eta_N
        \\
        \mathbf{n}_S
    \end{bmatrix} 
    \label{eq:stacked-res}
    %
\end{align}
Using \eqref{eq:stacked-res}, we perform the efficient LLT-based SRF update [see \eqref{eq:srf_up_cov} and \eqref{eq:srf_up_mean}] for the proposed $\sqrt{\mathrm{VINS}}$.

In particular, we  take full advantage of the sparse structure of the measurement Jacobian matrix in~\eqref{eq:stacked-res} when computing the matrix $\mathbf{C}$ and its Cholesky decomposition to obtain $\mathbf F$ [see \eqref{eq:update}].
%
%
%
Specifically, as the Jacobian of \eqref{eq:stacked-res} has zeros (the third block column) corresponding to  the new SLAM feature $\mathbf f_N$, we group the first two block columns corresponding to $\mathbf x_x$ and $\mathbf x_f$ as $\mathbf H_{xf}$.
In computing $\mathbf{C}$, as $\mathbf{U}^{\top}$ is lower-triangular, we {\em only} need to multiply the non-zero blocks of $\mathbf{H}_{xf}$  with its corresponding $\mathbf{U}_{xf}^{\top}$,
which naturally leads to {\em sparse} 
$\mathbf{C} = 
\begin{bmatrix}
        \mathbf{C}_{xf} & \mathbf{0} 
        \\
        \mathbf{0} & \mathbf{I}
    \end{bmatrix}
$, where $
\mathbf{C}_{xf} = 
\mathbf{U}_{xf}\mathbf{H}^{\top}_{xf}
\mathbf{R}^{-1} \mathbf{H}_{xf}\mathbf{U}_{xf}^{\top}
+\mathbf{I}
$.
As a result, we perform LLT {\em only} on $\mathbf{C}_{xf}$ rather than the full matrix, 
allowing us to efficiently solve for $\mathbf{F}_{xf}$ and thus {\em sparse} 
$\mathbf{F} = 
\begin{bmatrix}
        \mathbf{F}_{xf} & \mathbf{0} 
        \\
        \mathbf{0} & \mathbf{I}
    \end{bmatrix}
$.
We also exploit this special structure of $\mathbf{F}$ when performing the SRF update [see \eqref{eq:srf_up_cov}]; 
that is, we only need to compute $\mathbf{F}_{xf}^{-\top}$ along with its corresponding upper triangular matrix, avoiding processing the entire matrix.
Note that this saving can be significant in particular when multiple new SLAM features are initialized in the current window which is often the case in practice.



\color{black}


%

\subsubsection{Outlier Rejection}
\label{sec:chi2}

Measurement outliers are inevitable in practice.
Compared with an information-form estimator, we have covariance available at each update, thus, we employ the standard $\chi^2$ test to reject them by computing the following Mahalanobis distance:
\begin{align} \label{eq:dm}
    d_m =
     \mathbf{r}^{\top}
     \left(
     \mathbf{H}\mathbf{U}^{\top}\mathbf{U}\mathbf{H}^{\top} + \mathbf{R}
     \right)^{-1}
     \mathbf{r} 
\end{align}
where $\mathbf r$ and $\mathbf R$  generically refer to the measurement residual and noise covariance [see \eqref{eq:general-meas}].
Thanks to the special structure of the measurement Jacobian $\mathbf H$ in our case [e.g., see \eqref{eq:stacked-res}], we are able to compute the Mahalanobis distance~\eqref{eq:dm} very efficiently.
In particular, as the measurement residual of the MSCKF features $\bm\gamma_M$ is not related to features [see \eqref{eq:stacked-res}],
the following most expensive operation is carried out as:
\begin{align}
\label{eq:UH}
    \mathbf{U}\mathbf{H}^{{\top}}
    &=
    \begin{bmatrix}
    \mathbf{U}_1 & \mathbf{U}_2
    \\
    \mathbf{0} & \mathbf{U}_3
\end{bmatrix}
\begin{bmatrix}
   \bm\Gamma_{M,x}^\top 
    \\ 
    \mathbf{0} 
\end{bmatrix}
=
\begin{bmatrix}
      \mathbf{U}_1 \bm\Gamma_{M,x}^\top 
      \\
      \mathbf{0} 
\end{bmatrix}
\end{align}
Clearly, given the upper-triangular structure of $\mathbf{U}$ and the unique structure of the measurement Jacobian $\bm\Gamma_{M,x}$, we only need to compute $\mathbf{U}_1\bm\Gamma_{M,x}^\top$, instead of multiplying the measurement Jacobian  with the full $\mathbf{U}$.
This holds exactly the same for the measurement residual of the new SLAM features $\bm\gamma_N$ [see \eqref{eq:stacked-res}].
For the SLAM feature measurement residual $\mathbf r_S$ \eqref{eq:slam_up},
the sparsity of the measurement Jacobian which only relates to the corresponding camera pose and the observed features, also allows us to leverage the upper-triangular structure of $\mathbf{U}$ to compute the Mahalanobis distance $d_m$~\eqref{eq:dm} very efficiently.
Note that the computed $\mathbf{U}\mathbf{H}^{{\top}}$~\eqref{eq:UH} can also be {\em re-used} in the proposed SRF update to avoid repeated computation.


\subsubsection{SLAM Feature Propagation}

Anchor feature representation, which parameterizes features in a moving anchor frame $\{A\}$, has  demonstrated robust performance in VINS~\cite{Qin2018TRO, Geneva2020ICRA, Huai2021ICRA}, and thus is also used in this work.
The anchor frame in principle can be selected as any camera frame within the sliding window that observes the feature.
When the anchor frame is moving out of the current sliding window, a new anchor frame is selected and the features are transformed or anchored to this new frame.  
Though it appears to be straightforward, proper handling of the feature's covariance in the new anchor frame is often overlooked.
To this end, 
we build the relation between the feature represented in the \textit{old} anchor frame $\{A_1\}$, denoted by ${}^{A_1}\mathbf{f}$, and in the new anchor frame $ \{A_2\}$, denoted by ${}^{A_2}\mathbf{f}$.
This can be derived based on the fact that the global position of the static feature, ${}^{G}\mathbf{f}$, remains unchanged regardless of the choice of anchor frame:
\begin{align}
    {}^{G}\mathbf{f} = {}_{G}^{A_1}\mathbf{R}^{\top}{}^{A_1}\mathbf{f}
    + {}^{G}\mathbf{p}_{A_1}
    =
     {}_{G}^{A_2}\mathbf{R}^{\top}{}^{A_2}\mathbf{f}
    + {}^{G}\mathbf{p}_{A_2}
    \label{eq:anchor_change}
\end{align}
where 
$ \mathbf x_{A_i} =
\{{}_{A_i}^{G}\mathbf{R}, {}^{G}\mathbf{p}_{A_i}\} 
$ ($i=1,2$)
denotes the orientation and the position of the $i$-th anchor frame in the global frame.
Linearization of the above equations yields:
\begin{align}
     {}^{G}\Tilde{\mathbf{f}} 
     & =
     \nonumber
     \mathbf{H}_{A_1}\tilde{\mathbf{x}}_{A_1} + \mathbf{H}_{f_{A_1}}
     {}^{A_1}\tilde{\mathbf{f}}
     =
     \mathbf{H}_{A_2}\tilde{\mathbf{x}}_{A_2} + \mathbf{H}_{f_{A_2}}
     {}^{A_2}\tilde{\mathbf{f}}
     \\
     \Rightarrow ~
    {}^{A_2}\tilde{\mathbf{f}}
    \label{eq:anchor_change}
    &=
    \mathbf{H}_{f_{A_2}}^{-1}(\mathbf{H}_{f_{A_1}}
     {}^{A_1}\tilde{\mathbf{f}}
     +
      \mathbf{H}_{A_1}\tilde{\mathbf{x}}_{A_1}
      -
      \mathbf{H}_{A_2}\tilde{\mathbf{x}}_{A_2})
\end{align}
where $\mathbf{H}_{A_i}$ denotes the Jacobians with respect to anchor poses and  $\mathbf{H}_{f_{A_i}}$ denotes the Jacobians with respect to anchor feature represented in different anchors.
\textcolor{black}{Leveraging the Jacobians in equation above and covariance of the old anchor feature, old anchor pose and new anchor pose, we perform QR-based covariance propagation introduced in Section~\ref{sec:sr-prop}} to obtain the covariance matrix of the new anchor feature ${}^{A_2}\mathbf{f}$.

\subsection{Online Calibration}
To make the proposed $\sqrt{\mathrm{VINS}}$ more robust and easy-to-use,
we also perform online calibration of the camera-IMU spatiotemporal parameters and camera intrinsics.
To calibrate the extrinsic and intrinsic parameters, we include them in the state vector and build the measurement model that also depends on them~\eqref{eq:cam_linear_full}. With that, we perform SRF update of  these parameters along with the other states (see~\cite{Geneva2020ICRA}).
However, calibrating the time offset $t_d$ between the IMU and camera is not that straightforward and requires more care~\cite{Li2014IJRR}.
Specifically, when a new image is available at time $t_k$, 
given the prior estimate of the time offset $\hat{t}_d$, 
we propagate the IMU/camera pose up to time $(t_k+\hat{t}_d)$ with the IMU readings and obtain the prior pose estimate $\hat{\mathbf x}_T(t_k+\hat t_d)$ [see \eqref{eq:state}].
%
Note that the proposed $\sqrt{\mathrm{VINS}}$  clones the {\em true} pose ${\mathbf x}_T(t_k+t_d)$ in order to process the visual measurements of this image at a later time (along with the other measurements in the current window).
It is clear that the estimation error $\tilde{t}_d$ of the time offset contributes to the error of this cloned pose $\tilde{\mathbf x}_T(t_k+t_d)$,
which should be carefully compensated in its covariance via SRF propagation (which is often overlooked in practice). 
\textcolor{black}{As this is not the main contribution of our work, we refer the reader to \cite{Li2014IJRR} and our open-source implementation for further details. 
}

\subsection{Remarks}

\begin{algorithm} [b]
\small
\raggedright
\caption{$\sqrt{\mathrm{VINS}}$}
\label{algo1}
\textbf{Propagation:}
\begin{compactitem}
\item Propagate the state to $\mathbf{x}(t+t_d)$ by Eq.~(\ref{eq:srf_prop_cov}))~ (\textit{skip QR})
\end{compactitem}
\textbf{Clone and Marginalization:}
\begin{compactitem}
\item 
Cloning the latest IMU pose [Section~\ref{sec:clone}] (\textit{skip QR}) 
\item Anchor change for SLAM features by Eq.~(\ref{eq:srf_prop_cov})~(\ref{eq:anchor_change}) (\textit{skip QR})
\item 
Marginalize oldest clone and lost tracked SLAM features [Eq.~(\ref{eq:marg})] (\textcolor{black}{\textbf{QR}}) 
\end{compactitem}
\textbf{Measurements Formulation:} 
\\
Using the tracked features to formulate measurements and perform outlier rejection [Section~\ref{sec:chi2}]
to prepare for updates.
\begin{compactitem}
\item MSCKF features via nullspace projection [Eq.~\eqref{eq:msckf_up}]
\item SLAM feature initialization [Eq.~\eqref{eq:slam_init_pqr},\eqref{eq:slam_init},\eqref{eq:slam_init_cov}]
\item SLAM features re-observation [Eq.~\eqref{eq:slam_up}]
\end{compactitem}
\textbf{SRF update:}
\begin{compactitem}
\item Stack meas. [Eq.~\eqref{eq:stacked-res}] and do SRF update
for both the state and the square-root covariance
[Eq. \eqref{eq:srf_up_cov},\eqref{eq:srf_up_mean}]  (\textcolor{black}{\textbf{LLT}}).
\end{compactitem}
\end{algorithm}

At this point, we have presented the proposed $\sqrt{\mathrm{VINS}}$
and its main steps are summarized in Algorithm~\ref{algo1}. 
We here highlight a few design choices to take advantage of the structure of the system so as to improve the efficiency.

\begin{itemize}
    \item 
    \textit{Special state ordering:}
   The order of the state variables [see \eqref{eq:state}] is especially designed to speed up update and marginalization process.
   For example, $\mathbf{x}_I$ and $\mathbf{x}_{cb}$ are prioritized at the top as they would not be marginalized, while the clones  $\mathbf{x}_C$ are ordered from the latest to oldest for easy marginalization of the oldest one.
    The feature state, $\mathbf{x}_f$, is placed at the end, because: (i) SLAM features are marginalized frequently,
(ii) this ordering allows better sparsity when computing $\mathbf{C}$ in Eq.~\eqref{eq:update} and 
(iii) it ensures $\mathbf{U}$ is still upper-triangular after initializing a new SLAM feature, thereby preserving the structure and improving efficiency.

\item 
\textit{Delayed QR and LLT operation:} 
As introduced in previous sections, QR decomposition is used during propagation, state cloning, marginalization, and anchor changes when an anchor feature is used.
The efficiency can be further improved by skipping the QR step in propagation and cloning, with QR being performed only once during marginalization
before update to maintain an upper triangular square root covariance to facilitate the subsequent update process.
Meanwhile, measurements from different feature types (e.g., MSCKF, old SLAM, and new SLAM features) are processed individually and then stacked for a single, efficient LLT-based square-root update, as shown in Algorithm~\ref{algo1}.
%

\item \textit{Avoiding repeated computation:}
When performing SRF update, the computing of $\mathbf{C}$ is needed.
As mentioned in Section~\ref{sec:chi2}, the computation of $\mathbf{U}\mathbf{H}^\top$ during outlier rejection can be reused to enhance efficiency and avoid redundant calculations.
In certain cases where outlier rejection is not performed by $\chi^2$ test, we first stack all the measurements Jacobians and then fully leverage the sparsity to perform measurement compression~\cite{Mourikis2007ICRA}.
The triangular structure and symmetry are then exploited to efficiently compute $\mathbf{U}\mathbf{H}^\top$ and $\mathbf{C}$.
\end{itemize}

\section{Robust and Fast Dynamic Initialization}
\label{sec:dy_init}

Dynamic initialization is critical to ensure robust and smooth VINS operation without discontinuity. 
Successful initialization  not only requires an initial solution but also ensures stable and continuous motion tracking afterwards. 
However, fast initialization with minimal data is challenging due to limited visual parallax and low signal-to-noise ratios. 
To address this challenge, we propose a novel dynamic initialization method, tightly integrated with the proposed SRF, providing robust and efficient initialization, even in minimal conditions.

Specifically, our initialization consists of two main steps:
(i) \textit{Feature-less initialization} efficiently recovers the initial velocity and gravity without recovering 3D feature positions to improve both robustness and efficiency;
and (ii) \textit{SRF refinement} optimizes the IMU state, features, and covariance through efficient iterative SRF updates, ensuring smooth and robust VINS operation.
Note that in our initialization, as common practice, we  assume the IMU biases are reasonably accurate (e.g., obtained from prior calibration) and 
the camera-IMU calibration and time offset are known.

\begin{figure}
\centering
\includegraphics[width=0.99\linewidth]{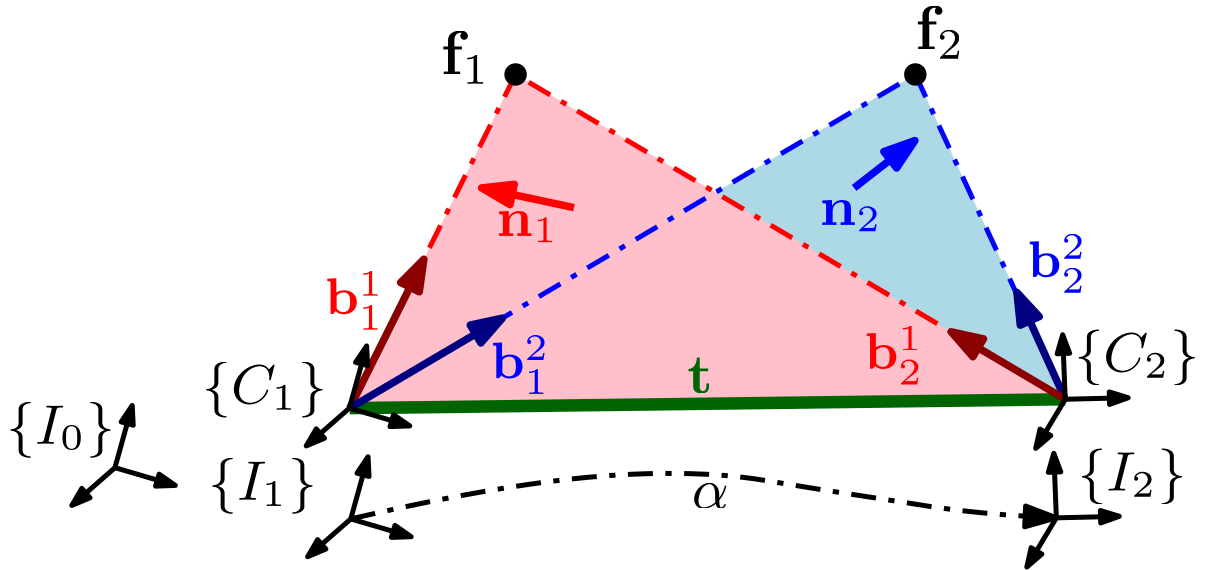}
\caption{
An illustration of two-view geometry for featureless initialization. 
The pink and blue planes represent two epipolar planes formed by the features $\mathbf{f}_1$ and $\mathbf{f}_2$, and their corresponding camera frames $\{C_1\}$ and $\{C_2\}$. 
The bearing observations are denoted by $\mathbf{b}$, while $\mathbf{n}_1$ and $\mathbf{n}_2$ represent the normal directions of the two epipolar planes. 
$\{I_0\}$ indicates the initial reference frame and $\mathbf{t}$ is the direction of the relative pose.
}
\label{fig:linear_init}
\end{figure}

\subsection{Feature-less Initialization}

To maximize efficiency, we first recover {\em only} the minimal states (i.e., without unobservable global position and yaw),
including the local velocity (${}^{I_0}\mathbf{v}_{I_0}$) and  gravity (${}^{I_0}\mathbf{g}$):
\begin{align}\label{eq:init-states}
    {}^{I_0}\mathbf{x} = 
    \begin{bmatrix}
        {}^{I_0}\mathbf{v}_{I_0}^{\top}
        &
        {}^{I_0}\mathbf{g}^{\top}
    \end{bmatrix}^{\top}
\end{align}
where $\{{I_0}\}$ is the initial IMU frame of reference. 
We parameterize the gravity with the minimal two parameters $(\alpha, \beta)$:
${}^{I_0}\mathbf{g}^\top =  |\mathbf{g}|\mathbf{\begin{bmatrix}
     \cos\alpha\sin\beta & \sin\alpha\sin\beta & \cos\beta
\end{bmatrix}}^\top$.
We now explain how to formulate the linear system to solve for ${}^{I_0}\mathbf{x}$.

\subsubsection{IMU-induced local motion}
\label{sec:init_imu}


Leveraging the inertial preintegration,
we integrate the IMU measurements in the time interval $[t_0,t_k]$ to compute the relative motion in the local, instead of the global, frame of reference $\{{I}_0\}$ [see \eqref{eq:RRR}, \eqref{eq:dp} and \eqref{eq:dv}]:
\begin{align}
{}^{I_{k}}_{I_0} \mathbf{R}
\label{eq:dR}
&:=  {}^{I_{k}}_{I_0} \Delta\mathbf{R}\\
\label{eq:dP}
{}^{I_0} \mathbf{p}_{I_{k}}
&:= 
{}^{I_0}\mathbf{v}_{I_0}\Delta T_k + \frac{1}{2}{}^{I_0}\mathbf{g} \Delta T_k^2 + {}^{I_0} \bm \alpha_{I_{k}}  
\\
{}^{I_0} \mathbf{v}_{I_{k}}
&:= 
\label{eq:dV}
{}^{I_0}\mathbf{v}_{I_0} + {}^{I_0}\mathbf{g} \Delta T_k + {}^{I_0} \bm \beta_{I_k} 
\end{align}
where $\Delta T_k = (t_k- t_0)$ is the time span for integration,
${}^{I_{k}}_{I_0} \Delta\mathbf{R},{}^{I_0} \bm \alpha_{I_{k}}$ and ${}^{I_0} \bm \beta_{I_k}$ are obtained by IMU preintegration~\cite{Eckenhoff2019IJRR}.
These can also be computed by rotating the orientation and velocity with ${}^{I_0}_G \mathbf{R}$ and computing the relative position change ${}^{I_0}\mathbf{p}_{I_{k}} = {}^{I_0}_{G}\mathbf{R}({}^G \mathbf{p}_{I_{k}} - {}^G \mathbf{p}_{I_0})$.

\subsubsection{Camera-induced up-to-scale relative motion}
\label{sec:init_cam}
\color{black}
With the visual measurements available in the time window $[t_o,t_k]$, we efficiently derive the relative motion direction, {\em without} relying on 3D features.
The result is summarized as follows:
\begin{proposition}
Given camera bearing measurements 
$\{\mathbf b_k^i\}$ of environmental features with known camera–IMU extrinsics, 
the up-to-scale relative translation direction $\mathbf t$ of the platform 
is obtained as the eigenvector corresponding to the smallest eigenvalue of
\begin{align}
\label{eq:init-eig-decom}
     \mathbf M = \sum_{i=1}^M \mathbf n_i \mathbf n_i^\top
\end{align}   
where $\mathbf n_i = \lfloor {}^{I_0}\mathbf b_1^i \rfloor \, {}^{I_0}\mathbf b_2^i$ 
denotes the normal vector of the epipolar plane formed by two camera observations represented in first IMU frame $\{I_0\}$.
\end{proposition}
\color{black}

\color{black}

\begin{proof}
As shown in Figure~\ref{fig:linear_init}, consider an environmental feature $\mathbf{f}_i$ observed by two camera frames $\{C_1\}$ and $\{C_2\}$. 
Each observation $\mathbf{b}_{k}^{i}$  is a 2D bearing measurement from the $k$-th camera frame to the $i$-th feature.
Based on the two-view geometry~\cite{Hartley2004},
an epipolar plane is thus formed with the two bearings and the relative pose $\mathbf{t}$ (e.g., see $\triangle C_1C_2\mathbf f_1$ in Figure~\ref{fig:linear_init}). 
To solve for the relative motion, we exploit the fact that the normal of \textit{any} epipolar plane is perpendicular to~$\mathbf{t}$, which allows us to formulate an eigenvalue problem of $\mathbf{M}$ as defined in Eq.~\ref{eq:init-eig-decom}.
Note that in this example, $\mathbf{n}_i$ is the normal direction of the epipolar plane $\triangle C_1C_2\mathbf f_i$. The direction $\mathbf{t}$ corresponds to the eigenvector associated with the smallest eigenvalue of $\mathbf{M}$. 
\color{black}
%
Specifically,  the normal direction $\mathbf{n}_i$ of the epipolar plane $\triangle C_1C_2\mathbf f_i$
can be computed as (see Figure~\ref{fig:linear_init}):
\begin{align}
    \mathbf{n}_i = 
    \lfloor
    {}^{I_0}\mathbf{b}_1^{i}
    \rfloor
    {}^{I_0}\mathbf{b}_2^{i}
\end{align}
where
$ {}^{I_0}\mathbf{b}_k^{i} = 
{}_{I_i}^{I_0}\mathbf{R}{}_{C}^{I}\mathbf{R}
\mathbf{b}_k^{i}
$
is the bearing measurement of the feature $\mathbf f_i$ rotated to the local IMU frame $\{{I}_0\}$.
Note that the extrinsic calibration between the IMU and camera is assumed to be known.
Geometrically, as the relative translation $\mathbf{t}$ is the intersection of all the epipolar planes with its corresponding camera frames,
all the normals  of these epipolar planes are perpendicular to  $\mathbf{t}$, such that:
\begin{align}
\mathbf{n}_i^{\top}\mathbf{t}=0 ~\Rightarrow~ \mathbf{n}_i\mathbf{n}_i^{\top}\mathbf{t}=\mathbf{0} 
~\Rightarrow~ \underbrace{\sum_{i=1}^{M} \mathbf{n}_i\mathbf{n}_i^{\top}}_{\mathbf M} ~\mathbf{t}=\mathbf{0} 
\end{align}
which implies that the relative motion $\mathbf{t}$ is an eigenvector corresponding to the zero eigenvalue (or null vector) of matrix $\mathbf{n}_i\mathbf{n}_i^\top$ and thus matrix $\mathbf{M}$, if noise free~\cite{Kneip2013ICCV}.
%
Although 
the $3\times 3$ matrix $\mathbf{M}$ ideally is rank-2   
with the null space spanned by $\mathbf{t}$,
given noisy measurements in practice, $\mathbf{M}$ would become full-rank and thus $\mathbf{t}$ should be the eigenvector corresponding to the smallest eigenvalue.
%
As such, determining the relative translation direction becomes an eigenvalue problem. 
Performing eigenvalue decomposition of   $\mathbf{M}$ yields:
\begin{align} \label{eq:init-eig-decom}
    \mathbf{M} \overset{eig.}{=} \begin{bmatrix} \mathbf{t} & \mathbf{e}_1 & \mathbf{e}_2 \end{bmatrix} 
   \mathbf{Diag} (\lambda_t, \lambda_1, \lambda_2)
   \begin{bmatrix} \mathbf{t} & \mathbf{e}_1 & \mathbf{e}_2 \end{bmatrix}^\top 
\end{align}
where
the diagonal matrix of eigenvalues arranged in an ascending order.
We thus have computed the up-to-scale translation (relative motion direction) $\mathbf t$, 
as well as the eigenvectors $\mathbf{e}_1$ and $\mathbf{e}_2$ corresponding to the larger eigenvalues $\lambda_1$ and $\lambda_2$.
\end{proof}

For convenience, we define $\mathbf{e} := [\mathbf{e}_1, \mathbf{e}_2]$, which lies in the plane orthogonal to  $\mathbf{t}$ (i.e., $\mathbf t \perp \mathbf e$) and will be used in the subsequent linear system formulation [see \eqref{eq:linear_final}].


\subsubsection{Linear system}
\label{sec:init_linear}
\color{black}
By combining the IMU-induced relative motion (with scale) and the 
camera-induced direction (without scale), we can then eliminate the unknown 
scale and recover the initial state through a linear system. 
The result is summarized below:
\begin{proposition}
Given IMU preintegration over $[t_0,t_k]$, camera bearings with known camera–IMU extrinsics, 
and the relative translation direction $\mathbf t$ (with its orthogonal basis $\mathbf e$, such that $\mathbf e^\top \mathbf t = 0$),  
the initial state ${}^{I_0}\mathbf{x}$ can be obtained by solving the linear system
\[
    \mathbf e^\top \mathbf A \, {}^{I_0}\mathbf x = \mathbf e^\top \mathbf b ,
\]
where $\mathbf A$ and $\mathbf b$ are constructed from IMU integration and extrinsics.
\end{proposition}
\color{black}
\begin{proof}


Note that in the following, we use the two keyframes $\{C_1\}$ and $\{C_2\}$ in Figure~\ref{fig:linear_init} to illustrate our derivations:
\begin{align}
    s  \mathbf{t} 
    &=
    {}^{I_0}\mathbf{p}_{C_2} - {}^{I_0}\mathbf{p}_{C_1}\\
&=
{}^{I_0}\mathbf{p}_{I_2} - {}^{I_0}\mathbf{p}_{I_1}
+
\left(
{}_{I_2}^{I_0}\mathbf{R} - {}_{I_1}^{I_0}\mathbf{R}
\right) 
{}^{I}\mathbf{p}_C 
\label{eq:constraint}
\end{align}
where $s$ is the unknown scale, and 
${}^{I_0}\mathbf{p}_{I_2}$ and ${}^{I_0}\mathbf{p}_{I_1}$ are computed by integration with IMU measurements [see~\eqref{eq:dp}]:
\begin{align}
{}^{I_0} \mathbf{p}_{I_{1}}
&= 
{}^G \mathbf{p}_{I_0}+{}^{I_0} \mathbf{v}_{I_0}
\Delta T_1 +\frac{1}{2}{}^{I_0} \mathbf{g} \Delta T_1^2 + {}^{I_0}
\boldsymbol{\alpha}_{I_{1}}
\\
{}^{I_0} \mathbf{p}_{I_{2}}
&= 
{}^G \mathbf{p}_{I_0}+{}^{I_0} \mathbf{v}_{I_0}
\Delta T_2 +\frac{1}{2}{}^{I_0} \mathbf{g} \Delta T_2^2 + {}^{I_0}
\boldsymbol{\alpha}_{I_{2}}
\end{align}
Substituting these into \eqref{eq:constraint}, we have:
\begin{align}
    s  \mathbf{t} &=
\underbrace{ 
\begin{bmatrix}
\Delta{T}_2 - 
\Delta{T}_1 
&
\frac{1}{2}
(
\Delta{T}^2_2 - \Delta{T}^2_1 
)
\end{bmatrix}
}_{-\mathbf A}
    \begin{bmatrix}
        {}^{I_0}\mathbf{v}_{I_0} \\ 
        {}^{I_0}\mathbf{g}
    \end{bmatrix} \nonumber \\
    &\quad + 
    \underbrace{   {}^{I_0}\boldsymbol{\alpha}_{I_2} - {}^{I_0}\boldsymbol{\alpha}_{I_1}
        + \left({}_{I_2}^{I_0}\mathbf{R} - {}_{I_1}^{I_0}\mathbf{R}\right) {}^{I}\mathbf{p}_C
        }_{\mathbf b}
    \label{eq:linear_derive} \\
    \Rightarrow~& \quad
    s  \mathbf{t}
    +
    \mathbf{A} {}^{I_0}\mathbf{x}
    =
    \mathbf{b}
    \label{eq:scale_linear}
\end{align}
%


To further improve the efficiency in solving the above linear system~\eqref{eq:scale_linear}, 
we remove the dependency of the unknown scale~$s$ 
by projecting~\eqref{eq:scale_linear} onto the nullspace of $\mathbf{t}$, such that:
\begin{align}
    \mathbf{e}^{\top} 
    \mathbf{A} {}^{I_0}\mathbf{x}
    =
    \mathbf{e}^{\top} 
    \mathbf{b}
    \label{eq:linear_final}
\end{align}
where we have employed the fact that 
$\mathbf{e}^{\top} \mathbf{t} = \mathbf{0}$.
\end{proof}
Interestingly, as the eigenvectors $\mathbf{e}$ have already been computed during the eigenvalue decomposition of $\mathbf{M}$~\eqref{eq:init-eig-decom}, we effectively eliminate the need for redundant computations.
\color{black}
It is worth mentioning that the proposed method is naturally applicable to static motion, as $s$ will be close to zero, the ambiguity of translation direction does not negatively impact the system.
\color{black}
Importantly, as compared to traditional initialization methods (e.g., \cite{Dong2012IROS,Qin2018TRO}), the proposed linear system~\eqref{eq:scale_linear} does {\em not} recover 3D feature positions, offering the following key advantages:
\begin{itemize}
    \item {\em Robust to outliers:} Without estimating 3D features, the influence of outliers on the solution is minimized.
    \item {\em Robust to small parallax due to low excitation:} Since feature positions are not explicitly recovered, even if some features are near rank-deficient, it does not lead to degeneracy.
    \item {\em Better efficiency:} Without including 3D features, the problem reduces to a 6$\times$6 linear system which is solved in constant time. The complexity with respect to the number of keyframes, $k$, is $\mathcal{O}(k^2)$, but since $k$ is typically small (e.g., 3-5 frames), this is marginal. The overall complexity is dominated by the linear dependence on the number of features, $\mathcal{O}(M)$. In contrast, methods that recover 3D features  have the complexity of $\mathcal{O}(M^3)$, making our approach significantly more efficient.
\end{itemize}

\subsection{SRF Refinement}
\label{sec:init_refine}

Due to measurement noise, the linear system solution~\eqref{eq:linear_final} -- though fast -- inevitably would be inaccurate in practice
especially given an extremely small initialization window. 
We thus perform iterative updates with SRF to refine the {\em full} (not minimal) state and the corresponding covariance, ensuring the successful operation of subsequent VINS.
%
The full state vector of initialization is defined as:
\begin{align}
    \mathbf{x}_\text{full}
    &= 
    \begin{bmatrix}
        \mathbf{x}_{K}^{\top}
        &
         \mathbf{x}_{F}^{\top}
    \end{bmatrix}^{\top}
    \\
    \mathbf{x}_{K} & = 
    \begin{bmatrix}
    \mathbf{x}_{I_0}^{\top} & \dots & \mathbf{x}_{I_N}^{\top} 
    \end{bmatrix}^{\top}
    \\
      \mathbf{x}_{F} & = \begin{bmatrix}
    {}^{G}\mathbf{f}_{1}^{\top}
    & \dots & 
    {}^{G}\mathbf{f}_{M}^{\top}
      \end{bmatrix}^{\top}
    \\
    \mathbf{x}_{I_k} &=
    \begin{bmatrix}
    {}^{I_k}_{G}\bar{q}^\top &
    {}^{G}\mathbf{p}_{I_k}^\top &
    {}^{G}\mathbf{v}_{I_k}^\top &
    \mathbf{b}_{g,k}^\top &
    \mathbf{b}_{a,k}^\top
    \end{bmatrix}^\top \label{eq:state_imu}
\end{align}
where $ \mathbf{x}_{K}$ denotes the keyframe states and $\mathbf{x}_{F}$ is the key features during initialization.
In contrast to the commonly used dynamic initialization pipeline (e.g., \cite{Qin2018TRO}), which solves the VI-BA problem using an optimization solver in the information form, our approach leverages {\em iterative} SRF update and offers several advantages: (i) seamless integration with the SRF-based VINS, (ii) improved computational efficiency and reduced numerical issues while enabling iterative error refinement from relinearization, and (iii) direct covariance computation without the need for matrix inversion.

In particular, we first use the IMU readings to propagate the keyframe states and their corresponding square-root covariance~\eqref{eq:cov111}, and \color{black} use the obtained pose to triangulate the features, \color{black}then use the feature-bearing measurements to iteratively update the states. 
For clarity, we use a single key feature, ${}^{G}\mathbf{f}_j$, as an example. Its measurements are stacked and linearized as [see \eqref{eq:cam_linear_full} and \eqref{eq:cam_linear}]:
\begin{align}
\mathbf{z} &= 
\mathbf{h}(
\mathbf{x}_{K},{}^{G}\mathbf{f}_{j}
) + \mathbf{n}
\\
\Rightarrow ~~
    \mathbf{r}^{(l)} &\simeq
    \mathbf{H}_{K}^{(l)}
    \tilde{\mathbf{x}}_{K}^{(l)}
    +
    \mathbf{H}_{F}^{(l)}
    {}^{G}\tilde{\mathbf{f}}_{j}^{(l)}
    +
    \mathbf{n}
\end{align}
where 
$\mathbf{H}_{K}^{(l)}$ and $\mathbf{H}_{F}^{(l)}$ are the Jacobians, and $\tilde{\mathbf{x}}_{K}^{(l)}$ and ${}^{G}\tilde{\mathbf{f}}_{j}^{(l)}$ are the error states for keyframes and key feature at the $l$-th iteration.
\textcolor{black}{
We perform a QR decomposition on the feature Jacobian $\mathbf{H}_F$, followed by a left multiplication of the full system, which results in two decoupled systems. This step is analogous to the process shown in Eqs.~\eqref{eq:slam_init_pqr} and \eqref{eq:slam_init}, where we apply the same technique. For brevity, we omit the detailed derivation here and refer readers to that earlier discussion.
}
\begin{align} \label{eq:two-sys}
     \begin{bmatrix}
        \mathbf{r}_1^{(l)} \\
        \mathbf{r}_2^{(l)}
     \end{bmatrix}
     = \begin{bmatrix}
             \mathbf{H}_{K,1}^{(l)}
\\
 \mathbf{H}_{K,2}^{(l)}
     \end{bmatrix}
    \tilde{\mathbf{x}}_{K}^{(l)}
    +
    \begin{bmatrix}
      \mathbf{H}_{F,1}^{(l)}
      \\
    \mathbf{0}
    \end{bmatrix}
    {}^{G}\tilde{\mathbf{f}}_{j}^{(l)}
    +
    \mathbf{n}
\end{align}
For each key feature, we construct the linear systems as described above. Thus, with all the features, we stack its corresponding bottom system \eqref{eq:two-sys} and perform the proposed LLT-based SRF update for the key state (i.e., $\mathbf{x}_K^{(l+1)}
= \mathbf{x}_K^{(l)} + d\mathbf{x}^{(l)}_K
$ [see~\eqref{eq:update}]). 
Then we can update each key feature estimate with the top system of \eqref{eq:two-sys} as:
\begin{align}
    {}^{G}\hat{\mathbf{f}}_{j}^{(l+1)}
    = 
      {}^{G}\hat{\mathbf{f}}_{j}^{(l)}
      + 
           \mathbf{H}_{F,1}^{{(l)}^{-1}}
           (
           \mathbf{r}_1^{(l)} -
           \mathbf{H}_{K,1}^{(l)} 
           d{\mathbf{x}}_{K}^{(l)}
           )
\end{align}
This process of linearization and update of state estimate is repeated until convergence,
while we update the square-root covariance only after being converged.
\color{black}

Note that this iterative SRF refinement allows us to directly obtain the initial covariance for use in the subsequent $\sqrt{\mathrm{VINS}}$. 
We have chosen to initialize not only the IMU states but also retain key features $\mathbf{x}_F$ within the initialization window, along with their covariance. 
This ensures smooth, tightly-coupled integration into $\sqrt{\mathrm{VINS}}$, enhancing its robustness,
as discussed in the following.


We first stress that our initialization consists of two main steps, \textit{without} and \textit{with} 3D features involved:
\begin{enumerate}
    \item In the feature-less initialization step, our primary focus is to estimate the initial velocity and gravity as quickly as possible, \textit{without} recovering 3D feature positions. 
    \item During SRF refinement step, we seamlessly transition to $\sqrt{\mathrm{VINS}}$, by initializing the inertial state \textit{with} the SLAM features and their corresponding covariance. 
These features can be immediately tracked in the subsequent VINS process and used for instant updates, preserving all information within the initialization window.
\end{enumerate}
\color{black}
Note that traditional methods take a completely opposite approach: they first recover the initial state using 3D feature positions and then refine only the IMU states in a second step.
In contrast, our proposed feature-less linear system offers robustness to outliers and small parallax, while significantly improving efficiency.
For the refinement step, our approach retains features from the initialization window, enabling immediate updates when new measurements of the same features become available.
This ensures that all information is preserved, tightly coupled, and seamlessly transferred to the VINS module.
Traditional methods, by comparison, solve only for the 15-DOF IMU state during refinement, which requires additional time and frames for VINS to initialize new features and operate effectively.
Furthermore, SRF directly accesses the state covariance without needing to invert the information matrix—a process commonly required in BA-based initialization methods—thereby avoiding potential numerical instability, particularly in challenging initialization scenarios.
As a result, by leveraging these advantages, our initialization method achieves high robustness and efficiency.
\color{black}

\begin{table}
\centering
\caption{
Simulation parameters and prior standard deviations for measurement perturbations.
}
\label{tab:sim_params}
\begin{adjustbox}{width=0.99\columnwidth,center}
\begin{tabular}{ccccc} \toprule
\textbf{Parameter} & \textbf{Value} & \textbf{Parameter} & \textbf{Value} \\ \midrule
Gyro. White Noise & 2.0e-4 & Gyro. Rand. Walk & 2.0e-5 \\
Accel. White Noise & 5.0e-4 & Accel. Rand. Walk & 4.0e-4 \\
Cam Freq. (Hz) & 10 & IMU Freq. (Hz) & 400 \\
Num. Clones & 11 & Tracked Feat. & 100 \\
Max. MSCKF Feat. & 40  & Max. SLAM Feat. & 50  \\\bottomrule
\end{tabular}
\end{adjustbox}
\end{table}

\begin{table}
\centering
\caption{RMSE values for orientation (deg.) and position (m) based on 200 runs on UD-ARL with different estimators.
}
\resizebox{0.99\columnwidth}{!}{%
\begin{tabular}{@{}ccccc@{}}
\toprule
\textbf{Methods} & \textbf{EKF}                      & \textbf{SRF (QR)}   & \textbf{SRF (LLT)}                      & \textbf{SRIF}            \\ \midrule
\textbf{double}  & 0.957 / 0.146                     & 0.957 / 0.146    & 0.956 / 0.146                     & 0.957 / 0.146                                                                                       \\
\textbf{float}   &  0.960 / 0.146 & 0.959 / 0.146  & 0.958 / 0.146 & \color{red}1.045 / 0.174\color{black} \\ \bottomrule
\end{tabular}
}
\label{tab:sim}
\end{table}

\begin{figure}
\centering
\includegraphics[width=.95\linewidth]{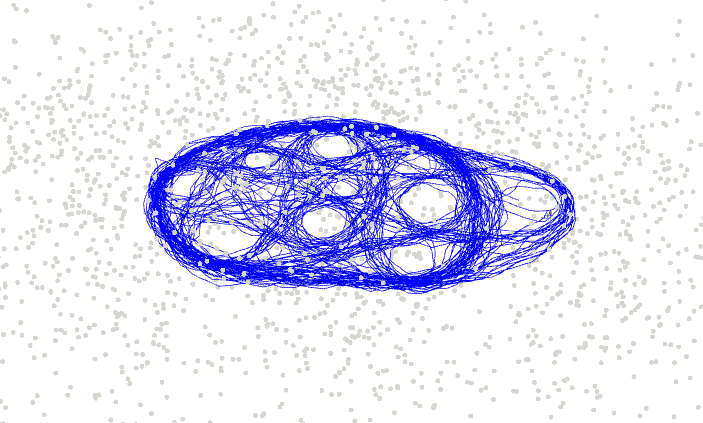}
\caption{Simulated 2.4km UD-ARL trajectory.
}
\label{fig:traj}
\end{figure}

\begin{figure}
\centering
\includegraphics[width=0.99\linewidth]{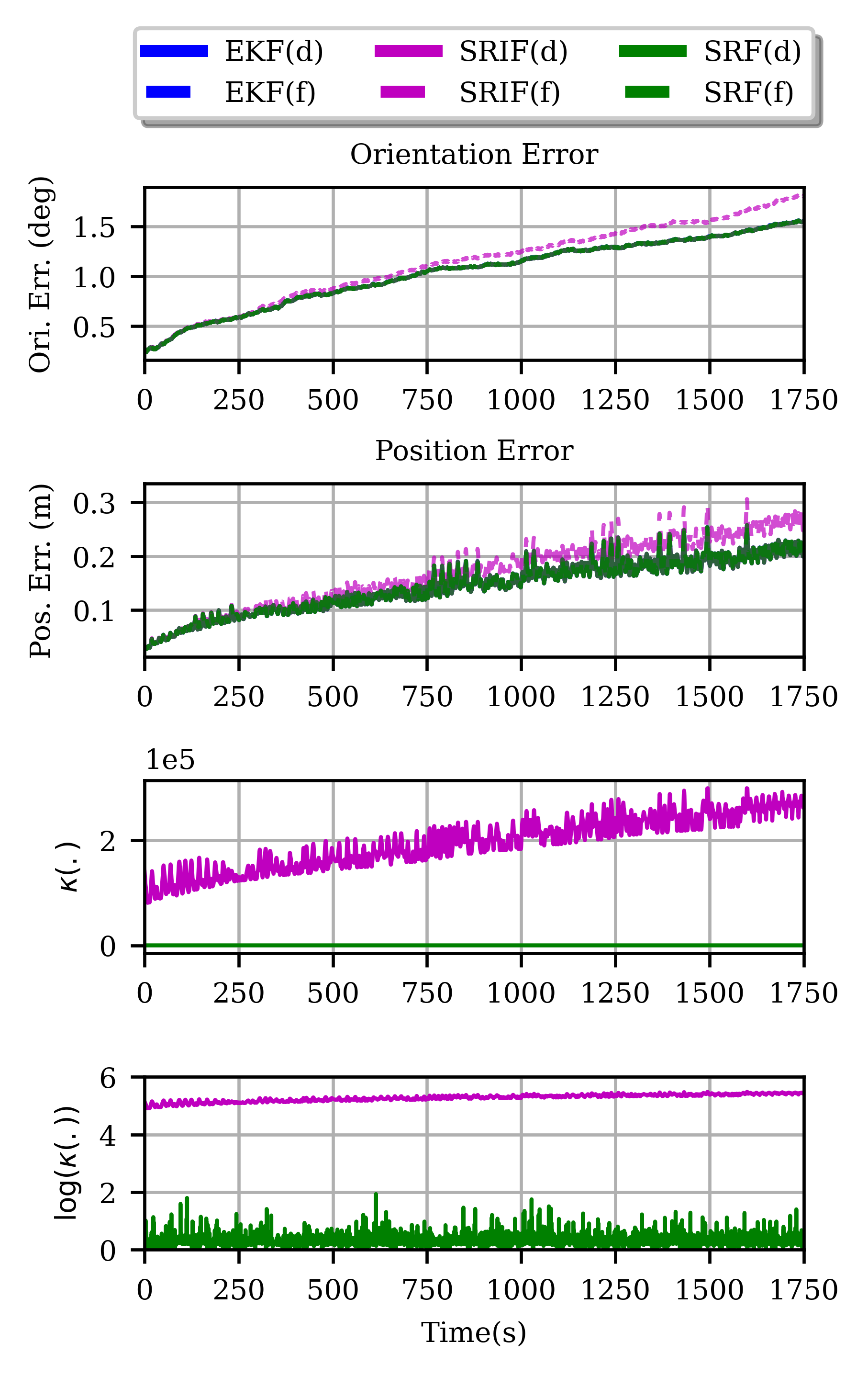}
\caption{
\textbf{Top:} Orientation/position errors of different estimators performed on UD-ARL dataset. `d' is for double; `f' is for float. 
While most estimators perform similarly and are hard to distinguish from the plot, SRIF(f) shows a clear drop in accuracy over time.
\textbf{Bottom:} 
Condition numbers of the square-root information matrix (purple line) and the Cholesky decomposed matrix $\mathbf{C}$ (green line, see Eq.~\eqref{eq:update}), presented in both standard (scientific) and logarithmic scales.
}
\label{fig:sim-condition}
\end{figure}

\section{Numerical Studies of LLT-based SRF}
\label{sec:exp_srf}
We now present the simulation results of the proposed LLT update method for SRF, showcasing its improved numerical stability and efficiency.
To ensure a fair comparison across estimators, we use OpenVINS~\cite{Geneva2019ICRA}, which implements an EKF-based approach as the baseline (denoted EKF(d)), along with a float version (EKF(f)). Additionally, we evaluate double- and float-precision versions of the square-root information filter (SRIF(d) and SRIF(f)) and the proposed square-root covariance $\sqrt{\mathrm{VINS}}$ {(SRF(d) and SRF(f))}.

\subsection{
Accuracy and Robustness of SRF
}

We first aim to demonstrate that $\sqrt{\mathrm{VINS}}$ offers improved numerical stability compared to other forms of estimators. 
To evaluate this, we use a 30-minute, 2.4 km UD-ARL trajectory (see Figure~\ref{fig:traj}) and generate realistic visual bearing and inertial measurements, as summarized in Table~\ref{tab:sim_params}.

Table~\ref{tab:sim} \color{black}reports~\color{black} the average Root Mean Square Error (RMSE) of different estimators based on 200 Monte Carlo runs. 
In Figure~\ref{fig:sim-condition}, the top two plots show the orientation and position errors of various estimators with both double and float precision. The bottom two plots the standard and logarithmic condition numbers of the square-root information matrix (purple) and the $\mathbf{C}$ matrix [see Eq.~\eqref{eq:update}] for the LLT-based update over time (green). 
Given the covariance matrix $\mathbf{P}$, the square-root information matrix $\mathbf{R}$ is $\mathbf{R}^{\top}\mathbf{R} = \mathbf{P}^{-1}$.
Note that this figure only presents the LLT-based SRF results, as the different update methods are expected to exhibit nearly identical accuracy performance, as shown in Table~\ref{tab:sim}.

Note that the authors have carefully reported the condition numbers of these two matrices to evaluate numerical stability.
For SRIF, the most challenging step is inverting the square-root information matrix $\mathbf{R}$, so we report its condition number. 
For the LLT-based SRF, two numerical challenges arise: i) the Cholesky decomposition of $\mathbf{C}$, which requires a positive definite matrix, and ii) the inversion of $\mathbf{F}$.
But $\mathbf{F}$ is the square root of $\mathbf{C}$, we plot the condition number of $\mathbf{C}$ as it also reflects that of $\mathbf{F}$.
From Figure~\ref{fig:sim-condition}, it is clear that the condition number of $\mathbf{C}$ (green line) remains stable, with a magnitude close to 1 and a maximum below $10^2$, demonstrating the improved numerical stability of the proposed SRF.
Intuitively, $\mathbf{F}^{-\top}$ serves as the transition matrix between $\mathbf{U}_{k|k-1}$ and $\mathbf{U}_{k|k}$ (the square-root covariance before and after propagation). As long as the measurement uncertainty (i.e., camera measurement noise) is not significantly smaller than the state uncertainty after propagation, we expect $\mathbf{F}^{-\top}$ to be close to the identity matrix and well-conditioned. Consequently, $\mathbf{C}$ will share the same benefit, which is almost always the case in practical VINS applications.

From the figure, we also observe as the condition number of $\mathbf{R}$ grows larger than $2e^{5}$, both orientation and position errors of SRIF(f) start showing a degraded performance compared to other estimators.
This can also be seen in Table~\ref{tab:sim}, the float SRIF is inaccurate with large RMSE values.
This is likely due to the numerical issue when performing inversion on ill-conditioned $\mathbf{R}$ to solve for state update under limited machine precision (see Chapter 3.5.1 in \cite{Golub2013Book}). 
In contrast, the covariance-form estimators, both EKF and the proposed SRF, no matter which update method is used (i.e., P-QR or LLT) demonstrated consistent performance regardless of using double or float. This is evident from the comparable RMSE values in Table~\ref{tab:sim}, as well as the consistent error trends in Figure~\ref{fig:sim-condition}.

\subsection{
Efficiency of LLT-based SRF Update
}

\begin{figure}
\centering
\includegraphics{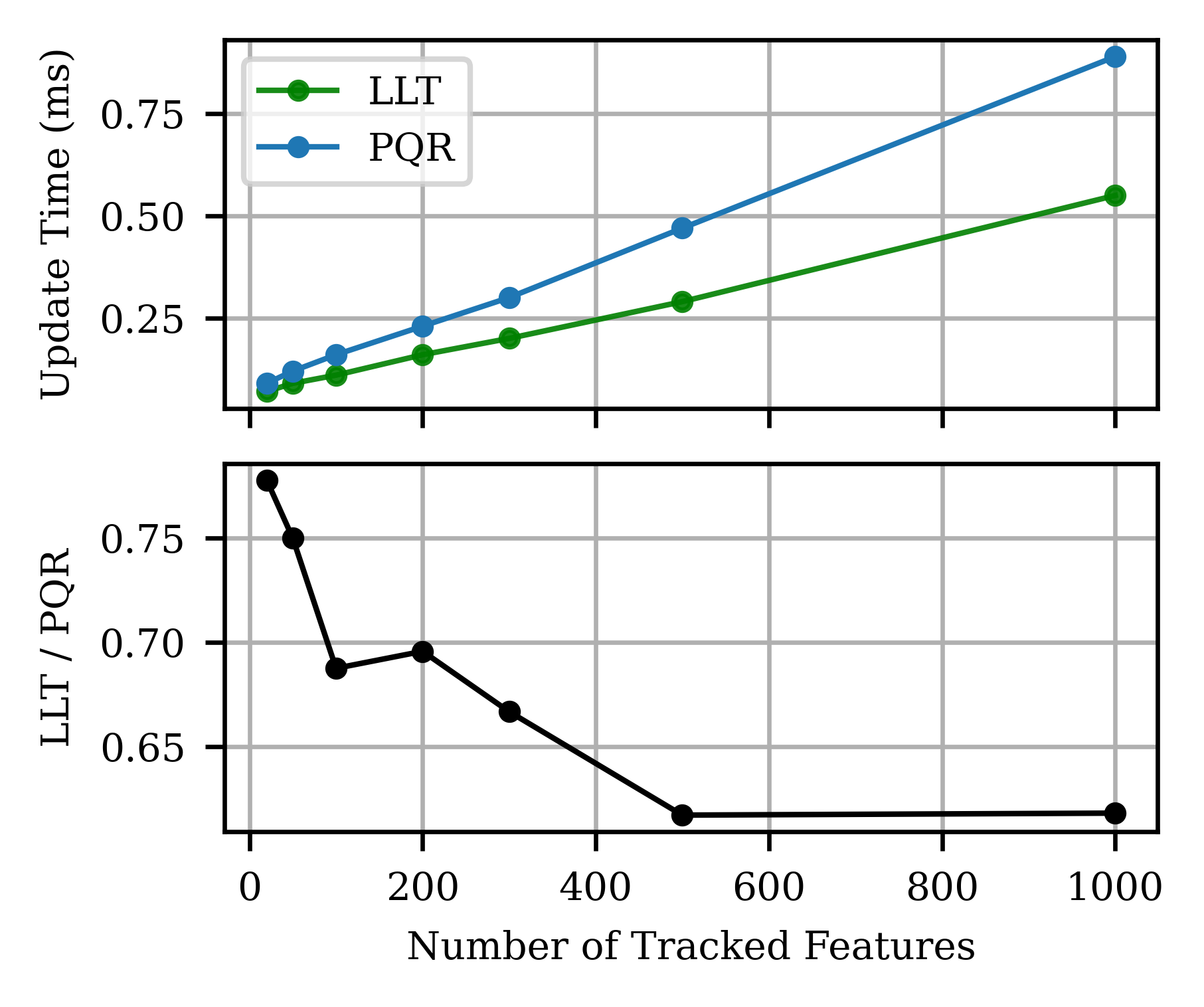}
\caption{Comparison of update efficiency between P-QR and LLT update methods.
The top plot illustrates the update times for the two methods: the \textcolor{black}{blue} line represents the original P-QR-based method, and the green line represents the proposed LLT-based method. The bottom plot shows the time ratio of LLT to P-QR update times (LLT/P-QR), indicating that the LLT-based method becomes increasingly efficient as the number of tracked features increases.
}
\label{fig:sim-time}
\end{figure}

In recent work~\cite{Peng2024ICRA}, the efficiency improvement of the P-QR-based SRF compared to other estimators was reported.
Here, we aim to thoroughly compare the novel LLT-based SRF with P-QR, demonstrating its efficiency improvements with theoretical guarantees, as shown in Table~\ref{tab:srf-flops}. 
To make our experimental evaluation more comprehensive, we will also evaluate and compare all forms of estimators in the following sections.

In this numerical study, we fix the state size and vary the number of tracked features to evaluate performance with different numbers of measurements. All features are treated as MSCKF features and the results are reported in Figure~\ref{fig:sim-time}. 

The top figure plots the update time (in ms) for both methods as the number of features changes, while the bottom figure shows the ratio of the update times (i.e., LLT/P-QR). The results clearly demonstrate that the LLT-based method is significantly more efficient, with the performance gap widening as the number of features increases. 
Moreover, the time ratio between the two methods (bottom figure) decreases progressively as the number of measurements becomes more dominating and, in the end, converges to roughly the theoretical ratio $\frac{2}{3}$, which proves the FLOPs analysis as shown in the Table~\ref{tab:srf-flops}.
\color{black}


\begin{figure}
\centering
\includegraphics[width=0.9\linewidth]{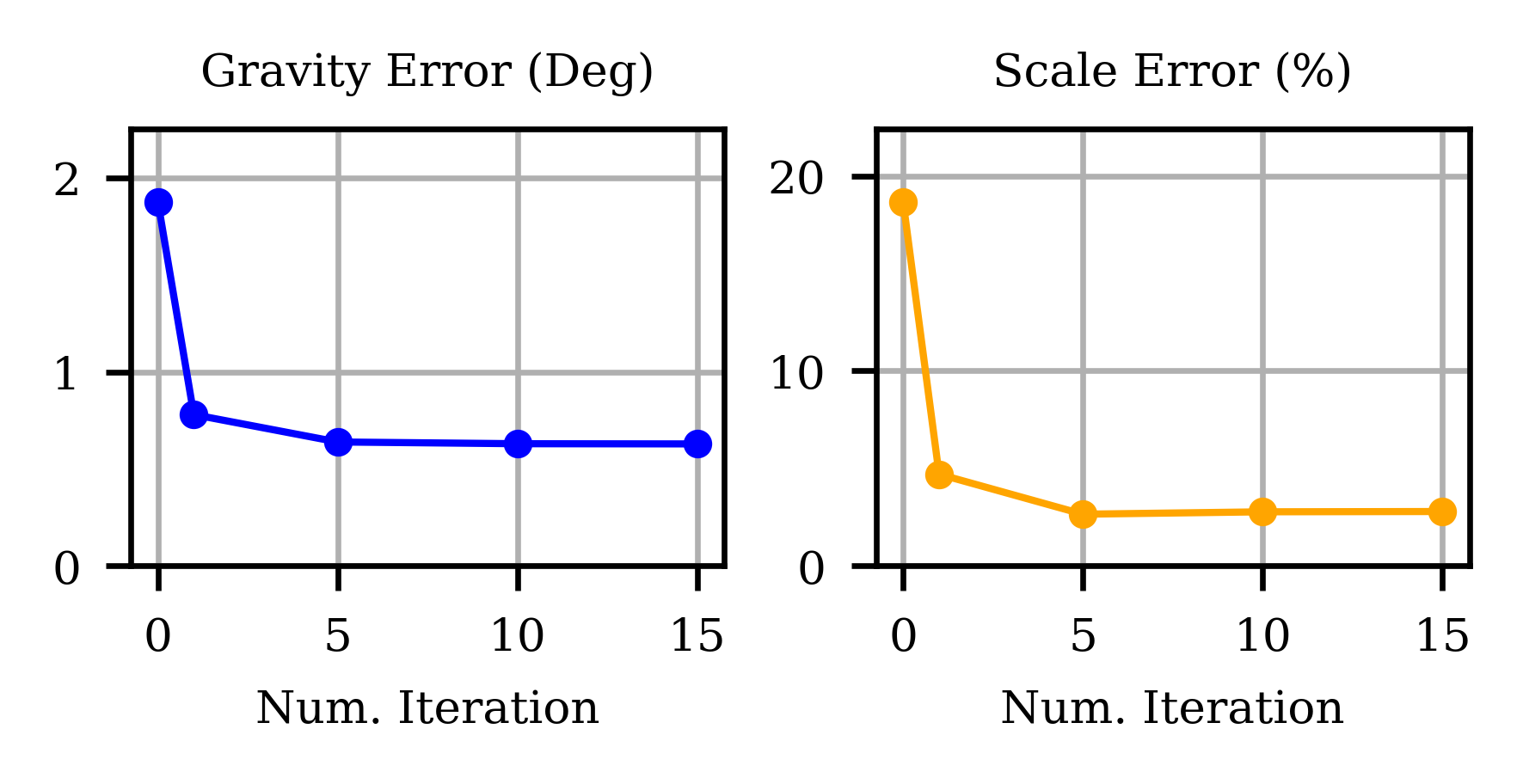}
\caption{ 
The impact of the number of iterations for initialization accuracy in monocular setup with the 0.5-second window in TUM-VI Dataset.
}
\label{fig:num_iter}
\end{figure}

\begin{figure*}
\centering
\includegraphics[width=0.9\linewidth]{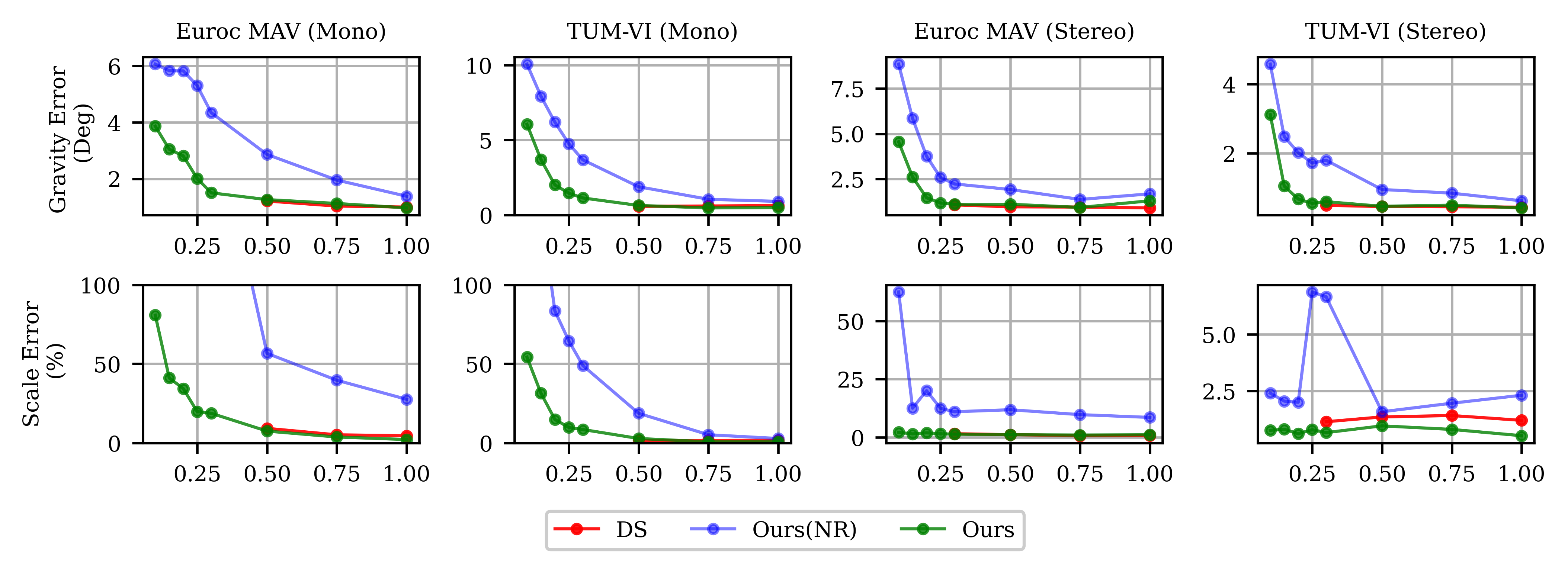}
\caption{
Initialization gravity error and scale error across different window sizes (time) and datasets are evaluated for both monocular and stereo setups. The results are compared using the DS method, our proposed method, and our method without iterative refinement, denoted as Ours (NR).
}
\label{fig:init_ACC}
\end{figure*}



\begin{figure*}[]
\centering
\includegraphics[width=0.8\textwidth]{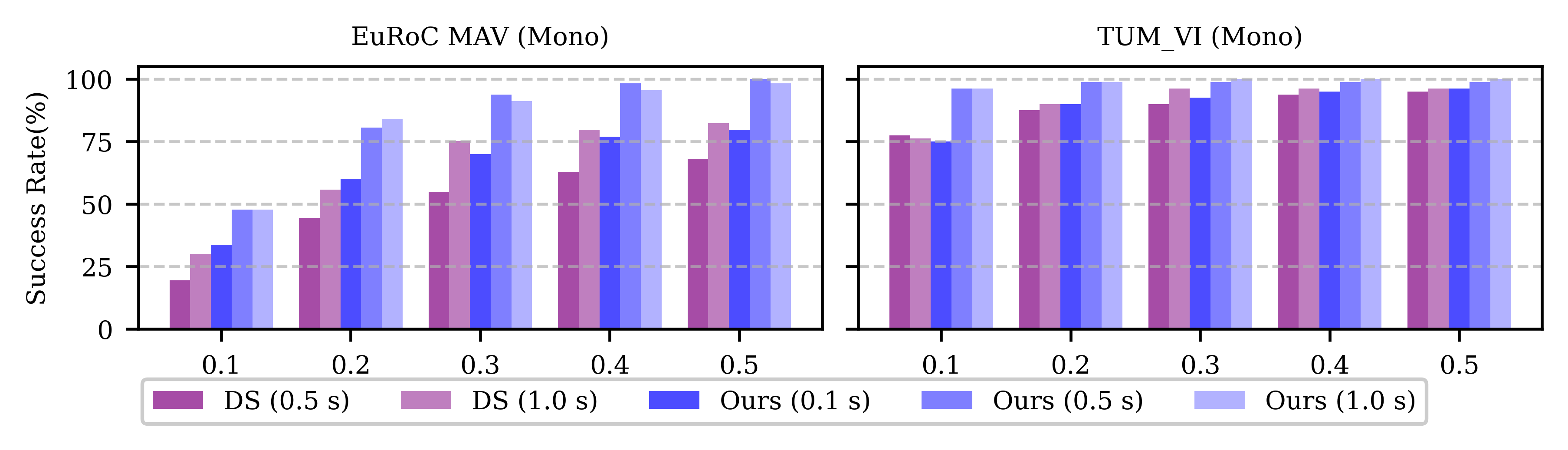}
\caption{ 
Successful initialization rates for the Euroc MAV dataset (left) and TUM VI dataset (right) with a 0.5m VIO position error threshold. The results compare the DS method with 0.5s and 1.0s windows to our method with 0.1s, 0.5s, and 1.0s windows, represented in different colors.
}
\label{fig:ate_init_ex}
\end{figure*}

\begin{table*}[]
\centering
\caption{\color{black}INITIALIZATION\color{black}~SUCCESS RATES ON THE Euroc MAV DATASET.}
\label{init_rate_mav}
\resizebox{0.95\linewidth}{!}{%
\begin{tabular}{@{}cc|ccccccc|ccccccc@{}}
\toprule
\multirow{2}{*}{} &
  \multirow{2}{*}{\textbf{Threshold (m)}} &
  \multicolumn{7}{c|}{
  \textbf{DS}
  } &
  \multicolumn{7}{c}{
  \textbf{Ours
  }} \\ \cmidrule(l){3-16} 
 &
   &
  0.1 s &
  0.15 s &
  0.2 s &
  0.3 s &
  0.5 s &
  0.75 s &
  1.0 s &
  0.1 s &
  0.15 s &
  0.2 s &
  0.3 s &
  0.5 s &
  0.75 s &
  1.0 s \\ \midrule
\multirow{3}{*}{Mono} &
  0.1 &
  \color{red}-\color{black} &
  \color{red}-\color{black} &
  \color{red}-\color{black} &
  \color{red}-\color{black} &
  \color{red}19.5\color{black} &
  \color{red}28.3\color{black} &
  \color{red}30.1\color{black} &
  \color{red}33.6\color{black} &
  \color{red}43.4\color{black} &
  \color{red}48.7\color{black} &
  \color{red}51.3\color{black} &
  \color{red}47.8\color{black} &
  \color{red}51.3\color{black} &
  \color{red}47.8\color{black} \\
 &
  0.3 &
  \color{red}-\color{black} &
  \color{red}-\color{black} &
  \color{red}-\color{black} &
  \color{red}-\color{black} &
  \color{red}54.9\color{black} &
  \color{red}63.7\color{black} &
  \color{red}75.2\color{black} &
  \color{red}69.9\color{black} &
  \color{red}83.2\color{black} &
  \color{red}82.3\color{black} &
  \color{red}86.7\color{black} &
  \color{red}93.8\color{black} &
  \color{red}86.7\color{black} &
  \color{red}91.2\color{black} \\
 &
  0.5 &
  \color{red}-\color{black} &
  \color{red}-\color{black} &
  \color{red}-\color{black} &
  \color{red}-\color{black} &
  \color{red}68.1\color{black} &
  \color{red}76.1\color{black} &
  \color{red}82.3\color{black} &
  \color{red}79.6\color{black} &
  \color{red}87.6\color{black} &
  \color{red}87.6\color{black} &
  \color{red}92.0\color{black} &
  \textbf{\color{ForestGreen}100.0}\color{black} &
  \textbf{\color{ForestGreen}96.5}\color{black} &
  \textbf{\color{ForestGreen}98.2}\color{black} \\ \midrule
\multirow{3}{*}{Stereo} &
  0.1 &
  \color{red}-\color{black} &
  \color{red}-\color{black} &
  \color{red}-\color{black} &
  \color{red}45.1\color{black} &
  \color{red}44.2\color{black} &
  \color{red}47.8\color{black} &
  \color{red}43.4\color{black} &
  \color{red}63.7\color{black} &
  \color{red}65.5\color{black} &
  \color{red}68.1\color{black} &
  \color{red}69.0\color{black} &
  \color{red}64.6\color{black} &
  \color{red}62.8\color{black} &
  \color{red}61.1\color{black} \\
 &
  0.3 &
  \color{red}-\color{black} &
  \color{red}-\color{black} &
  \color{red}-\color{black} &
  \color{red}84.1\color{black} &
  \color{red}86.7\color{black} &
  \color{red}82.3\color{black} &
  \color{red}85.8\color{black} &
  \color{red}94.7\color{black} &
  \textbf{\color{ForestGreen}96.5}\color{black} &
  \textbf{\color{ForestGreen}96.5}\color{black} &
  \textbf{\color{ForestGreen}96.5}\color{black} &
  \textbf{\color{ForestGreen}96.5}\color{black} &
  \textbf{\color{ForestGreen}95.6}\color{black} &
  \color{red}94.7\color{black} \\
 &
  0.5 &
  \color{red}-\color{black} &
  \color{red}-\color{black} &
  \color{red}-\color{black} &
  \color{red}90.3\color{black} &
  \color{red}92.0\color{black} &
  \color{red}90.3\color{black} &
  \color{red}93.8\color{black} &
  \textbf{\color{ForestGreen}96.5}\color{black} &
  \textbf{\color{ForestGreen}99.1}\color{black} &
  \textbf{\color{ForestGreen}99.1}\color{black} &
  \textbf{\color{ForestGreen}100.0}\color{black} &
  \textbf{\color{ForestGreen}99.1}\color{black} &
  \textbf{\color{ForestGreen}98.2}\color{black} &
  \textbf{\color{ForestGreen}98.2}\color{black} \\ \bottomrule
\end{tabular}%
}
\color{black}
\vspace{1.2em}
\noindent
\parbox{0.95\linewidth}{
\footnotesize
\textbf{Note:} This table compares the DS method and our proposed method across different initialization window sizes for monocular and stereo setups on the TUM VI dataset. Rates $\geq$ 95\% are shown in \textbf{\textcolor{ForestGreen}{green and bold}}, while rates $<$ 95\% are shown in \textcolor{red}{red}. Missing entries indicate failure to meet the defined success criteria.
}
\color{black}
\end{table*}

\begin{table*}[]
\centering
\caption{Initialization success rates on the TUM VI dataset.}
\label{init_rate_tum}
\resizebox{0.95\linewidth}{!}{%
\begin{tabular}{@{}cc|ccccccc|ccccccc@{}}
\toprule
\multirow{2}{*}{} &
  \multirow{2}{*}{\textbf{Threshold (m)}} &
  \multicolumn{7}{c|}{\textbf{DS}} &
  \multicolumn{7}{c}{\textbf{Ours}} \\ \cmidrule(l){3-16} 
 &
   &
  0.1 s &
  0.15 s &
  0.2 s &
  0.3 s &
  0.5 s &
  0.75 s &
  1.0 s &
  0.1 s &
  0.15 s &
  0.2 s &
  0.3 s &
  0.5 s &
  0.75 s &
  1.0 s \\ \midrule
\multirow{3}{*}{Mono} &
0.1 & \color{red}-\color{black} & \color{red}-\color{black} & \color{red}-\color{black} & \color{red}-\color{black} & \color{red}77.5\color{black} & \color{red}77.5\color{black} & \color{red}76.2\color{black} & \color{red}75.0\color{black} & \color{red}82.5\color{black} & \color{red}91.2\color{black} & \color{red}93.8\color{black} & \textbf{\color{ForestGreen}96.2}\color{black} & \textbf{\color{ForestGreen}100.0}\color{black} & \textbf{\color{ForestGreen}96.2}\color{black} 
\\
& 
0.3 & \color{red}-\color{black} & \color{red}-\color{black} & \color{red}-\color{black} & \color{red}-\color{black} & \color{red}90.0\color{black} & \color{red}90.0\color{black} & \textbf{\color{ForestGreen}96.2}\color{black} & \color{red}92.5\color{black} & \color{red}91.2\color{black} & \textbf{\color{ForestGreen}96.2}\color{black} & \textbf{\color{ForestGreen}97.5}\color{black} & \textbf{\color{ForestGreen}98.8}\color{black} & \textbf{\color{ForestGreen}100.0}\color{black} & \textbf{\color{ForestGreen}100.0}\color{black} \\
& 0.5 & \color{red}-\color{black} & \color{red}-\color{black} & \color{red}-\color{black} & \color{red}-\color{black} & \color{red}95.0\color{black} & \color{red}93.8\color{black} & \textbf{\color{ForestGreen}96.2}\color{black} & \textbf{\color{ForestGreen}96.2}\color{black} & \color{red}93.8\color{black} & \textbf{\color{ForestGreen}97.5}\color{black} & \textbf{\color{ForestGreen}100.0}\color{black} & \textbf{\color{ForestGreen}98.8}\color{black} & \textbf{\color{ForestGreen}100.0}\color{black} & \textbf{\color{ForestGreen}100.0}\color{black} 
\\
\midrule
\multirow{3}{*}{Stereo} &
0.1 & \color{red}-\color{black} & \color{red}-\color{black} & \color{red}-\color{black} & \color{red}91.2\color{black} & \color{red}86.2\color{black} & \color{red}92.5\color{black} & \textbf{\color{ForestGreen}98.8}\color{black} & \textbf{\color{ForestGreen}98.8}\color{black} & \textbf{\color{ForestGreen}100.0}\color{black} & \textbf{\color{ForestGreen}100.0}\color{black} & \textbf{\color{ForestGreen}100.0}\color{black} & \textbf{\color{ForestGreen}100.0}\color{black} & \textbf{\color{ForestGreen}98.8}\color{black} & \textbf{\color{ForestGreen}100.0}\color{black} \\
& 0.2 & \color{red}-\color{black} & \color{red}-\color{black} & \color{red}-\color{black} & \textbf{\color{ForestGreen}97.5}\color{black} & \textbf{\color{ForestGreen}98.8}\color{black} & \textbf{\color{ForestGreen}100.0}\color{black} & \textbf{\color{ForestGreen}100.0}\color{black} & \textbf{\color{ForestGreen}100.0}\color{black} & \textbf{\color{ForestGreen}100.0}\color{black} & \textbf{\color{ForestGreen}100.0}\color{black} & \textbf{\color{ForestGreen}100.0}\color{black} & \textbf{\color{ForestGreen}100.0}\color{black} & \textbf{\color{ForestGreen}98.8}\color{black} & \textbf{\color{ForestGreen}100.0}\color{black} \\
& 0.3 & \color{red}-\color{black} & \color{red}-\color{black} & \color{red}-\color{black} & \textbf{\color{ForestGreen}97.5}\color{black} & \textbf{\color{ForestGreen}100.0}\color{black} & \textbf{\color{ForestGreen}100.0}\color{black} & \textbf{\color{ForestGreen}100.0}\color{black} & \textbf{\color{ForestGreen}100.0}\color{black} & \textbf{\color{ForestGreen}100.0}\color{black} & \textbf{\color{ForestGreen}100.0}\color{black} & \textbf{\color{ForestGreen}100.0}\color{black} & \textbf{\color{ForestGreen}100.0}\color{black} & \textbf{\color{ForestGreen}98.8}\color{black} & \textbf{\color{ForestGreen}100.0}\color{black} 
\\
  \bottomrule
\end{tabular}%
}
\color{black}
\vspace{1.2em}
\noindent
\parbox{0.95\linewidth}{
\footnotesize
\textbf{Note:} This table compares the DS method and our proposed method across different initialization window sizes for monocular and stereo setups on the TUM VI dataset. Rates $\geq$ 95\% are shown in \textbf{\textcolor{ForestGreen}{green and bold}}, while rates $<$ 95\% are shown in \textcolor{red}{red}. Missing entries indicate failure to meet the defined success criteria.
}
\color{black}
\end{table*}

\begin{figure}[]
\centering
\includegraphics[width=0.8\linewidth]{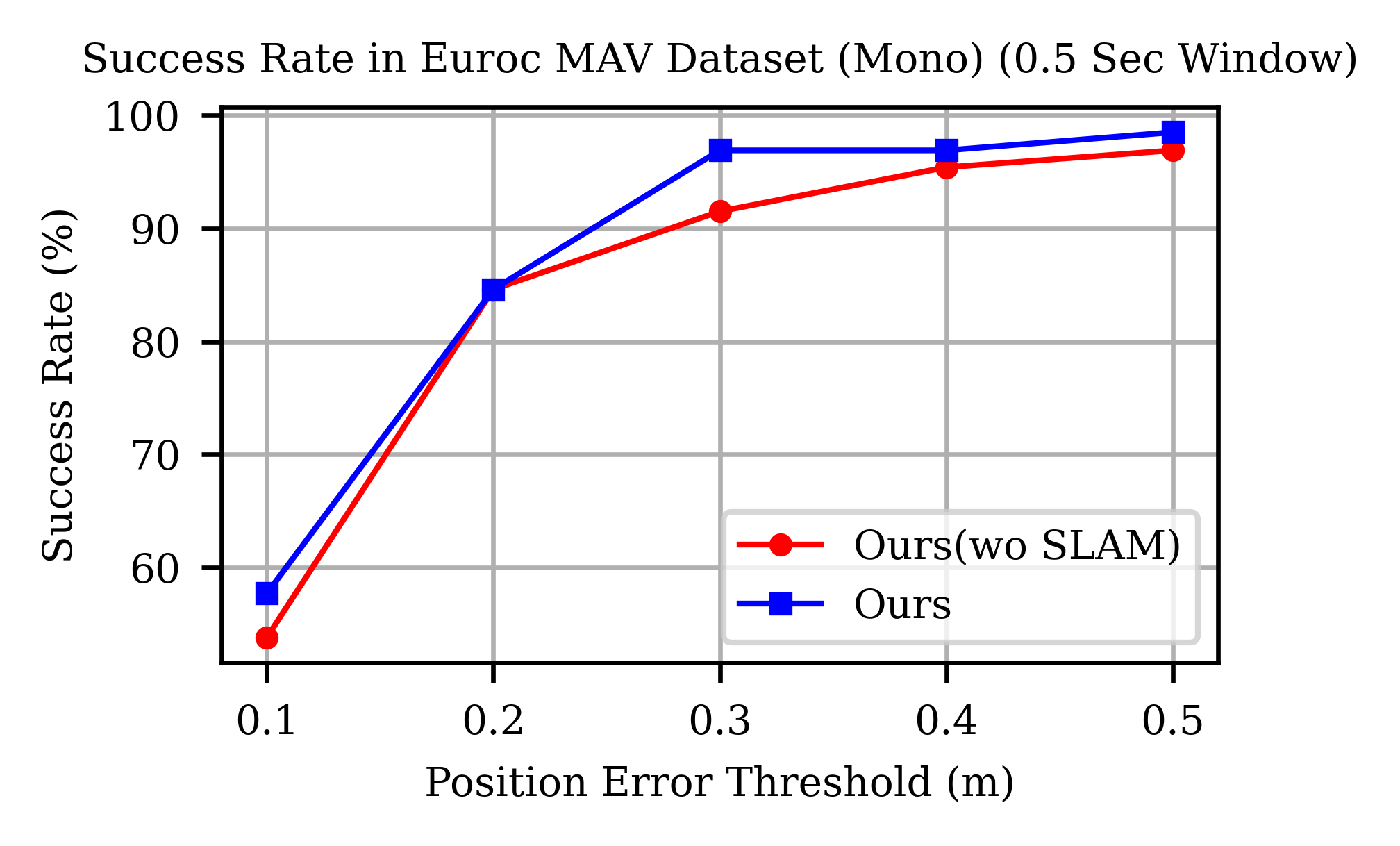}
\caption{ 
{Visual-inertial tracking successful rate over different position error
threshold using dynamic initialization in EurocMAV Dataset with 0.5-second
initialization window and monocular setup with and without initialization of
SLAM features in the state.
}
}
\label{fig:init_slam}
\end{figure}

\begin{figure*}[]
\centering
\includegraphics[width=0.42\linewidth]{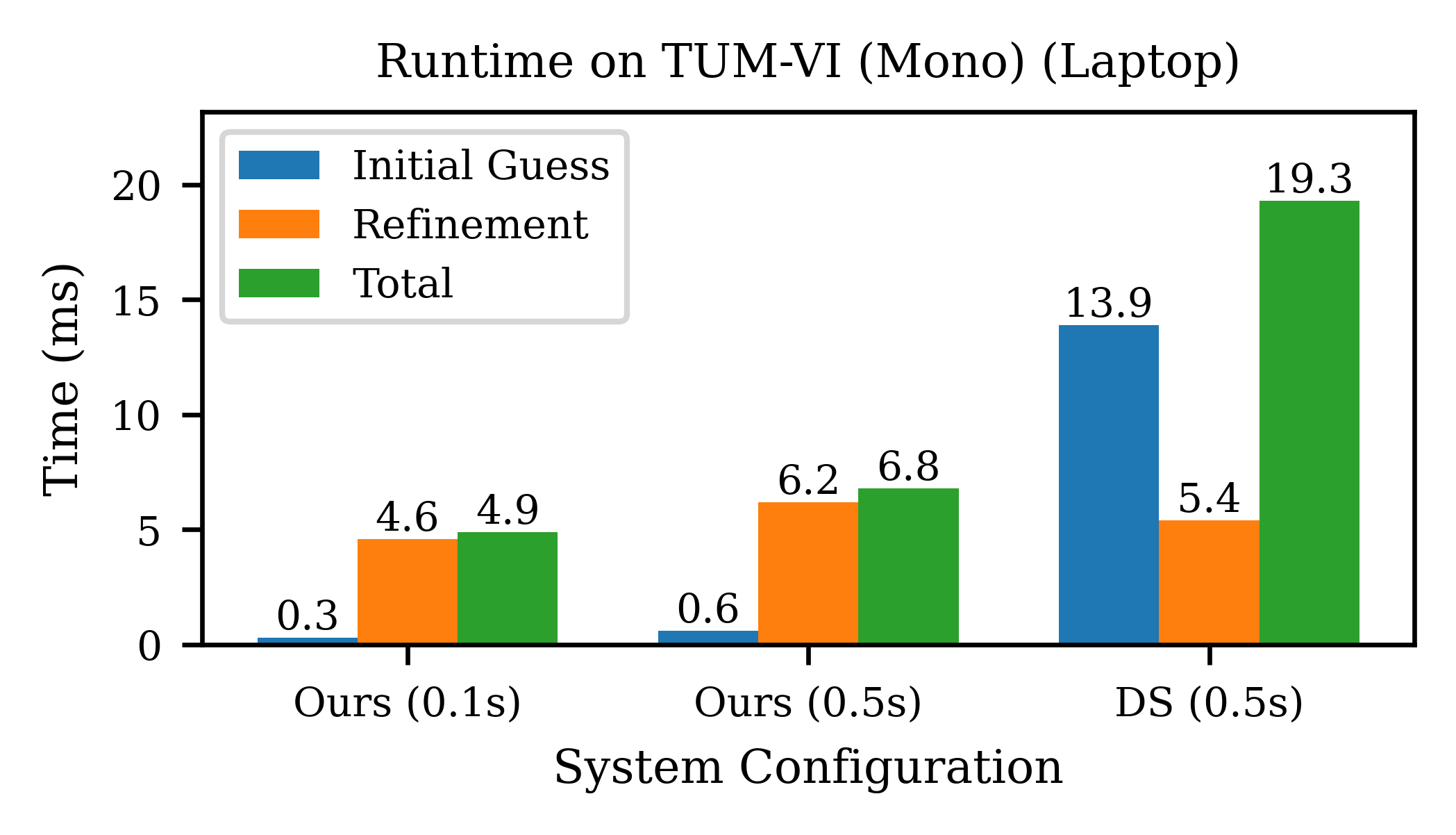}
\includegraphics[width=0.42\linewidth]{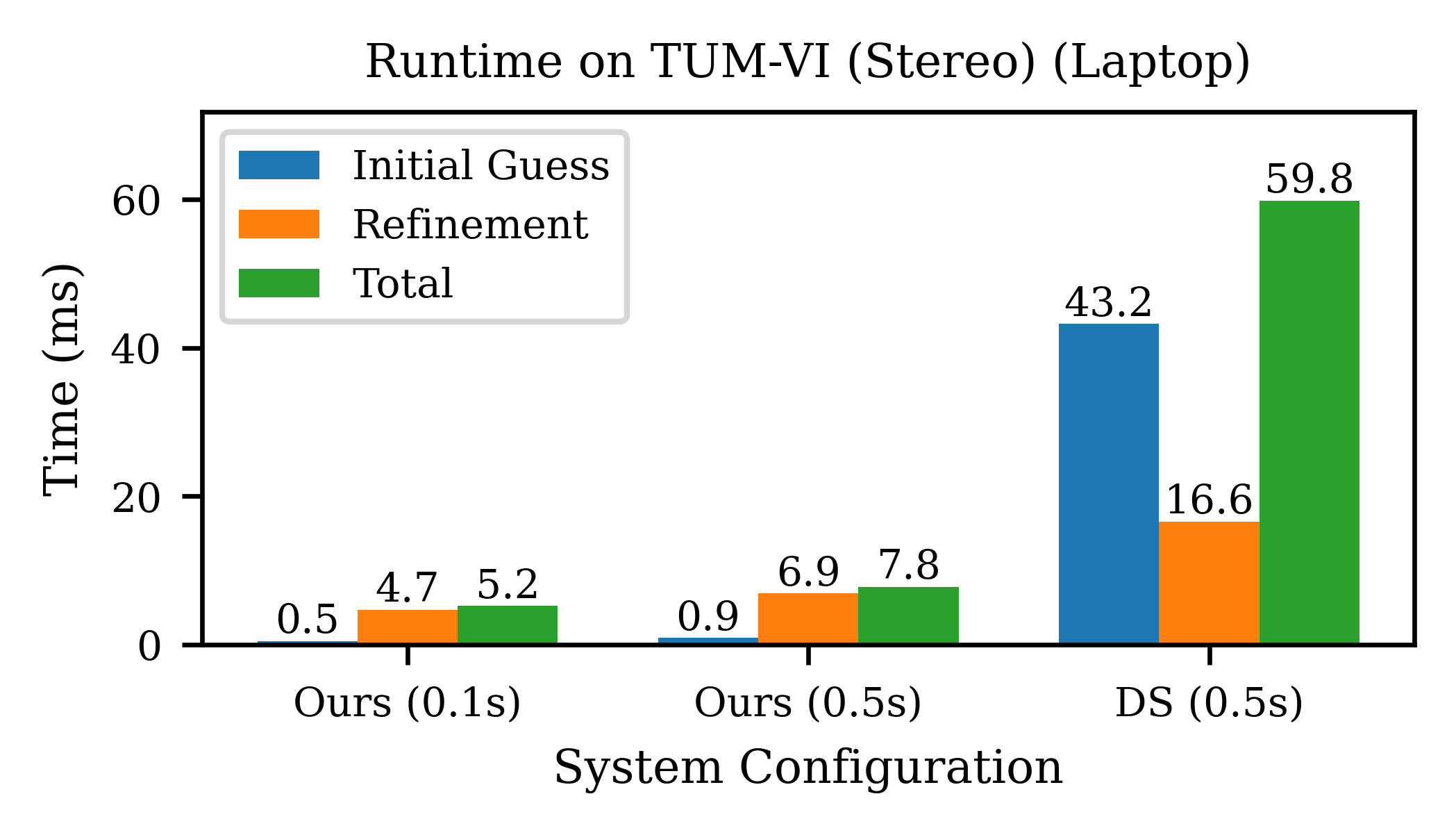}
\includegraphics[width=0.42\linewidth]{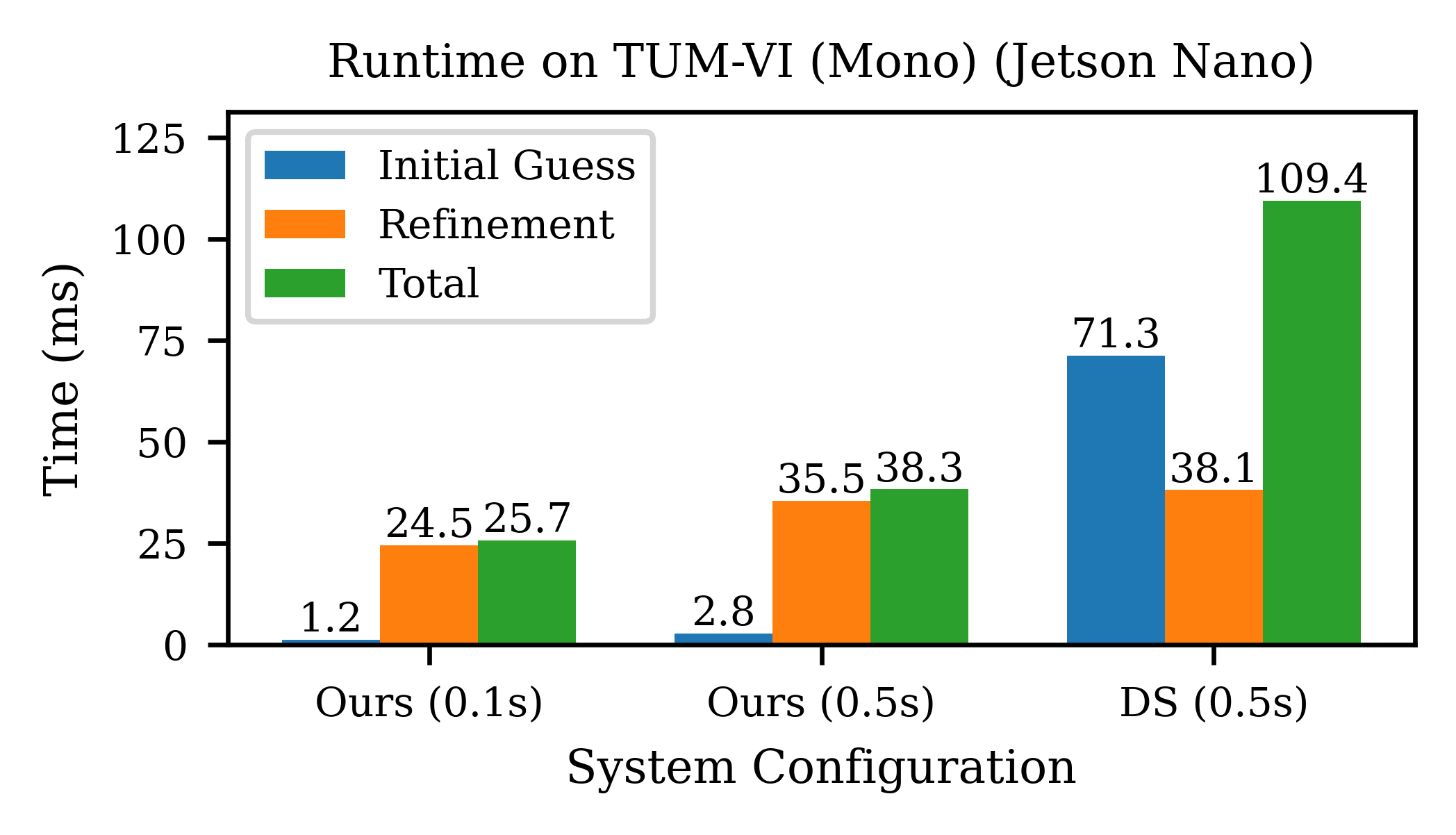}
\includegraphics[width=0.42\linewidth]{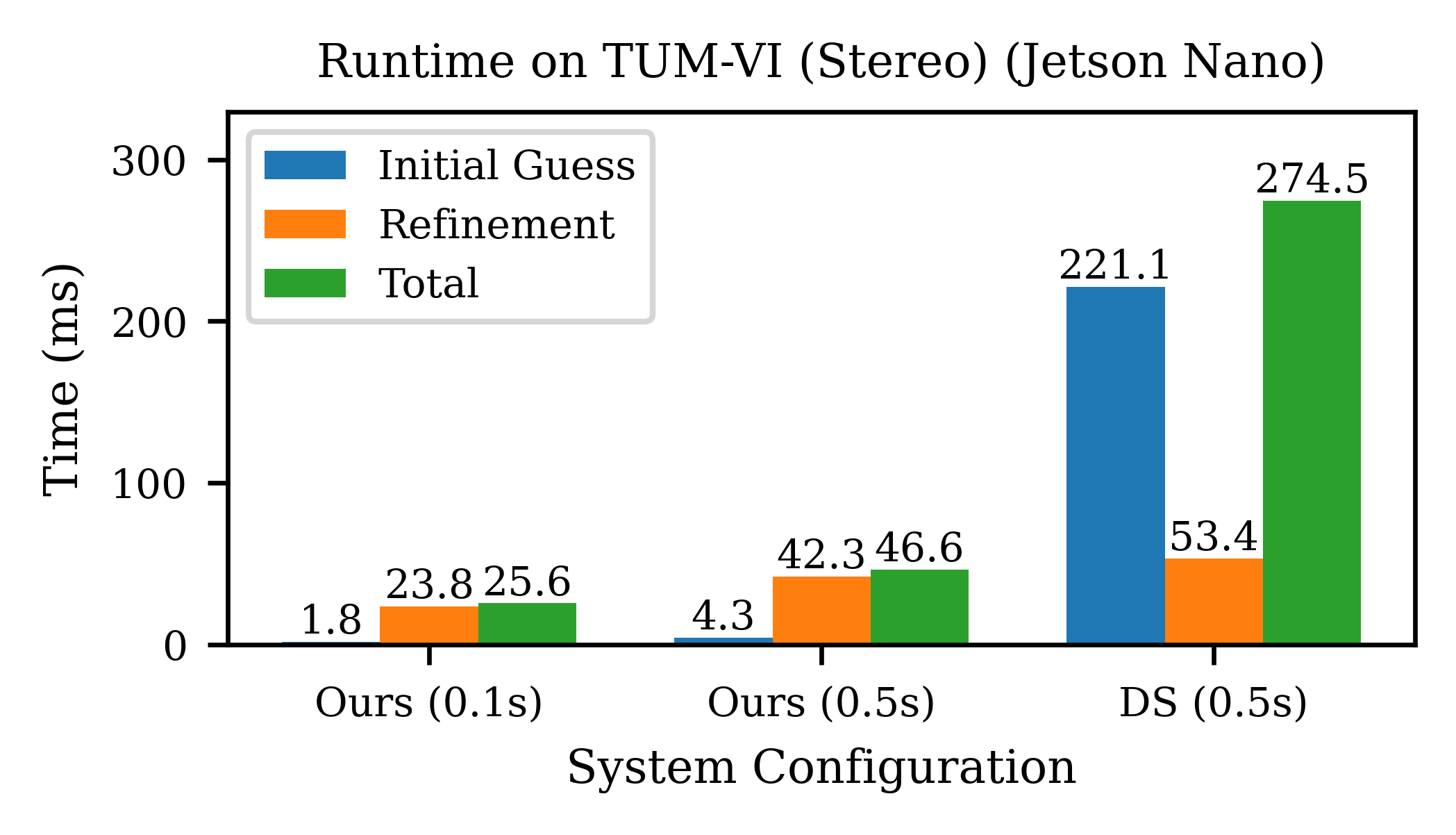}
\caption{ 
Initialization runtime on the TUM VI dataset evaluated on a laptop (top) and Jetson Nano (bottom) for different system configurations, comparing our method with the DS method.
}
\label{fig:init_time_lt}
\end{figure*}


\section{Experimental Evaluation of System Initialization}
\label{sec:exp_init}
To demonstrate the enhanced performance of the proposed novel dynamic initialization method, we validate it using two widely recognized and publicly available visual-inertial (VI) datasets: EuRoC MAV~\cite{Burri2016IJRR} and TUM-VI~\cite{Schubert2018IROS}. 
We compare our method (\textbf{Ours}) with the state-of-the-art dynamic initialization approach in OpenVINS, a reimplementation of Dongsi's method~\cite{Dongsi2012IROS}, referred to as \textbf{DS} in the following sections.
We also evaluate the performance with both monocular and stereo cameras.
{In stereo setup, except for independent KLT tracking of both cameras, we also perform tracking between each stereo pair to formulate stereo constraint. The parallax gained from stereo allows easy feature triangulation even in not fully excited motion.
}
To evaluate initialization performance, we divide each sequence into 10-second windows, run initialization at each entry point, and average the results across all runs. 
\color{black}
The keyframes are selected based on the average parallax, and all the features obtained from tracking are used to formulate the linear system constraints.
\color{black}

\subsection{Initialization Accuracy}
\label{sec:exp_init_acc}
We begin by reporting the accuracy of the initialized scale and gravity under various setups, as shown in Figure~\ref{fig:init_ACC}.
The figure presents the gravity error (top) and the scale error (bottom) for different initialization methods, including DS (red), our method without iterative SRF refinement (blue), and our method (green), evaluated with varying window sizes (in second) for both monocular and stereo setups.
Specifically, a Sim(3) transformation is fitted between the estimated and ground truth trajectories. The scale error is computed as $100 \times (\max(s, 1/s) - 1)$, where $s$ represents the scale obtained from the Sim(3) transformation.

From the figure, it is evident that the initialization errors decrease as the initialization time increases due to the availability of more measurements. However, the key advantage of the proposed initialization method lies in its ability to successfully initialize within a very short time frame (0.1 s), where DS method fails.
The results also highlight the improved performance achieved with the iterative update in SRF especially with small window sizes.
%
To further illustrate this, we compare the gravity and scale errors across different numbers of iterations, as shown in Figure~\ref{fig:num_iter}. 
It is evident that the errors consistently decrease with an increasing number of iterations and eventually converge.
Interestingly, we also observe that as the window size increases, the accuracy gain from refinement decreases.
Because with the increasing motion for inertial measurements and increasing parallax for visual measurements, the signal-noise-ratio increases.
Thanks to the robust novel constraint formulation, our method achieves desirable accuracy without further refinement.
This indicates the possibility to reduce or eliminate iterations for further efficiency gains, as discussed in the following section.
%

\subsection{Successful Initiation Rate}
\label{sec:result_init_success}
We now present the success rates based on our defined criteria for success under various system initialization setups. 
Traditionally, the success rate is defined by dividing each sequence into small windows, running initialization at each entry point, and calculating the percentage of successful dynamic initializations. 
It is important to note that we have applied a stricter criterion for success, given the strong performance of our proposed method.
Our criterion ensures initialization is tested at various points within each sequence, assessing both the accuracy of the initialization window poses and the subsequent VIO performance using the initialization results.
The criteria are as follows:
\begin{itemize}
    \item 
 We divide each sequence
into 10-second windows, run the initialization with a different initialization window (from 0.1 to 1s) at each entry
point.  
The initial states must be successfully solved for \textit{all} runs in the sequences, including ensuring that the linear system can be solved without issues. Additionally, the initial states must converge after refinement, and their covariance must be successfully obtained.
    \item 
    The position ATE of VIO within the first 10 seconds must remain below a predefined threshold.
\end{itemize}
Given these criteria, the results for the Euroc MAV dataset are shown in Table~\ref{init_rate_mav}, while Table~\ref{init_rate_tum} reports the results for the TUM VI dataset. 
In these tests, we vary the initialization window (from 0.1s to 1.0s) and report the success rates for different thresholds (i.e., 0.1m, 0.3m, and 0.5m).
In the tables, success rates higher than 95\% are highlighted in green, while those below 95\% are marked in red. 
Results are not reported if the defined success criteria are not met.

The results clearly demonstrate that the proposed SRF method consistently achieves a higher success rate compared to DS. For instance, with monocular setup initialization windows of 0.5s, 0.75s, and 1.0s, SRF significantly outperforms DS for both datasets.
Our method also consistently outperforms DS, especially in more challenging scenarios. For example, for the Euroc MAV dataset (shown in Table~\ref{init_rate_mav}), our method successfully initializes VIO with a 33\% success rate for a 0.1m position error threshold and a 79.6\% success rate for a 0.5m position error threshold using a 0.1s window. In contrast, the DS method fails to guarantee successful initialization in such a short time period.
These results demonstrate that the proposed method is well-constrained, robust to changes in window size, and significantly more efficient and reliable compared to DS.
To further illustrate this, Figure~\ref{fig:ate_init_ex} shows the success rates for both the DS method and our proposed method under 0.5s and 1.0s initialization windows. 
Additionally, we include results for our method with a 0.1s window.
From this plot, it is evident that, given the same initialization window, our method consistently outperforms DS (e.g., Ours (0.5s) achieves a higher success rate compared to DS (0.5s)). 
Impressively, Ours (0.1s) not only outperforms DS (0.5s) but also achieves comparable performance to DS (1.0s), demonstrating that our method is 
ultrafast and highly effective, even with minimal initialization windows.

We also investigate the inclusion of SLAM features and their impact on initialization, as shown in Figure~\ref{fig:init_slam}. 
Incorporating SLAM features significantly increases the success rate. 
As discussed in Section~\ref{sec:init_refine}, keeping SLAM features allows more information to be preserved after initialization, effectively enhancing the robustness and improving the performance of VIO.
Notably, despite the inclusion of additional states (i.e., SLAM features), the proposed method achieves remarkable efficiency, enabled by our streamlined linear system formulation and SRF update methods, as demonstrated in the following section.

\subsection{Initialization Timing Analysis}
We report the runtime of our initialization method across various window sizes on both a laptop and a Jetson Nano, compared with the DS method, as shown in Figure~\ref{fig:init_time_lt}.\footnote{Computational results were performed in a single thread on Laptop with an Intel(R) Core(TM) i7-11800H @ 2.30GHz and Jetson Nano with ARM Cortex-A57 4 Core @ 1.5GHz.}
Results for the DS initialization method are not reported for initialization windows shorter than 0.5s because it fails under these conditions. 
We should note that our refinement time also includes the time used for feature triangulation. But in DS's method this is calculated in initial guess step.
For covariance recovery, which is part of the refinement stage, DS only includes recovery of the 15-DoF IMU state, while we also allow for recovery of covariance of SLAM features.
The results clearly show that the proposed method with a 0.1s initialization window is the most efficient.
For a 0.5s initialization window, ours demonstrates significantly improved efficiency compared to DS. 
Notably, when running in Jetson Nano, DS's method can not be initialized before the next camera measurement comes (for a 20 Hz camera, dynamic initialization is required to finish within 50 ms to guarantee real-time), while our system can still achieve real-time performance.

This improvement stems from several key factors. 
First, the proposed method avoids the need to solve for feature positions in the linear system, significantly reducing the computational time required for generating the initial guess. 
Second, the efficient iterative update mechanism effectively minimizes errors while maintaining computational efficiency.
Furthermore, the proposed method provides direct access to the covariance matrix without requiring the inversion of the information matrix—a process that is computationally expensive and numerically unstable, often necessitating inflation of the initial covariance to stabilize the system. These advantages make ours method not only accurate and robust but also highly efficient.

We further discuss the 0.1-second window (3 keyframes) and 0.5-second window (5 keyframes) reported in Figure~\ref{fig:init_time_lt} for our method to provide a deeper understanding. 
The key difference lies in the number of keyframes and the total number of measurements.
As discussed in Section \ref{sec:dy_init}. at the initial guess stage, the computational complexity of our method is quadratic with respect to the number of keyframes and linear with respect to the number of tracked features per frame. 
Therefore, using only three keyframes requires approximately half the computation time compared to five keyframes.
In the refinement stage, the measurement size has a greater impact than the state size, as the latter remains relatively small.
Since our refinement method has linear complexity with respect to measurements, the computation time for 5 keyframes increases linearly compared to 3 keyframes.
As mentioned in Section~\ref{sec:exp_init_acc}, our method provides a good initial guess for large initialization windows. When computational resources are limited, it is possible to skip iterative refinement and use a conservative initial covariance, allowing the system to initialize in 0.6 ms on a laptop and 2.8 ms on a low-end embedded system.


\begin{table*}[]
\centering
\caption{Average Absolute Trajectory Error (ATE) in degrees/meters.}
\resizebox{1.0\textwidth}{!}{%
\begin{tabular}{@{}ccccccccccccccccccc@{}} \toprule
\textbf{Algo.} & 
\textbf{V101} & \textbf{V102} & \textbf{V103} & \textbf{V201} & \textbf{V202} & \textbf{V203} &
\textbf{MH01} & \textbf{MH02} & \textbf{MH03} & \textbf{MH04} & \textbf{MH05}  \\\toprule
\textbf{OpenVINS}(d) & 0.70 / 0.06 & 1.67 / 0.06 & 2.88 / 0.07 & 0.95 / 0.10 & 1.38 / 0.06 & 1.28 / 0.14 &
1.74 / 0.10 & 0.91 / 0.17 & 1.14 / 0.12 & 0.95 / 0.25 & 1.03 / 0.41 \\
\textbf{OpenVINS}(f) & 
0.71 / 0.06 & 1.66 / 0.06 & 2.87 / 0.06 & 0.94 / 0.10 & 1.40 / 0.06 & 1.25 / 0.14 &
1.76 / 0.10 & 0.91 / 0.17 & 1.18 / 0.13 & 0.94 / 0.25 & 1.04 / 0.41 \\
{$\sqrt{{\mathbf{VINS}}}$}(d) & 0.54 / 0.06 & 1.60 / 0.06 & 2.94 / 0.05 & 1.09 / 0.10 & 1.41 / 0.06 & 1.36 / 0.12 &
1.90 / 0.11 & 0.77 / 0.14 & 1.06 / 0.12 & 1.02 / 0.23 & 1.14 / 0.35 \\
{$\sqrt{{\mathbf{VINS}}}$}(f) & 
0.54 / 0.06 & 1.65 / 0.05 & 2.68 / 0.06 & 1.09 / 0.10 & 1.48 / 0.06 & 1.17 / 0.11 &
1.97 / 0.10 & 0.74 / 0.14 & 0.88 / 0.11 & 0.99 / 0.25 & 1.13 / 0.35 \\
\midrule
{$\sqrt{{\mathbf{VINS}}}$}(M) & 
0.63 / 0.07 & 1.73 / 0.08 & 1.76 / 0.07 & 0.80 / 0.07 & 1.39 / 0.09 & 1.48 / 0.14 & 
2.18 / 0.16 & 0.57 / 0.15 & 1.59 / 0.23 & 0.59 / 0.14 & 0.49 / 0.30  \\
{$\sqrt{{\mathbf{VINS}}}$}(L) & 0.73 / 0.05 & 1.83 / 0.10 & 2.75 / 0.06 & 0.71 / 0.06 & 1.22 / 0.07 & 1.53 / 0.15 & 
 1.32 / 0.13 & 0.73 / 0.15 & 1.57 / 0.23 & 0.86 / 0.21 & 0.74 / 0.40  \\
\textbf{RVIO2} & 0.88 / 0.09 & 2.27 / 0.10 & 2.02 / 0.10 & 2.19 / 0.13 & 1.90 / 0.11 & 1.50 / 0.15 & 2.60 / 0.17 & 1.00 / 0.15 & 1.08 / 0.19 & 1.10 / 0.24 & 0.95 / 0.32 \\
\textbf{EqVIO} & 0.66 / 0.05 & 2.64 / 0.14 & 3.37 / 0.19 & 1.31 / 0.10 & 1.67 / 0.18 & 1.71 / 0.20 & 2.14 / 0.14 & 0.89 / 0.15 & 1.11 / 0.09 & 2.09 / 0.35 & 1.29 / 0.23 \\
\textbf{VINS-Mono} & 
0.82 / 0.07 & 2.74 / 0.10 & 5.15 / 0.15 & 2.13 / 0.09 & 2.57 / 0.13 & 3.43 / 0.29 & 
0.78 / 0.20 & 0.86 / 0.18 & 1.82 / 0.23 & 2.51 / 0.41 & 0.94 / 0.29 \\
\bottomrule
\end{tabular}%
}
\color{black}
\noindent
\parbox{0.95\linewidth}{
\footnotesize
\textbf{Note:} ‘d’ and ‘f’ indicate the use of double and float precision, respectively. \color{black}$\sqrt{VINS}$(M)\color{black} uses float precision and MSCKF features for comparison with RVIO2. \color{black}$\sqrt{VINS}$(L)\color{black} is a lightweight configuration optimized for ultra-efficiency.
}
\color{black}
\label{tab:real}
\end{table*}



\begin{table} 
\centering
\caption{\textcolor{black}{Estimator runtime (ms) excluding feature tracking on EuRocMAV (Laptop).} 
}
\resizebox{0.95\linewidth}{!}{%
\label{tab:runtime}
\begin{tabular}{@{}cccccccc@{}}
\toprule
 & \textbf{OpenVINS} & \textbf{$\sqrt{\mathbf{VINS}}$} & \multicolumn{1}{l}{\textbf{$\sqrt{\mathbf{VINS}}$(M)}} & \textbf{$\sqrt{\mathbf{VINS}}$(L)} & \textbf{RVIO2}& \color{black}\textbf{EqVIO}\color{black}  & \textbf{VINS-Mono}  \\ \midrule
\textbf{Double}    & 4.2          & 2.4          & -      & -     & -                       & \color{black}0.9\color{black}  & 22.4      \\
\textbf{Float}     & 3.0         & 1.7 & 0.7          & 0.4                                 & 1.8     & \color{black}-\color{black}         & -                   \\ \bottomrule
\vspace{0.1em}
\end{tabular}%
}
\color{black}
\vspace{1.5em}
\noindent
\parbox{0.95\linewidth}{
\footnotesize
\textbf{Note:} Estimator runtime (ms) excluding feature tracking on EuRocMAV (Laptop). The efficiency of EqVIO primarily stems from its relatively compact state size compared to other methods. However, further speed improvements may be achieved by reformulating the system in SRF form.
}
\color{black}
\end{table}

\begin{figure}[]
\centering
\includegraphics[width=0.9\linewidth]{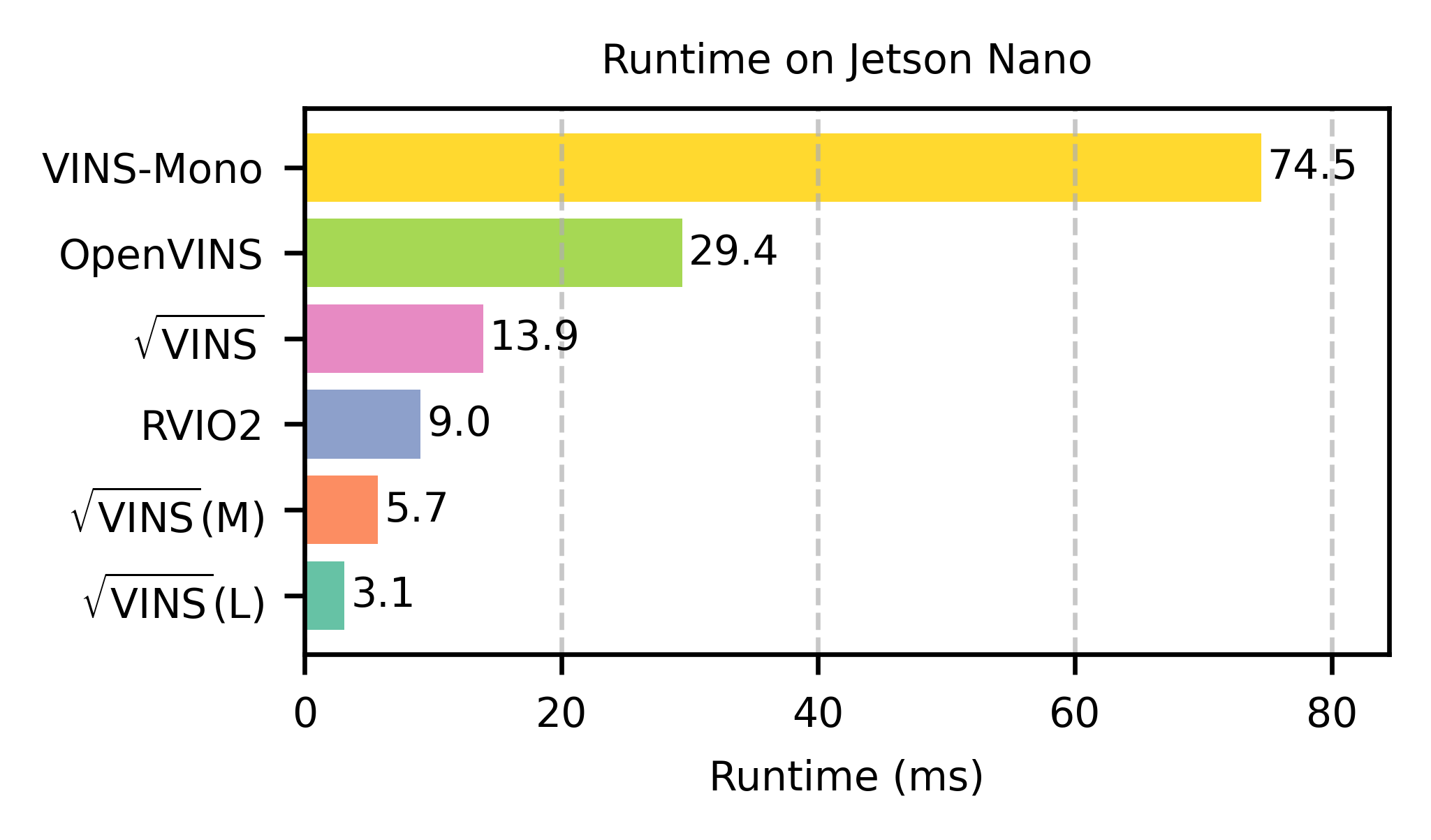}
\caption{ 
Estimator runtime (ms) excluding feature tracking on EuRocMAV (Jetson Nano).
}
\label{fig:init_time_jetson}
\end{figure}

\begin{table*}
\centering
\caption{Results on the Aria Everyday Activities dataset.}
\resizebox{0.8\textwidth}{!}{%
\begin{tabular}{@{}ccccccccccccccc@{}} 
\toprule
\textbf{Algo.} & 
\textbf{Location 1} & \textbf{Location 2} & \textbf{Location 3} & \textbf{Location 4} & \textbf{Location 5} & \textbf{Avg.} & \textbf{Runtime} \\ 
\midrule
\textbf{OpenVINS}* &  1.02 / 0.05 & 1.25 / 0.05 & 1.22 / 0.04 & 1.30 / 0.04 & 1.38 / 0.07 & 1.23 / 0.05 & 2.1 \\
{$\mathbf{\sqrt{{{VINS}}}}$} &  1.13 / 0.04 & 0.95 / 0.04 & 1.13 / 0.04 & 1.42 / 0.05 & 1.15 / 0.06 & 1.16 / 0.05 & 0.9 \\
\bottomrule\\
\end{tabular}%
}
\centering
\vspace{0.5em}
\parbox{0.8\textwidth}{
\footnotesize
\textcolor{black}{
    \textbf{Note:} Accuracy is reported as Average Absolute Trajectory Error (ATE) in degrees/meters; estimator runtime is in microseconds.} \\
\textsuperscript{*}Some OpenVINS runs diverged (position error $>$ 1m); only successful runs are reported, which may unfairly favor $\sqrt{\mathrm{VINS}}$.
}
\label{tab:ata_ate}
\end{table*}

\begin{figure}[]
\centering
\includegraphics[width=0.8\linewidth]{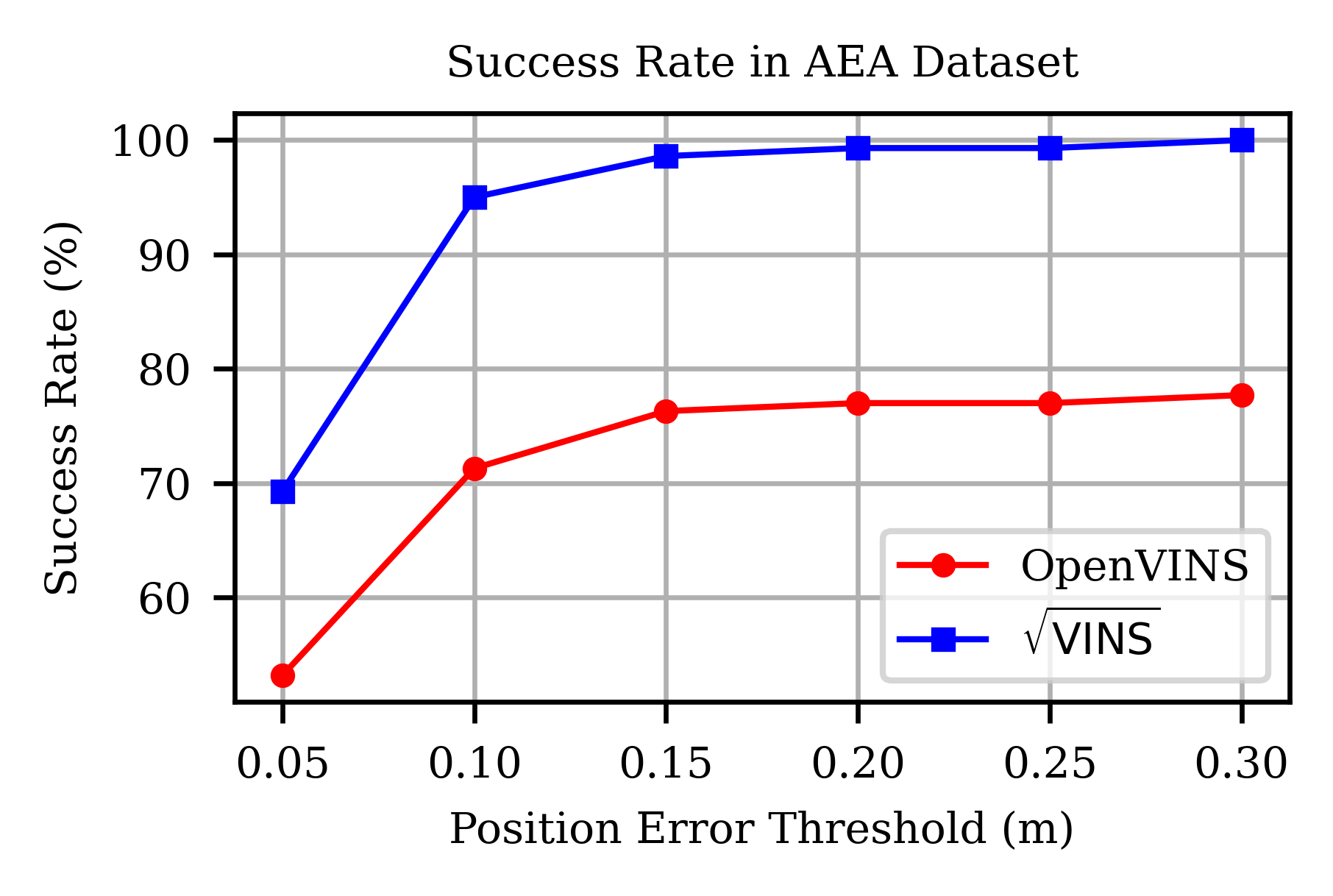}
\caption{ 
{Successful rate over different position error thresholds with dynamic initialization in Aria Everyday Activities (AEA) Dataset.
}
}
\label{fig:aea_init}
\end{figure}

\begin{figure}[]
\centering
\includegraphics[width=0.8\linewidth]{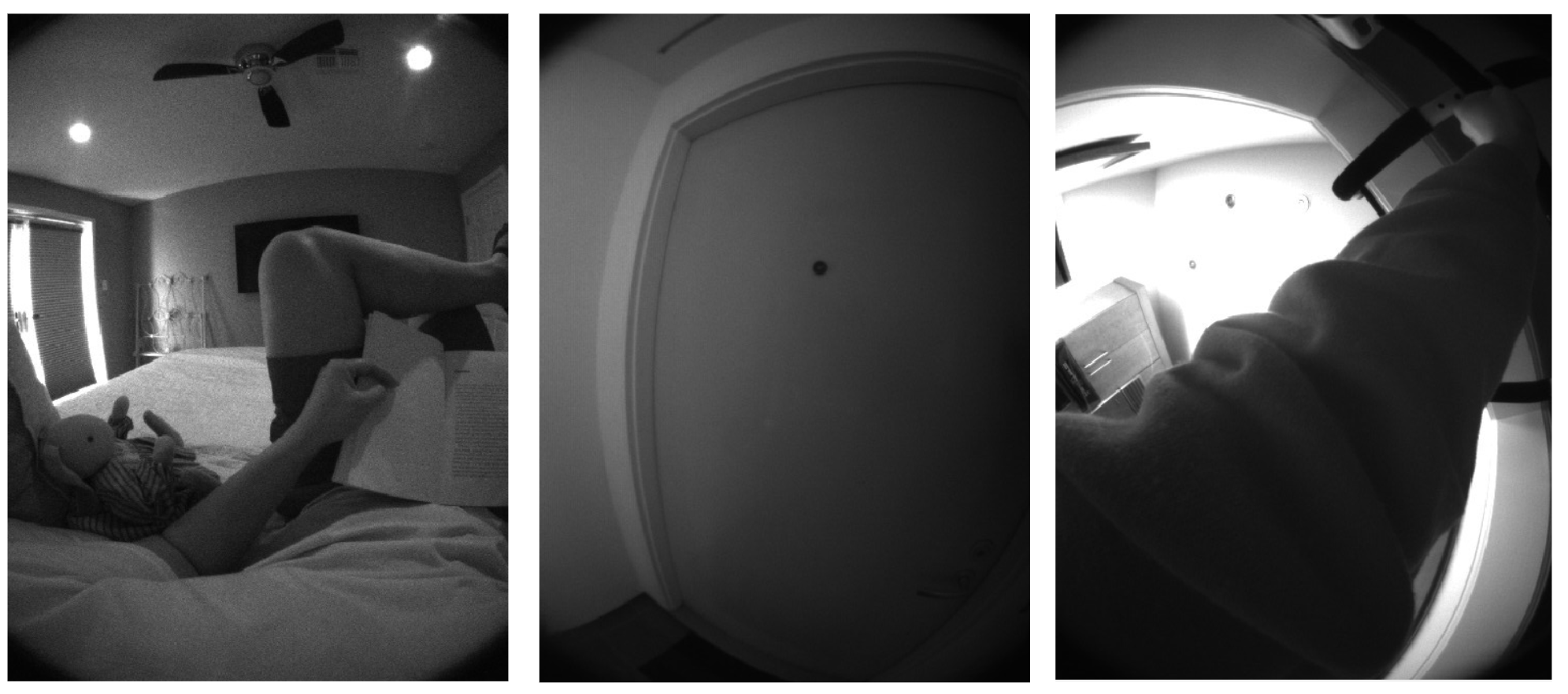}
\caption{ 
{Example images from the Aria Everyday Activities (AEA) Dataset.
}
}
\label{fig:aea_example}
\end{figure}

\section{Experimental Evaluation of $\sqrt{\mathrm{VINS}}$}
\label{sec:exp_vins}

In this section, we demonstrate the capabilities of $\sqrt{\mathrm{VINS}}$ in comparison with state-of-the-art VINS. The system is tested with both double-precision ($\sqrt{\mathrm{VINS}}$ (d)) and float-precision ($\sqrt{\mathrm{VINS}}$ (f)) versions. For comparison, we use OpenVINS~\cite{Geneva2019ICRA} in its original double-precision (OpenVINS(d)) and implement the float-precision version (OpenVINS(f)), which required tuning to avoid divergence due to negative covariance diagonals. Additionally, we compare with RVIO2~\cite{Huai2022RAL}, a square-root inverse filter VIO with a robocentric state formulation, and VINS-Mono~\cite{Qin2018TRO}, an optimization-based sliding-window VIO.

Since RVIO2 uses only MSCKF features by default, resulting in a smaller state size, we also test in a similar configuration (15 clones, 200 tracked features, all MSCKF), denoted as $\sqrt{\mathrm{VINS}}$(M), for a fair comparison.
Finally, $\sqrt{\mathrm{VINS}}$(L) is a lightweight setup that features ultra-efficiency, which tracks 150 features, keeps a maximum of 4 clones, and 15 SLAM features, and uses a maximum of 20 MSCKF features in the update. 
\color{black}
We also compared another state-of-the-art EKF-based system EqVIO\color{black}~\cite{Van2021ICRA,Van2023TRO}\color{black} 
that features equivariant formulation and impressive computational efficiency.
\color{black}
In the following, we report the performance of the above mentioned systems on the EurRoC MAV dataset~\cite{Burri2016IJRR} and the Aria Everyday Activities Dataset~\cite{Lv2024arXiv}.

\subsubsection{EurRoC MAV Dataset}
To evaluate this dataset, we use the default configuration of OpenVINS. This setup extracts 200 sparse features, keep 11 clones, and tracks up to 50 SLAM features and 40 MSCKF features. 
The system performs online calibration for camera-IMU extrinsics, time offsets, and camera intrinsics. 
Evaluation is conducted using only the left camera, with initialization performed from a static state.

The averaged ATE values are reported in Table~\ref{tab:real}. 
It is clear that the estimation accuracy of $\sqrt{\mathrm{VINS}}$(d), $\sqrt{\mathrm{VINS}}$(f), OpenVINS(d), and OpenVINS(f) are very similar as expected.
They are not exactly the same in the real world due to two reasons. 
First, $\chi^2$ test is adopted to reject outliers and robustify the estimator and might introduce randomness. 
For example, in certain cases, $\sqrt{\mathrm{VINS}}$(d) might reject measurements that pass $\chi^2$ test in $\sqrt{\mathrm{VINS}}$(f) because of slight numerical differences, this will cause different versions to use different measurements and have different performance. 
Second, OpenVINS performs a ``sequential" update, which first processes MSCKF features and then SLAM features for the consideration of efficiency, while \textcolor{black}{$\sqrt{\mathrm{VINS}}$} performs the update all at once. This also introduces differences in the state linearization points.
Compared with RVIO2, \textcolor{black}{EqVIO} and VINS-mono, \textcolor{black}{$\sqrt{\mathrm{VINS}}$} also achieves superior performance in almost all the sequences.
Surprisingly, even $\sqrt{\mathrm{VINS}}$(M) and $\sqrt{\mathrm{VINS}}$(L) achieves similar or even better performance than the other systems.

We then look into the efficiency of the estimators as reported in Table \ref{tab:runtime}. 
\color{black}
Feature tracking is also a crucial components in VINS, however, as it is not the focus of this work, we exclude its runtime in the efficiency analysis.
\color{black}
Clearly, $\sqrt{\mathrm{VINS}}$ is much faster than OpenVINS, reducing the runtime by half. 
Regardless of being in double or float format, $\sqrt{\mathrm{VINS}}$ consistently prevails over OpenVINS. 
Remarkably, the double precision $\sqrt{\mathrm{VINS}}$ even outperforms the float OpenVINS.
VINS-Mono runs the slowest as it performs iterative optimization. 
RVIO2 is also developed in float and shows excellent efficiency, but with a similar setup and a fair comparison, $\sqrt{\mathrm{VINS}}$(M) still prevails. 
\color{black}
EqVIO achieves strong efficiency, outperforming vanilla $\sqrt{\mathrm{VINS}}$ but remaining less efficient than $\sqrt{\mathrm{VINS}}(L)$ and $\sqrt{\mathrm{VINS}}(M)$. Its efficiency stems primarily from the SLAM-based state design: unlike MSCKF-based VINS, which maintains multiple clones, EqVIO keeps only a single clone while tracking a small number of SLAM features, thereby reducing both state and measurement dimensions, while compromising accuracy. Importantly, however, EqVIO’s key novelty lies in its use of Lie-group symmetry to improve estimator consistency. This contribution is orthogonal to our work, implying that a VIO system can, in principle, combine Lie-group symmetry with square-root filtering. In EqVIO, the Riccati matrix plays a role analogous to the covariance in the EKF. 
Tracking it in square-root form allows the use of SRF-based propagation and update methods, enabling robust operation under lower floating-point precision, exploiting triangular and symmetric structures for additional speedups, while simultaneously benefiting from the improved consistency offered by the equivariant formulation.
\color{black}
Finally, $\sqrt{\mathrm{VINS}}$(L) achieves the best efficiency with 0.4 ms in estimator runs, which means it can run over 2.5kHz, especially suitable for running on a computation-constrained platform. 
The efficiency gain mainly
comes from the proposed LLT-based SRF update method, which fully explored the problem structure (state order, upper-triangular covariance, Jacobian structure, avoid redundent computation) as discussed in Section~\ref{sec:vins}.

We also report the estimator runtime on the Jetson Nano, as shown in Figure~\ref{fig:init_time_jetson}. VINS-Mono fails to meet the real-time requirement (50 ms), $\sqrt{\mathrm{VINS}}$ demonstrates the highest efficiency among systems with features in the state (e.g., OpenVINS and VINS-Mono) and achieves comparable efficiency to RVIO2, which uses only MSCKF features. $\sqrt{\mathrm{VINS}}$(L) requires only 3.1 ms in an estimator run, which is almost 10 times faster than the default baseline OpenVINS, achieving ultra-efficiency.

\subsubsection{Aria Everyday Activities Dataset}

Aria Everyday Activities (AEA) dataset~\cite{Lv2024arXiv} provides an ego-centric perspective view recorded from Aria AR glasses, as shown in Figure~\ref{fig:aea_example}.
It contains 143 daily activity sequences recorded by multiple users in 5 indoor locations. 
The Aria glasses have two 10 Hz cameras and two IMUs. 
In the evaluation, both cameras are used, and only the right 1000 Hz IMU is used in our evaluation. 
The challenge for this dataset is that the two tracking cameras have limited view overlap so that does not provide good stereo depth. 
Also, the dataset is recorded mostly within motion at the beginning and requires dynamic initialization. 
In certain sequences, the cameras face a white textureless wall or door or have a large portion of dynamic objects (e.g. the user carries an object in front of the glass) further challenges feature tracking and the overall system robustness.

In our eveluation, OpenVINS and $\sqrt{\mathrm{VINS}}$ use the same setup, which tracks 200 features, uses a maximum of 6 clones, 15 SLAM features, and 40 MSCKF features in the update. 
Factory calibrations parameters from the dataset are used. 
For dynamic initialization, both systems use a 0.5-second window with 5 keyframes. 

We first report the success rate for this challange dataset, reported in Figure~\ref{fig:aea_init}.
In this test, after initialization, the subsequent VINS is run for the full trajectory, and the position error is evaluated at the end. 
The figure illustrates the percentage of runs that result in VIO achieving a position error below a certain threshold.
It is easy to observe that $\sqrt{\mathrm{VINS}}$ significantly outperforms in this challenging AR/VR scenarios, achieving over 95\% success with VIO errors under 0.1m at the end, compared to just 70\% for OpenVINS. 

We also report the VIO accuracy and estimator runtime on the AEA dataset in Table~\ref{tab:ata_ate}. 
For OpenVINS, only successful runs with position errors under 1 meter are included, whereas full results are presented for $\sqrt{\mathrm{VINS}}$, as all runs were successful. Despite this seemingly `unfair' comparison favoring OpenVINS, $\sqrt{\mathrm{VINS}}$ demonstrates superior accuracy. 
Runtime performance is also highlighted, with $\sqrt{\mathrm{VINS}}$ taking 0.9 ms compared to OpenVINS's 2.1 ms. Together with previous results, this evaluation further confirms the superior efficiency of $\sqrt{\mathrm{VINS}}$.
\section{Conclusions and Future Work}
\label{sec:conclu}

In this paper, we have developed a complete, robust and fast square-root VINS ($\sqrt{\mathrm{VINS}}$) with a faster-than-ever dynamic initialization which is able to robustly initialize the system even in just 100 ms with 3 keyframes. 
Our system sets a new VINS benchmark, offering over twice the speed, improved numerical stability, and enhanced robustness compared to SOTA algorithms of 3D motion tracking.
The square-root covariance filter (SRF) offers substantial benefits for VINS, especially in embedded systems, thanks to its improved numerical stability and efficiency. However, exploiting these advantages has been hindered by inefficiencies in the update process, particularly with large measurements. 
Building on a recent work~\cite{Peng2024ICRA}, this work proposed a novel LLT-based SRF update method, which leverages the structure of the VINS problem to achieve high efficiency and operational effectiveness in $\sqrt{\mathrm{VINS}}$.
Additionally, the proposed dynamic initialization in $\sqrt{\mathrm{VINS}}$ ensures rapid and reliable initialization under minimal conditions, marking a {\em first} in the literature. 
This capability is critical for practical deployments, where quick reinitialization is often required following system resets or failures.
Extensive numerical studies and real-world experiments have validated the robustness and efficiency of $\sqrt{\mathrm{VINS}}$. 
Notably, our system achieves twice the speed of SOTA methods while maintaining high accuracy, even under challenging conditions such as 32-bit single-precision float operations,
and the real-world results 
further demonstrate its superior performance across diverse scenarios for edge computing platforms.
\textcolor{black}{
Future work will focus on extending $\sqrt{\mathrm{VINS}}$ to support multi-sensor fusion and operate in dynamic environments, further improving scalability and robustness. We hope this work—and its open-source release—opens new possibilities for fast, reliable, and resource-efficient state estimation across a wide range of applications.
}


{
\bibliographystyle{IEEEtran}  

\begin{IEEEbiography}[{\includegraphics[width=1in,height=1.25in,clip,keepaspectratio]{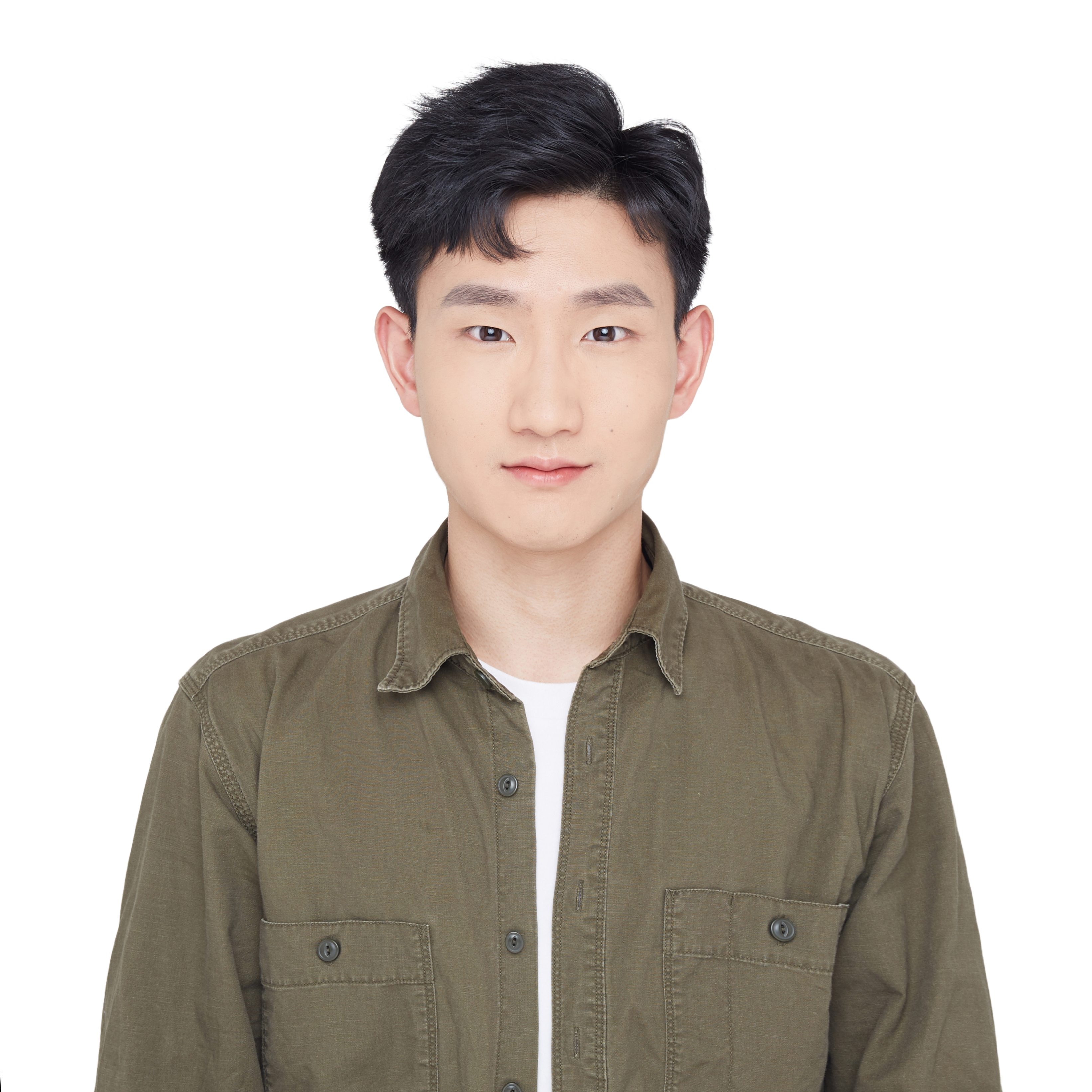}}]{Yuxiang Peng}
(Student Member, IEEE) received the Bachelor of Engineering degree in Mechatronics Engineering from Zhejiang University, Hangzhou, China, in 2019. Now
he is a Ph.D. candidate in Mechanical Engineering at the University of Delaware, Newark, DE, USA. 
He is a member of the Robot Perception and Navigation Group (RPNG) at the University of Delaware.  His research interests focus on efficient state estimation and spatial computing for embedded and resource-constrained robotic and XR platforms. He was a Finalist for the Best Paper Award on Robot Vision at the ICRA 2024. 
\end{IEEEbiography}
\vspace{-4.5\baselineskip}

\begin{IEEEbiography}[{\includegraphics[width=1in,height=1.25in,clip,keepaspectratio]{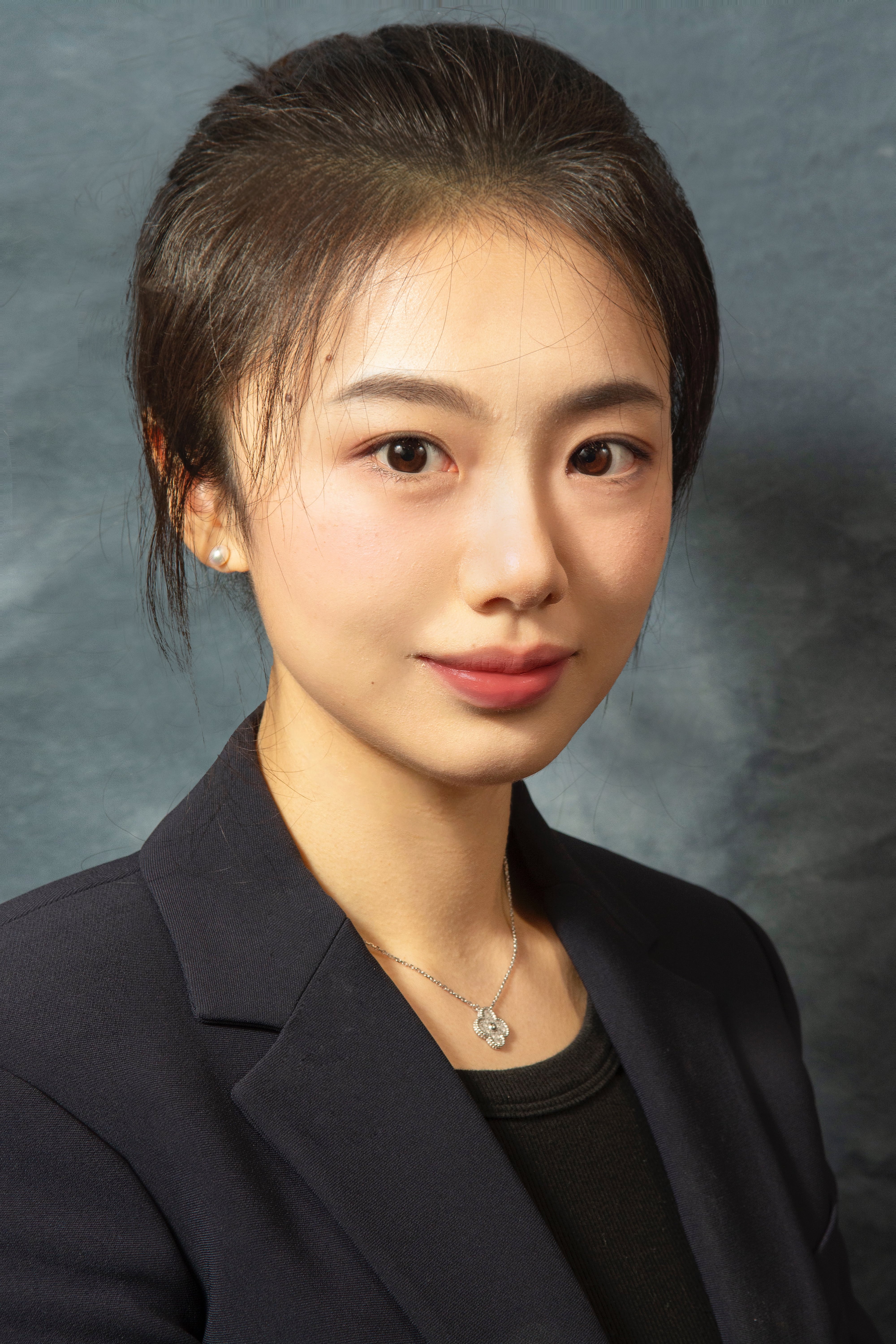}}]{Chuchu Chen} (Member, IEEE) 
received the B.Eng. degree in Automation from Harbin Engineering University (HEU), China, in 2017, and the Ph.D. degree in Mechanical Engineering from the University of Delaware (UD), Newark, DE, USA, in 2025, where she worked in the Robot Perception and Navigation Group (RPNG). She is currently an Assistant Professor in the Department of Mechanical and Aerospace Engineering at The George Washington University (GWU), Washington, DC, USA, where she directs the Estimation, Perception, and Intelligent Computing (EPIC) Lab. Her research interests lie at the intersection of state estimation, spatial perception, and embodied intelligence.
She has received multiple recognitions, including the Doctoral Fellowship for Excellence at the University of Delaware, the Best Student Paper Award Finalist at RSS 2023, and the Best Paper Award Finalist at ICRA 2024 (Robot Vision).
\end{IEEEbiography}
\vspace{-4.5\baselineskip}

\begin{IEEEbiography}[{\includegraphics[width=1in,height=1.25in,clip,keepaspectratio]{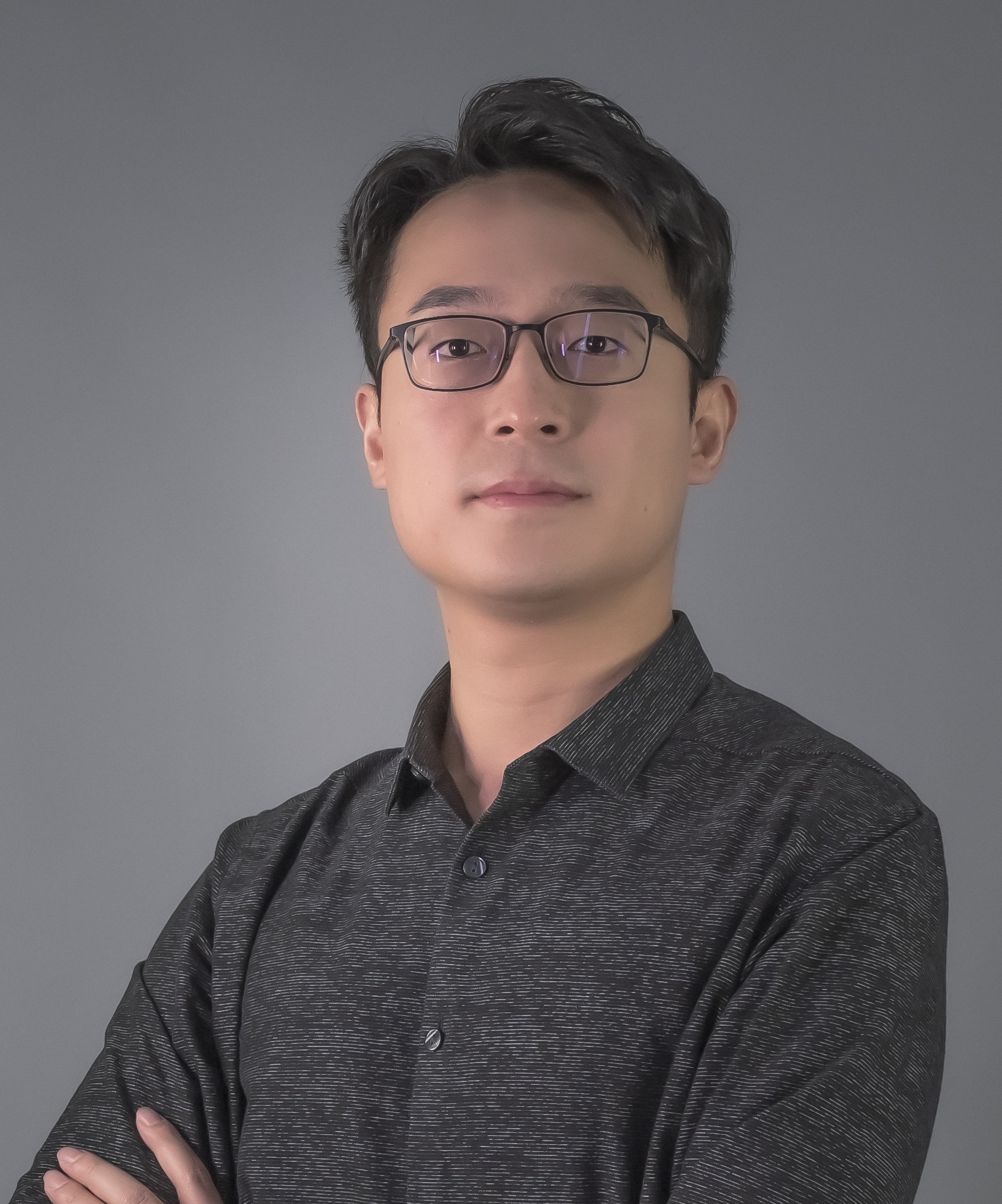}}]{Kejian Wu}
received a Bachelor of Engineering in Information and Communication Technology from Zhejiang University, Hangzhou, China, in 2010, and a Ph.D. in Electrical and Computer Engineering from University of Minnesota - Twin Cities, Minneapolis, MN, USA, in 2020. His research interests focus on simultaneous localization and mapping (SLAM), visual-inertial odometry (VIO), computer vision, estimation theory, sensor fusion and calibration, and artificial intelligence. From 2013 to 2018, he was a member of UMN MARS Lab’s partnership with Project Tango, Google, and an intern visual-inertial system researcher, Daydream, Google. 
He is currently the co-founder and head of perception at XREAL Inc, where he leads the perception and algorithm team at the company, building technical foundations of AR glasses.
\end{IEEEbiography}
\vspace{-4.5\baselineskip}

\begin{IEEEbiography}[{\includegraphics[width=1in,height=1.25in,clip,keepaspectratio]{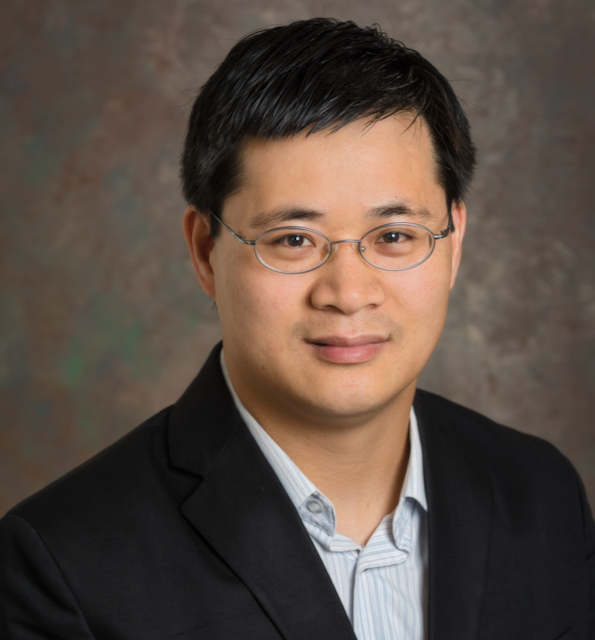}}]{Guoquan (Paul) Huang}
(Senior Member, IEEE) received his BS in
automation (electrical engineering) from the University of Science and
Technology Beijing, China, in 2002, and MS and PhD in computer science
from the University of Minnesota--Twin Cities, in 2009 and 2012,
respectively. He currently is a Professor of Mechanical Engineering
(ME) and Computer and Information Sciences (CIS) at the University of
Delaware (UD), where he is leading the Robot Perception and Navigation
Group (RPNG).  From 2012 to 2014, he was a Postdoctoral Associate with
the MIT CSAIL (Marine Robotics). His research interests focus on state
estimation and spatial computing for autonomous robots and mobile
devices, including sensing, calibration, localization, mapping,
perception, and navigation of ground, aerial and underwater vehicles.
He serves as an Associate Editor for the IEEE Transactions on Robotics
(T-RO) and IET Cyber-Systems and Robotics (CSR). Dr. Huang has
received multiple honors and awards including the ICRA 2022 Best Paper
Award (Navigation), 2022 GNC Journal Best Paper Award, and the
Finalists of the ICRA 2024 Best Paper Award (Robot Vision), RSS 2023
Best Student Paper Award, ICRA 2021 Best Paper Award (Robot Vision),
RSS 2009 Best Paper Award, and Mid-Career Faculty Excellence in
Scholarship Award, as well as several faculty research awards from
Google and Meta Reality Labs, and the NSF and NASA Research Initiation
Awards.
\end{IEEEbiography}


\end{document}